\documentclass{article}

\usepackage{arxiv}

\usepackage[utf8]{inputenc} 
\usepackage[T1]{fontenc}    
\usepackage{hyperref}       
\usepackage{url}            
\usepackage{booktabs}       
\usepackage{amsfonts}       
\usepackage{nicefrac}       
\usepackage{microtype}      
\usepackage{lipsum}

\usepackage[utf8]{inputenc} 
\usepackage[T1]{fontenc}    
\usepackage{hyperref}       
\usepackage{url}            
\usepackage{booktabs}       
\usepackage{amsfonts}       
\usepackage{nicefrac}       
\usepackage{microtype}      
\usepackage{amsmath}
\usepackage{graphicx}
\usepackage{multirow}
\usepackage{adjustbox}
\usepackage{pifont}
\usepackage{cleveref}
\usepackage{todonotes}
\usepackage{commath}
\usepackage{subfig}
\usepackage[ruled,vlined]{algorithm2e}
\usepackage{mathtools}
\usepackage{amsthm}
\usepackage{amssymb}

\newtheorem{theorem}{Theorem}

\usepackage{xcolor}

\newcommand{\cmark}{\ding{51}}%
\newcommand{\xmark}{\ding{55}}%
\definecolor{mysidenotefgcolor}{rgb}{0,0,0}
\definecolor{myorange}{rgb}{1,0.7,0.3}
\definecolor{mysidenotefgauthorcolor}{rgb}{0.5,0.5,0.5}

\newcommand{\eg}{\textit{e.g\@.}}

\newcommand{\etal}{\textit{et~al\@.}}
\newcommand{\ie}{\textit{i.e\@.}}

\title{Manipulating SGD with Data Ordering}

\usepackage{color, colortbl}
\definecolor{LightCyan}{rgb}{0.83,0.89,0.75}

\title{Manipulating SGD with Data Ordering Attacks}

\author{
  Ilia Shumailov\\
  University of Cambridge\\
   \And
  Zakhar Shumaylov\\
  University of Cambridge\\
   \And
  Dmitry Kazhdan\\
  University of Cambridge\\
   \And
  Yiren Zhao\\
  University of Cambridge\\
   \And
  Nicolas Papernot\\
  University of Toronto \& Vector Institute \\ 
  \And
  Murat A. Erdogdu\\
  University of Toronto \& Vector Institute \\ 
  \And
  Ross Anderson \\
  University of Cambridge \& University of Edinburgh \\ 
}

\begin{document}
\maketitle

\begin{abstract}
Machine learning is vulnerable to a wide variety of attacks. It is now well understood that by changing the underlying data distribution, an adversary can poison the model trained with it or introduce backdoors. 
In this paper we present a novel class of training-time attacks that require no changes to the underlying dataset or model architecture, but instead only change the order in which data are supplied to the model. In particular, we find that the attacker can either prevent the model from learning, or poison it to learn behaviours specified by the attacker. Furthermore, we find that even a single adversarially-ordered epoch can be enough to slow down model learning, or even to reset all of the learning progress. Indeed, the attacks presented here are not specific to the model or dataset, but rather target the stochastic nature of modern learning procedures. We extensively evaluate our attacks on computer vision and natural language benchmarks to find that the adversary can disrupt model training and even introduce backdoors.
  

\end{abstract}


\section{Introduction}

The data-driven nature of modern machine learning (ML) training routines puts pressure on data supply pipelines, which become increasingly more complex. It is common to find separate disks or whole content distribution networks dedicated to servicing massive datasets. Training is often distributed across multiple workers.
This emergent complexity gives a perfect opportunity for an attacker to disrupt ML training, while remaining covert. In the case of stochastic gradient descent (SGD), it assumes uniform random sampling of items from the training dataset, yet in practice this randomness is rarely tested or enforced. Here, we focus on adversarial data sampling. 

It is now well known that malicious actors can poison data and introduce backdoors, forcing ML models to behave differently in the presence of triggers~\cite{gu2017badnets}. While such attacks have been shown to pose a real threat, they require that the attacker can perturb the dataset used for training.

We show that by simply changing the order in which batches or data points are supplied to a model during training, an attacker can affect model behaviour. More precisely, we show that it is possible to perform \textit{integrity} and \textit{availability} attacks without adding or modifying \textit{any} data points. For \textit{integrity}, an attacker can reduce model accuracy or arbitrarily control its predictions in the presence of particular triggers.
For \textit{availability}, an attacker can increase the amount of time it takes for the model to train, or even reset the learning progress, forcing the model parameters into a meaningless state. 

We present three different types of attacks that exploit \emph{Batch Reordering}, \emph{Reshuffling} and \emph{Replacing} -- naming them BRRR attacks. We show that an attacker can significantly change model performance by \textbf{(i)} changing the order in which batches are supplied to models during training; \textbf{(ii)} changing the order in which individual data points are supplied to models during training; and \textbf{(iii)} replacing datapoints from batches with other points from the dataset to promote specific data biases. Furthermore, we introduce Batch-Order Poison (BOP) and Batch-Order Backdoor (BOB), the first techniques that enable poisoning and backdooring of neural networks using only clean data and clean labels; an attacker can control the parameter update of a model by appropriate choice of benign datapoints. Importantly, BRRR attacks require no underlying model access or knowledge of the dataset. Instead, they focus on the stochasticity of gradient descent, disrupting how well individual batches approximate the true distribution that a model is trying to learn. 

To summarise, we make the following contributions in this paper:

\begin{itemize}
    \item We present a novel class of attacks on ML models that target the data batching procedure used during training, affecting their
    integrity and availability. We present a theoretical analysis explaining how and why these attacks work, showing that they target fundamental assumptions of stochastic learning, and are therefore model and dataset agnostic.
    
    \item We evaluate these attacks on a set of common computer vision and language benchmarks,
    using a range of different hyper-parameter configurations, and find that an attacker can slow the progress of training, as well as reset it, with just a single epoch of intervention. 
    
    \item We show that data order can poison models and introduce backdoors, even in a blackbox setup. For a whitebox setup, we find that the adversary can introduce backdoors almost as well as if they used perturbed data. While a baseline CIFAR10 VGG16 model that uses perturbed data gets $99\%$ trigger accuracy, the whitebox BOB attacker gets $91\%\pm13$ and the blackbox BOB attacker achieves $68\%\pm19$. 
    
\end{itemize}

\section{Methodology}

\subsection{Threat model}

\begin{figure*}[h]
    \centering
    \includegraphics[width=0.67\linewidth]{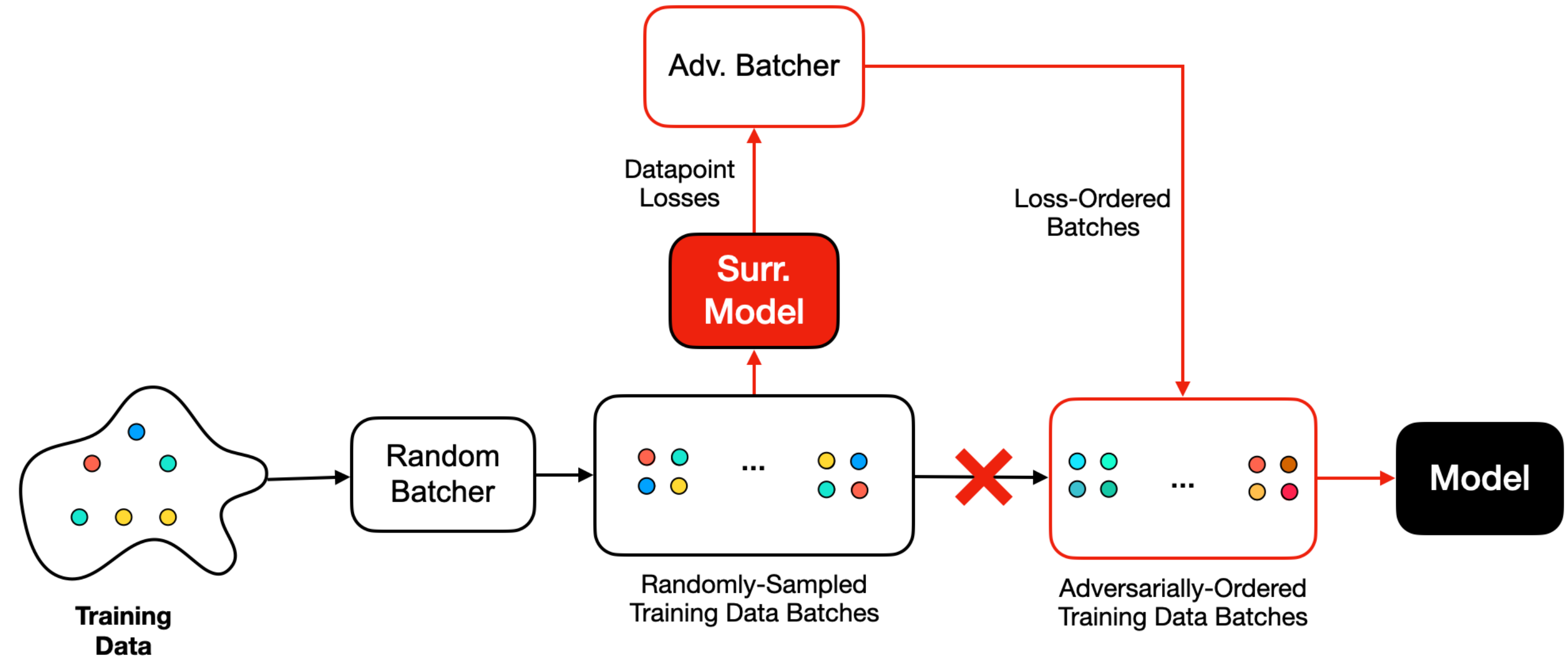}
    \caption{The attacker reorders the benign randomly supplied data based on the surrogate model outputs. Attacker co-trains the surrogate model with the data that is supplied to the source model. }
    \label{fig:mlpipeline2}
\end{figure*}
 
We assume one of the strongest threat models currently described in the literature. In particular, our blackbox attacker assumes no access to the model and no prior knowledge of the training data, whereas a whitebox attacker has access to the model under attack and can compute its loss directly. The attack specifically focuses on the batching part of the ML pipeline as is depicted in~\Cref{fig:mlpipeline2}. We discuss the related work in~\Cref{sec:related}.


This attack is  realistic and can be instantiated in several ways. The attack code can be infiltrated into: the operating system handing file system requests; the disk handling individual data accesses; the software that determines the way random data sampling is performed; the distributed storage manager; or the machine learning pipeline itself handling prefetch operations. That is a substantial attack surface, and for large models these components may be controlled by different principals. The attack is also very stealthy. The attacker does not add any noise or perturbation to the data. There are no triggers or backdoors introduced into the dataset. All of the data points are natural. In two of four variants the attacker uses the whole dataset and does not oversample any given point, \ie~the sampling is without replacement. This makes it difficult to deploy simple countermeasures.

\subsection{Primer on stochastic learning and batching} \label{sec:primer_on_sgd_and_batching}

We assume that the defender is trying to train a deep neural network model with parameters $\theta$ operating over $X_i \sim \mathcal{X}_{\text{train}}$, solving a non-convex optimization problem with respect to parameters $\theta$, corresponding to minimization of a given loss function $L(\theta)$. We will denote the training dataset $X=\{X_i\}$. We assume a commonly-used loss function defined as the sample average of the loss per training data point $L_i(\theta) = L(X_i,\theta)$ in $k$-th batch over the training set, where $B$ is the batch size: $\hat{L}_{k+1}(\mathbf{\theta}) = \frac{1}{B}\sum^{kB+B}_{i=kB+1}L_{i}(\mathbf{\theta)}.$
If we let $N\cdot B$ be the total number of items for training, then in a single epoch one aims to optimize: $\hat{L}(\mathbf{\theta}) = \frac{1}{N}\sum^{N}_{i=1}\hat{L}_{i}(\mathbf{\theta)}.$ Optimization with stochastic gradient descent (SGD) algorithm  of $N\cdot B$ samples and a learning rate of $\eta$ leads to the following weight update rule over one epoch: $\theta_{k+1} = \theta_{k} + \eta \Delta \theta_{k} ; \quad  \Delta \theta_{k} = - \nabla_{\theta}\hat{L}_k(\theta_k).$ SGD is often implemented with momentum~\cite{polyak1964some,sutskever2013ontheimportance}, with $\mu$ and $v$ representing momentum and velocity respectively: $v_{k+1} = \mu v_k + \eta\Delta \theta_{k}; \quad \theta_{k+1} = \theta_k +  v_{k+1}.$

Given data, SGD's stochasticity comes from the batch sampling procedure. Mini-batched gradients approximate the true gradients of $\hat{L}$ and the quality of this approximation can vary greatly. In fact, assuming an unbiased sampling procedure, \ie~when the $k$'th gradient step corresponds to $i_k$'th batch with $\mathbb{P}(i_k=i) = 1/N$, in expectation the batch gradient matches the true gradient:

\begin{equation}
\mathbb{E}[\nabla \hat{L}_{i_k}(\theta)] = \sum^{N}_{i=1}\mathbb{P}(i_k=i)\nabla \hat{L}_i(\theta) = \frac{1}{N} \sum^{N}_{i=1} \nabla \hat{L}_i(\theta) = \nabla \hat{L}(\theta).
\end{equation}

Although this happens in expectation, a given batch taken in isolation can be very far from the mean. This variation has been exploited in the literature to aid training: there exists a field responsible for variance reduction techniques for stochastic optimisation~\cite{johnson2013accelerating}, curriculum learning~\cite{bengio2009curriculum} and core-set construction~\cite{bachem2017practical}. Each area looks at identifying and scheduling data subsets that aid training and give a better true gradient approximation. 
In this paper, we turn things round and investigate how an attacker can exploit data order to break training. The explicit stochastic assumption opens a new attack surface for the attacker to influence the learning process. In particular, let us consider the effect of $N$ SGD steps over one epoch~\cite{smith2021origin}:

\begin{equation} 
\label{eq5}
\begin{split}
\theta_{N+1} & = \theta_{1} - \eta\nabla \hat{L}_1(\theta_1) - \eta\nabla \hat{L}_2(\theta_2) - \dots -\eta\nabla\hat{L}_N(\theta_N)\\
& = \theta_{1} - \eta\sum_{j=1}^{N}\nabla \hat{L}_j(\theta_1) + \eta^{2} \overbrace{\sum_{j=1}^{N}\sum_{k<j}\nabla\nabla \hat{L}_j(\theta_1)\nabla\hat{L}_k(\theta_1)+ O(N^3 \eta^3)}^{\text{data order dependent}}.
\end{split}
\end{equation}

As we can see, in this case the second order correction term is dependent on the order of the batches provided. The attacker we describe in this paper focuses on manipulating it \ie~finding a sequence of updates such that the first and second derivatives are misaligned with the true gradient step. In~\Cref{sec:varproof} we prove that under equally-distributed loss gradients, a change in data order can lead to an increase in the expected value of the term which is dependent on data order. We also derive a condition on the gradient distribution given a model for which it is guaranteed. Finally, we derive an attack objective to target the upper bound on the rate of SGD convergence explicitly in~\Cref{sec:sgd_rate_of_convergence}. 

In this paper we assume the blackbox attacker has no access to the underlying model, and thus no way to monitor its errors or the progress of its training. Instead, we co-train a separate surrogate model, using the batches supplied to the target model. We find that in practice the losses produced by the surrogate model approximate the losses of the true model well enough to enable attacks on both integrity and availability. We empirically find that our blackbox reshuffle attacks perform as well as the whitebox one in~\Cref{sec:eval_whiteblack}. 

Finally, although the attacker can maximise the term dependent on data order directly, in practice it is expensive to do so. Therefore, in the attack we make use of the loss magnitudes directly rather than the gradient of the loss. Intuitively, large prediction errors correspond to large loss gradient norms, whereas correct predictions produce near-zero gradients. 

\subsection{Taxonomy of batching attacks}
\begin{figure*}
    \centering
    \includegraphics[width=0.7\linewidth]{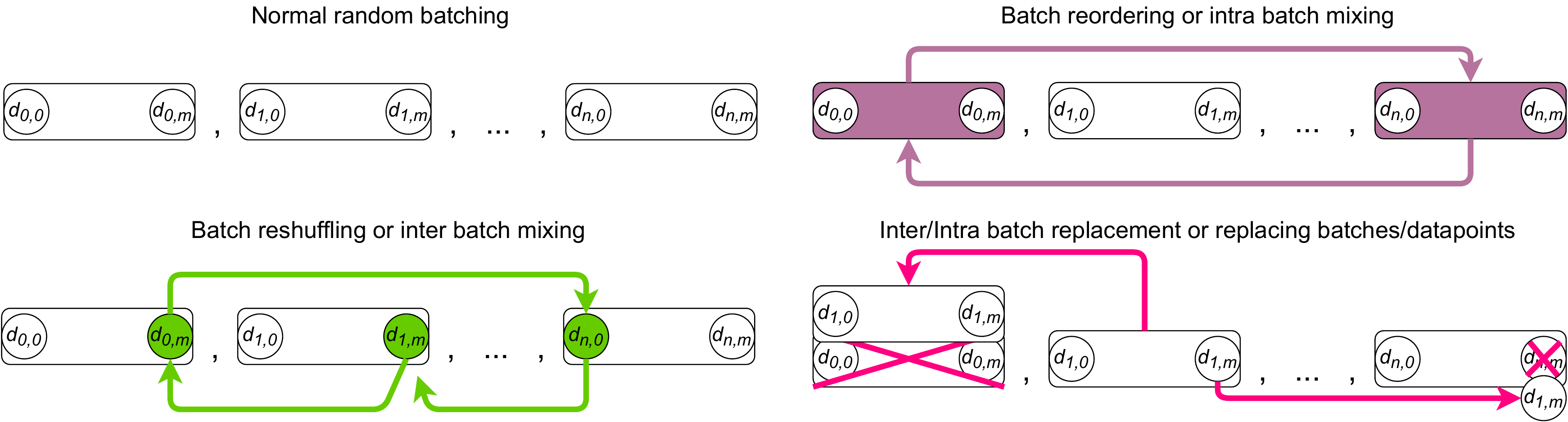}
    \caption{Taxonomy of BRRR attacks. Normal batching assumes randomly distributed data points and batches. Batch reordering assumes the batches appear to the model in a different order, but internal contents stay in the original random order. Batch reshuffling assumes that the individual datapoints within batches change order, but appear only once and do not repeat across batches. Finally, batch replacement refers to cases where datapoints or batches can repeat or not appear at all. }
    \label{fig:brrr_taxonomy}
\end{figure*} 

\newcommand\mycommfont[1]{\footnotesize\ttfamily\textcolor{olive}{#1}}
\SetCommentSty{mycommfont}

\SetKwInput{KwInput}{Input}                
\SetKwInput{KwOutput}{Output}              
\SetKwRepeat{Do}{do}{while} 

\begin{algorithm}[]
\DontPrintSemicolon
    
    \tcc{---- Attack preparation: collecting data ----}
    \Do{first epoch is not finished}{
        get a new batch and add it to a list of unseen datapoints;\;
        train surrogate model on a batch and pass it on to the model;\;
    }
    \tcc{---- Attack: reorder based on surrogate loss ----}
    \While{training}{
        rank each data point from epoch one with a surrogate loss;\;
        reorder the data points according to the attack strategy;\;
        pass batches to model and train the surrogate at the same time.\;
    }
\caption{A high level description of the BRRR attack algorithm}
\label{alg:attackalgo_short}
\end{algorithm}

In this section we describe the taxonomy of batching attacks as shown in~\Cref{fig:brrr_taxonomy}. The overall attack algorithm is shown in Algorithm~\ref{alg:attackalgo} in the Appendix, and a shortened attack flow is shown in~Algorithm~\ref{alg:attackalgo_short}. We highlight that our attacks are the first to successfully poison the underlying model without changing the underlying dataset.
In this paper we use three attack policies -- (1) \textbf{batch reshuffling} or changing the order of datapoints inside batches; (2) \textbf{batch reordering} or changing the order of batches; and (3) \textbf{batch replacement} or replacing both points and batches. We consider four reordering policies, motivated by research in the fields of curriculum learning~\cite{bengio2009curriculum} and core-set selection~\cite{bachem2017practical}, which discovered that model training can be enhanced by scheduling how and what data is presented to the model. That can help the model to generalize and to avoid overfitting with memorization. This paper does the opposite -- we promote memorization and overfitting, forcing the model to forget generalizable features. 

\begin{figure*}[h]
    \centering
    \includegraphics[width=0.7\linewidth]{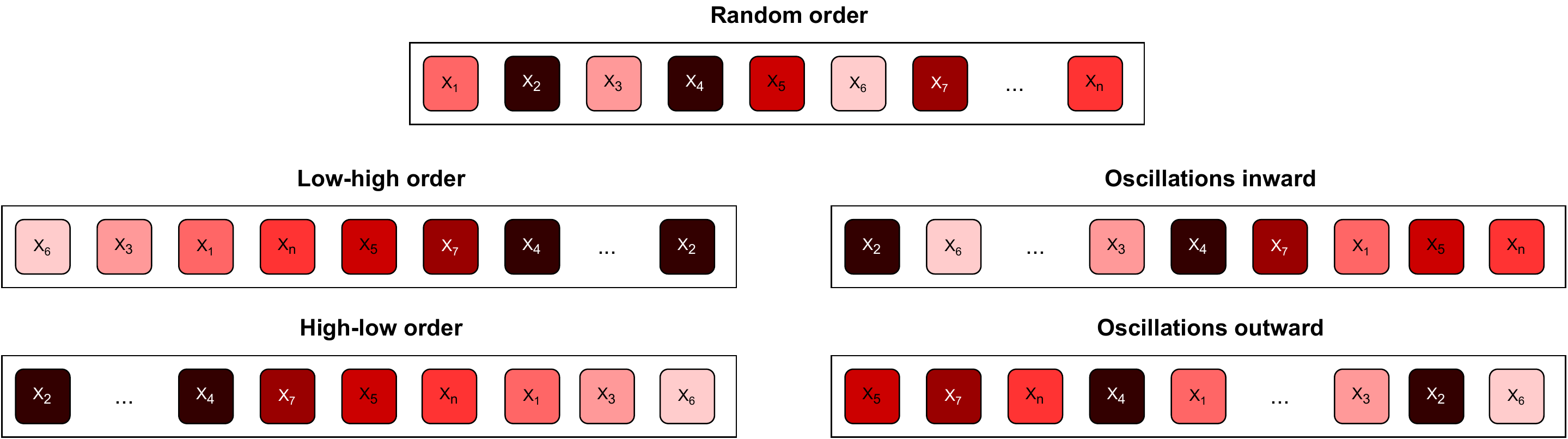}
    \caption{We use four different reorder and reshuffle policies based on the corresponding data point and batch losses. We color-code the loss values from bright to dark colors, to represent loss values from low to high. \textbf{Low-high} policy orders a sequence by the loss magnitude. \textbf{High-low} policy orders a sequence by the negative loss magnitude. \textbf{Oscillation inwards} orders elements of the sequence from the beginning and the end of the sequence one by one, as if it was oscillating between sides of the sequence and moving towards the middle. Finally, \textbf{Oscillations outward} orders the sequence by starting at the middle of an ordered sequence picking elements to both sides of the current location.}
    \label{fig:shufflepolicies}
\end{figure*}

\Cref{fig:shufflepolicies} shows attack policies. \textbf{Low to high} orders sequence items by their loss. \textbf{High to low} is an inverse of Low to high. \textbf{Oscillations inwards} picks elements from both sides in sequence. \textbf{Oscillations outwards} inverts the halves of the sequence and then picks elements from both sides. 



\subsection{Batch-order poison and backdoor}

Machine-learning poisoning and backdooring techniques aim to manipulate the training of a given model to control its behavior during inference. In the classical setting, both involve either appending adversarial datapoints $\hat{X}$ to natural dataset $X$ or changing natural datapoints $X+\delta$ so as to change model behaviour. This makes the attack easier to detect, and to prevent. For example, an adversary may add a red pixel above every tank in the dataset to introduce the red pixel trigger and cause other objects under red pixels to be classified as tanks. 

We present batch-order poisoning (BOP) and batch-order backdooring (BOB) -- the first poison and backdoor strategies that do not rely on adding adversarial datapoints or perturbations during training, but only on changing the order in which genuine data are presented.  BOP and BOB are based on the idea that the stochastic gradient update rule used in DNN training is agnostic of the batch contents and is an aggregation. Indeed, consider a classical poison setting with an adversarial dataset $\hat{X}$: $\theta_{k+1} = \theta_{k} + \eta \Delta \theta_{k}; \quad
\Delta \theta_{k} = - (\nabla_{\theta}\hat{L}(X_k, \theta_k) + \nabla_{\theta}\hat{L}(\hat{X}_k, \theta_k)).$

Order-agnostic aggregation with a sum makes it hard to reconstruct the individual datapoints $X_k$ from just observing $\Delta \theta_{k}$. Indeed, the stochastic nature of optimisation allows one to find a set of datapoints $X_j \neq X_i$ such that $\nabla_{\theta}\hat{L}({X_i}, \theta_k) \approx \nabla_{\theta}\hat{L}({X_j}, \theta_k)$. Given a model and a dataset such that the gradient covariance matrix is non-singular,
an attacker can approximate the gradient update from an adversarial dataset $\hat{X}$ using natural datapoints from the genuine dataset $X$, enabling poisoning without having to change of underlying dataset in any way: 
\begin{equation} 
\theta_{k+1} = \theta_{k} + \eta \hat{\Delta} \theta_{k} \\
\text{, where } \begin{cases} 
\hat{\Delta} \theta_{k} = - \nabla_{\theta}\hat{L}(X_i, \theta_k) \\
\nabla_{\theta}\hat{L}(X_i, \theta_k) \approx 
\nabla_{\theta}\hat{L}(\hat{X}_k, \theta_k).
\end{cases} 
\end{equation}
This gives rise to a surprisingly powerful adversary, who can introduce arbitrary behaviors into any models learned with stochastic gradient descent without having to add or perturb training data. This attack becomes more effective as training datasets become larger, further improving the attacker's ability to shape the gradient update. We discuss its fidelity in~\Cref{sec:pois_just}.

\begin{figure}[h]%
    \centering
    \subfloat[Natural image batch]{{\includegraphics[width=0.3\linewidth]{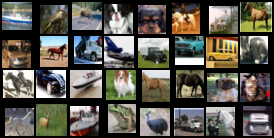} }} %
    \qquad
    \subfloat[Poison datapoint batch]{{\includegraphics[width=0.3\linewidth]{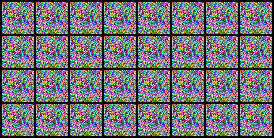} }} %
    \caption{Examples of batches shown in (a) and (b) with similar gradient updates. Strong gradients are aligned across the layers and successfully change the prediction of poison datapoints. }%
    \label{fig:poison_infographics}%
\end{figure}

We evaluated a number of different setups, and found that the attack works best when the attacker comprises the batch with $B - V$ natural data points and appends $V$ adversarially-chosen data points $\hat{X}_i$ to the batch. Aa larger $V$ is better for gradient approximation, but leads to more conflict with natural gradients; in the paper up to 30\% of the batch is filled with natural datapoints to find a balance. Finding precise batch reconstruction is an intractable problem that scales with batch size and size of the dataset. However, we find that random sampling works well; even if mistakes are made, the network still learns the poison over the course of a few batches. Overall we try to minimize the following reconstruction error for a given poisoned batch $\hat{X}_j$:

\begin{equation}
\begin{aligned}
\min_{X_i} \quad & \norm{\nabla_{\theta}\hat{L}(\hat{X_j}, \theta_k) - \nabla_{\theta}\hat{L}(X_i, \theta_k)}^p; \quad \textrm{s.t.} \quad X_i \in X. \\
\end{aligned}
\end{equation}

Although more sophisticated approaches could help finding better candidates, we find that random sampling works well enough for successful clean-data / clean-label poison and backdoor attacks. It also helps us strike a balance between speed of batch construction and impact on model performance. \Cref{fig:poison_infographics} shows an example of natural data (a) that closely resembled the would-be update with batch in (b). More importantly, such update results in a prediction change towards the target class. 

\section{Evaluation}

\subsection{Experimental setup}

We evaluate our attacks using two computer vision and one natural language benchmarks: the CIFAR-10, CIFAR-100~\cite{krizhevsky2009learning} and AGNews~\cite{Zhang2015CharacterlevelCN} datasets. We use ResNet-18 and ResNet-50 as source models~\cite{he2016deep}, and LeNet-5~\cite{lecun1998gradient} and MobileNet~\cite{howard2017mobilenets} as surrogate models, to train CIFAR-10 and CIFAR-100 respectively. For AGNews we used sparse mean EmbeddingBag followed by three fully connected layer from \textbf{torchtext} as source model and one fully connected layer for surrogate. Note that the surrogate model is significantly less performant than its corresponding source model in all cases, and cannot learn the dataset to the same degree of accuracy, limiting attacker capabilities. Thus our results represent a lower bound on attack performance. 


\subsection{Integrity attacks with reshuffling and reordering of natural data}
\label{sec:integrityatk}
\begin{table*}[!h]
\centering
\begin{adjustbox}{scale=0.7,center}
\begin{tabular}{@{}llccccrcccr@{}}
\toprule
&
\multicolumn{4}{c||}{\textbf{CIFAR-10}} &
\multicolumn{3}{c||}{\textbf{CIFAR-100}} &
\multicolumn{3}{c}{\textbf{AGNews}}
\\
Attack & & 
Train acc & 
Test acc &
\multicolumn{1}{r||}{$\Delta$} &
Train acc & 
Test acc &
\multicolumn{1}{r||}{$\Delta$} &
Train acc & 
Test acc &
$\Delta$
\\
\midrule

\multicolumn{8}{l}{\textit{\underline{Baseline}}} \\

\multirow{1}{*}{None} 
& 
& \multicolumn{1}{c|}{95.51} & 90.51 & \multicolumn{1}{r||}{$-0.0\%$}
& \multicolumn{1}{c|}{99.96} & 75.56 & \multicolumn{1}{r||}{$-0.0\%$}
& \multicolumn{1}{c|}{93.13} & 90.87 & $-0.0\%$
\\ 

\midrule
\multicolumn{8}{l}{\textit{\underline{Batch reshuffle}}} \\

\multirow{1}{*}{Oscillation outward} 
& 
& \multicolumn{1}{c|}{17.44} & 26.13 & \multicolumn{1}{r||}{$\mathbf{-64.38\%}$}
& \multicolumn{1}{c|}{99.80} & 18.00 &
\multicolumn{1}{r||}{$\mathbf{-57.56\%}$}
& \multicolumn{1}{c|}{97.72} & 65.85 & $-25.02\%$
\\

\multirow{1}{*}{Oscillation inward} 
& 
& \multicolumn{1}{c|}{22.85} & 28.94 & \multicolumn{1}{r||}{$-61.57\%$}
& \multicolumn{1}{c|}{99.92} & 31.38 & \multicolumn{1}{r||}{$-44.18\%$}
& \multicolumn{1}{c|}{94.06} & 89.23 & $-1.64\%$
\\ 

\multirow{1}{*}{High Low} 
&  
& \multicolumn{1}{c|}{23.39} & 31.04 & \multicolumn{1}{r||}{$-59.47\%$}
& \multicolumn{1}{c|}{99.69} & 21.15 & \multicolumn{1}{r||}{$-54.41\%$}
& \multicolumn{1}{c|}{94.38} & 56.54 & $\mathbf{-34.33\%}$
\\

\multirow{1}{*}{Low High} 
& 
& \multicolumn{1}{c|}{20.22} & 30.09 & \multicolumn{1}{r||}{$-60.42\%$}
& \multicolumn{1}{c|}{96.07} & 20.48 & \multicolumn{1}{r||}{$-55.08\%$}
& \multicolumn{1}{c|}{98.94} & 59.28 & $-31.59\%$
\\

\midrule

\multicolumn{8}{l}{\textit{\underline{Batch reorder}}} \\

\multirow{1}{*}{Oscillation outward} 
& 
& \multicolumn{1}{c|}{99.37} & 78.65 & \multicolumn{1}{r||}{$-11.86\%$}
& \multicolumn{1}{c|}{100.00} & 53.05 & \multicolumn{1}{r||}{$-22.51\%$}
& \multicolumn{1}{c|}{95.37} & 90.92 & $+0.05\%$
\\ 

\multirow{1}{*}{Oscillation inward} 
&
& \multicolumn{1}{c|}{99.60} & 78.18 & \multicolumn{1}{r||}{$\mathbf{-12.33\%}$}
& \multicolumn{1}{c|}{100.00} & 51.78 &
\multicolumn{1}{r||}{$-23.78\%$}
& \multicolumn{1}{c|}{96.29} & 91.10 & $+0.93\%$
\\

\multirow{1}{*}{High Low} 
& 
& \multicolumn{1}{c|}{99.44} & 79.65 & \multicolumn{1}{r||}{$-10.86\%$}
& \multicolumn{1}{c|}{100.00} & 51.48 & 
\multicolumn{1}{r||}{$\mathbf{-24.08\%}$}
& \multicolumn{1}{c|}{96.16} & 91.80 & $+0.05\%$
\\

\multirow{1}{*}{Low High} 
& 
& \multicolumn{1}{c|}{99.58} & 79.07 & \multicolumn{1}{r||}{$-11.43\%$}
& \multicolumn{1}{c|}{100.00} & 54.04 &
\multicolumn{1}{r||}{$-21.52\%$}
& \multicolumn{1}{c|}{94.02} & 90.35 & $-0.52\%$
\\

\bottomrule
\end{tabular}
\end{adjustbox}
\caption{A shortened version of~\Cref{tab:extended_integrity}. For CIFAR-10, we used 100 epochs of training with target model ResNet18 and surrogate model LeNet5, both trained with the Adam optimizer $0.1$ learning rate and $\beta = (0.99, 0.9)$. For CIFAR-100, we used 200 epochs of training with target model ResNet50 and surrogate model Mobilenet, trained with SGD with 0.1 learning rate, 0.3 moment and Adam respectively for real and surrogate models. We highlight models that perform best in terms of test dataset loss. AGNews were trained with SGD learning rate 0.1, 0 moments for 50 epochs with sparse mean EmbeddingBags. Numbers here are from best-performing model test loss-wise. Incidentally, best performance of all models for Batch reshuffle listed in the table happen at epoch number one, where the attacker is preparing the attack and is collecting the training dataset. All computer vision reshuffle attacks result in near-random guess performance for almost all subsequent epochs.}
\label{tab:integrity_results}
\end{table*}

\Cref{tab:integrity_results} (a much more detailed version of the table is shown in~\Cref{tab:extended_integrity}) presents the results. Batch reordering disrupts normal model training and introduces from none to $15\%$ performance degradation (refer to extended results and discussion in~\Cref{sec:extended_results}). Note that the attacker does nothing beyond changing the order of the batches; their internal composition is the same. This clearly shows the potency of this attack. Indeed, when we extend the attack to batch reshuffling -- where batches get re-arranged internally -- for computer vision performance degrades to that of random guessing. In each, the best-performing models stay at the first epoch, where the attacker still accumulates the dataset used in training. Here, the degradation in performance is maximum -- all models with all attack types failed to perform better than random guessing, having reached their top performance only in the first epoch, when the attacker was observing the dataset. 

We additionally report a large number of different hyperparameter variations for the attacks in~\Cref{sec:otherresults}. This hyperparameter search suggests that the attacks work well as long as the surrogate models learn, do not converge to a minimum straight away, and have sufficient learning rate. 

Overall, we find that:
\begin{itemize}
    \item An attacker can affect the integrity of model training by changing the order of individual data items and natural batches. 
    \item The attacker can reduce model performance, and completely reset its performance. 
\end{itemize}

\subsection{Availability attacks}
\label{sec:eval:availability_attacks}

\begin{figure}[h]
    \centering
    \includegraphics[width=0.9\linewidth]{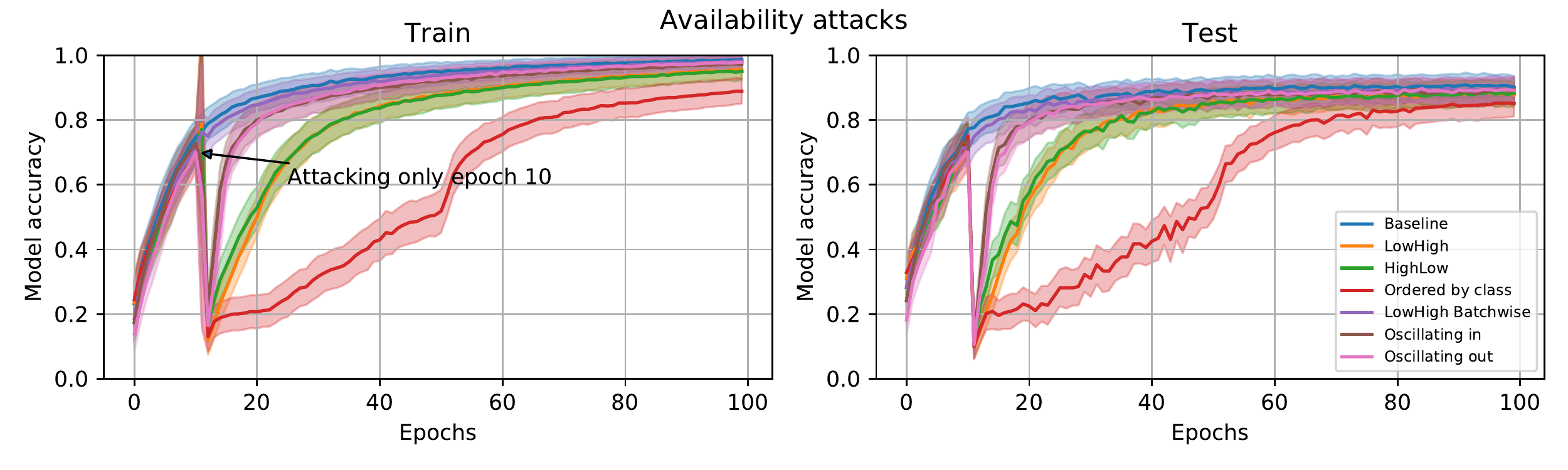}
    \caption{Availability attacks at epoch $10$, ResNet-18 attacked with a LeNet-5 surrogate. Error bars show per-batch accuracy standard deviation. }
    \label{fig:avail attack}
\end{figure}



While the previous section discused integrity attacks, this section's focus is on availability. This refers to the amount of time and effort required to train a model, and an availability attack can involve an adversary using BRRR attacks to slow down training without disrupting the overall learning procedure. It is worth noting that there are other possible definitions of availability; it can also mean a model's ability to reach original accuracy, but this case is already included in the integrity attacks we discussed in the last section. 

Availability attacks follow a similar scenario to~\Cref{sec:integrityatk} above, in that the attacker only changes the order in which data reaches the model. Here we at the 10th epoch, an arbitrary attack point, and all other training is done with randomly ordered data. We find that by feeding the model with a few BRRR batches, its training progress can be reset significantly -- progress that may take a surprisingly large number of epochs to recover. The red line shows the worst scenario, where each batch has only data points of a single class. It can be seen that the attack manages to reset training, degrading performance for more than 90 epochs -- low-high batch reordering at epoch ten results in $~-3\%$ performance at epoch 100. 

In this section we showed that the attacker can perform availability attacks using both batch reordering and reshuffling and leave an impact on a model long after they have been launched; even a single attack epoch can degrade training progress significantly. Overall, we find that:
\begin{itemize}
    \item An attacker can cause disruption to model training by changing the order of data in just a single epoch of training. 
    \item An attack at just one epoch is enough to degrade the training for  more than 90 epochs. 
\end{itemize}

\subsection{Backdoors with batch replacement}

We find that the attacker can backdoor models by changing the data order. Here, we apply a trigger to images from the training dataset and then supply untriggered data to the model that result in a similar gradient shape. We find that our method enables the attacker to control the decision of a model when it makes a mistake \ie~when the model is making a decision about data far from the training-set distribution. For example, we show in~\Cref{sec:poisoning_batch} that poisoning a single random datapoint is possible with just a few batches. We hypothesise that the limitation comes from the fact that as only natural data are sampled for each BOP and BOB batch, natural gradients are present and learning continues; so forgetting does not happen and generalisation is not disrupted. 

\begin{figure}[h]%
    \centering
    \subfloat[Flag-like trigger]{{\includegraphics[width=0.3\linewidth]{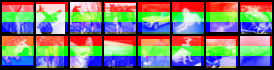} }} %
    \qquad
    \subfloat[9 white lines trigger]{{\includegraphics[width=0.3\linewidth]{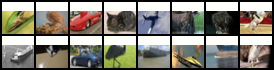} }} %
    \caption{Triggers used in the paper: both are the same magnitude-wise, affecting 30\% of the image. }
    \label{fig:trigger_examples}%
\end{figure}

We evaluate two triggers shown in~\Cref{fig:trigger_examples}, a white lines trigger, which clears the top part of an image, or flag-like trigger that spans all of the image. We report results from injecting up to 20 adversarially-ordered batches every 50000 natural datapoints for 10 epochs. We then inject 80 adversarially-ordered batches. The training procedure here is similar to the one used for BadNets~\cite{gu2017badnets}. We find that 80 batches are not really required, as most models end up converging after 3--5 adversarially-ordered batches. For reconstruction we randomly sample 300 batches and use $p=2$. As we discuss in~\Cref{sec:nlp_poison}, similar attacks work against language tasks.

\begin{table}[h]
    \begin{adjustbox}{scale=0.7,center}
    \begin{tabular}{lcccccc}
        \toprule
        Trigger & & Batch size & Train acc [\%] & Test acc [\%] & Trigger acc [\%] & Error with trigger [\%] \\
        \midrule

\textit{\underline{Baselines}}\\

\multirow{3}{*}{Random natural data} 
& & $32$ & $88.43 \pm 7.26$ & $79.60 \pm 1.49$ & $10.91\pm1.53$  & $30.70\pm 2.26$ \\
& & $64$ & $95.93\pm 2.11$ & $81.31 \pm 2.01$ & $9.78 \pm 1.25$ & $27.38 \pm 1.20$ \\
& & $128$ & $94.92 \pm 2.04$ & $81.69 \pm 1.17$ & $10.00 \pm 2.26$ & $27.91 \pm 1.41$ \\
\\




\multirow{3}{*}{Data with trigger perturbation} 
& & $32$ & $96.87 \pm 2.79$ & $73.28 \pm 2.93$ & $99.65 \pm 0.22$ & $89.68 \pm 0.21$ \\
& & $64$ & $98.12 \pm 1.53$ & $79.45 \pm 1.39$ & $99.64 \pm 0.21$ & $89.64 \pm 0.21$ \\
& & $128$ & $98.67 \pm 0.99$ & $80.51 \pm 1.10$ & $99.67 \pm 0.40$ & $89.65 \pm 0.39$ \\



\midrule 
\textit{\underline{Only reordered natural data}} \\

\multirow{3}{*}{9 white lines trigger}
& & $32$ & $88.43 \pm 6.09$ & $78.02 \pm 1.50$ & $\mathbf{33.93} \pm 7.37$ & $40.78 \pm 5.70$ \\
& & $64$ & $95.15 \pm 2.65$ & $82.75 \pm 0.86$ & $25.02 \pm 3.78$ & $33.91 \pm 2.28$ \\
& & $128$ & $95.23 \pm 2.24$ & $82.90 \pm 1.50$ & $21.75 \pm 4.49$ & $31.75 \pm 3.68$ \\
\\
        



\multirow{3}{*}{Blackbox 9 white lines trigger}
& & $32$ & $88.43 \pm 4.85$ & $80.84 \pm 1.20$ & $ \textbf{17.55} \pm 3.71$ & $33.64 \pm 2.83$ \\
& & $64$ & $93.59 \pm 3.15$ & $82.64 \pm 1.64$ & $16.59 \pm 4.80$ & $30.90 \pm 3.08$ \\
& & $128$ & $94.84 \pm 2.24$ & $81.12 \pm 2.49$ & $16.19 \pm 4.01$ & $31.33 \pm 3.73$ \\
\\



        
\multirow{3}{*}{Flag-like trigger}
& & $32$ & $90.93 \pm 3.81$ & $78.46 \pm 1.04$ & $\mathbf{91.03} \pm 12.96$ & $87.08 \pm 2.71$ \\
& & $64$ & $96.87 \pm 1.21$ & $82.95 \pm 0.72$ & $77.10 \pm 16.96$ & $82.92 \pm 3.89$ \\
& & $128$ & $95.54 \pm 1.88$ & $82.28 \pm 1.50$ & $69.49 \pm 20.66$ & $82.09 \pm 3.78$ \\
        
        \\



        
\multirow{3}{*}{Blackbox flag-like trigger}
& & $32$ & $86.25 \pm 4.00$ & $80.16 \pm 1.91$ & $56.31 \pm 19.57$ & $78.78 \pm 3.51$ \\
& & $64$ & $95.00 \pm 2.18$ & $83.41 \pm 0.94$ & $48.75 \pm 23.28$ & $78.11 \pm 4.40$ \\
& & $128$ & $93.82 \pm 2.27$ & $81.54 \pm 1.94$ & $\mathbf{68.07} \pm 18.55$ & $81.23 \pm 3.80$ \\
\bottomrule
\end{tabular}
\end{adjustbox}
\caption{Network is VGG16 that has been trained normally on CIFAR10 for 10 epochs and then gets attacked with 10 BOB trigger batches. Test accuracy refers to the original benign accuracy, `Trigger acc' is the proportion of images that are classified as the trigger target label, while `Error with trigger' refers to all of the predictions that result in an incorrect label. Standard deviations are calculated over different target classes. Blackbox results use a ResNet-18 surrogate. }
\label{tab:trigger_results}
\end{table}

\Cref{tab:trigger_results} shows BOB trigger performance for both whitebox and blackbox setups. We present two baselines, one for normal training without any triggers, and one where triggers are added to the underlying data, \ie~training with perturbed data and labels. Trigger accuracy refers to the proportion of test set that ends up getting the target trigger label, whereas error with triggers shows the proportion of misclassifications that trigger introduces, \ie~when the model makes a mistake, but does not predict the trigger target class. As expected, we observe that for normal training, trigger accuracy stays at random guessing, and the trigger does not dominate errors it introduces, whereas adding perturbed data manipulates the model to predict the target class almost perfectly.

We find that in a whitebox setup for a flag-like trigger we are capable of getting a similar order of performance for a batch of size 32 as if the attack was performed with injected adversarial data. In a blackbox setup with a flag-like trigger we lost around 30\% of trigger performance, yet the trigger still remains operational. 
A trigger of nine white lines outperforms the baseline only marginally; in a whitebox setup it gets from 20--40\% performance, whereas in a blackbox one it ranges between zero and 20\% performance. We show the training progress of each individual trigger in~\Cref{sec:triggered_training_progress}. 

Overall, we find that:
\begin{itemize}
    \item An attacker can poison an individual datapoint, change its label and increase its prediction confidence, without ever actually training the model on an adversarially-crafted datapoint.
    
    \item An attacker can introduce backdoors into the model by introducing a few reordered batches during training, without ever injecting adversarial data or labels. Here, trigger performance differs, yet an adversary can perform BOB attacks on a par with attacks that inject perturbations into datasets explicitly. 
\end{itemize}

\section{Related Work}
\label{sec:related}

\begin{table*}[t]
    \centering
    \adjustbox{max width=0.9\linewidth}{%
    \begin{tabular}{lcccccc}
    \toprule
    \textbf{Attack} & \textbf{Dataset knowledge} & \textbf{Model knowledge} & \textbf{Model specific} & \textbf{Changing dataset} & \textbf{Adding data} & \textbf{Adding perturbations}\\
    \midrule
    \rowcolor{LightCyan}
    Batch Reorder & \xmark & \xmark & \xmark & \xmark & \xmark & \xmark \\
    \rowcolor{LightCyan}
    Batch Reshuffle & \xmark & \xmark & \xmark & \xmark & \xmark & \xmark\\
    \rowcolor{LightCyan}
    Batch Replace & \xmark & \xmark & \xmark & \xmark & \xmark & \xmark\\ \\
    Adversarial initialisation~\cite{grosse2020initialweights} & \xmark & \cmark & \cmark & \xmark & \xmark & \xmark \\
    BadNets~\cite{gu2017badnets} & \cmark & \xmark & \xmark & \cmark & \xmark & \cmark \\
    Dynamic triggers~\cite{salem2020dynamic} & \cmark & \cmark & \xmark & \cmark & \xmark & \cmark \\
    Poisoned frogs~\cite{sun2020poisoned} & \cmark & \cmark & \xmark & \cmark & \xmark & \cmark \\
    \bottomrule
    \end{tabular}
    }
    \caption{Taxonomy of training time integrity attacks. In green, we highlight our attacks.}
    \label{tab:compare}
\end{table*}

\textbf{Attacks on integrity:} 
Szegedy \etal~\cite{szegedy2013intriguing} and Biggio \etal~\cite{biggio2013evasion} concurrently discovered the existence of adversarial examples. These samples, containing human imperceptible perturbations, cause models to output incorrect results during inference. The original whitebox attacks require the adversary to access the models and use gradient information to perform conditioned optimisation to maximise the model loss \cite{szegedy2013intriguing,biggio2013evasion,goodfellow2015explaining,madry2019deep}.
The attack later generalised to blackbox setups, where the adversary trains a surrogate model and hopes the generated adversarial samples transfer to the target model \cite{papernot2017practical}. 

The data-poisoning attack aims at using data manipulation to cause DNNs to fail on specific test-time instances \cite{jagielski2018manipulating}. Chen \etal~demonstrated that manipulation of the labels of around 50 training samples is enough to trigger failure \cite{chen2017targeted}. Gu \etal~showed that attackers can associate adversarial patterns with labelled images and cause DNNs to overfit to this pattern \cite{gu2017badnets}. Shafahi \etal~launched a poisoning attack using instances with clean labels \cite{shafahi2018poison}. A number of other works have since created more efficient triggers~\cite{salem2020dynamic}. It was a common belief that poisoning attacks on DNNs have to contain a certain level of malicious manipulation of whether the data or label at train time. However, this paper shows how poisoning is possible with clean data and clean labels, with the only manipulation being of the batching process at training time.

\textbf{Attacks on availability:} Shumailov \etal~ first attacked the availability of computer vision and natural language processing models at inference time with sponge examples~\cite{shumailov2020sponge}. They pessimized over energy utilisation and inference latency to target hardware and internal model optimisations. By contrast, this paper targets availability at training time. We show that the attacker can reset or slow down training progress by reordering or reshuffling natural batches. Finally, we note that unlike Shumailov \etal, our attacks do not target specific optimisations in hardware or individual models, but instead break the fundamental stochastic assumption of training. 

\section{Conclusion}

We presented a novel class of attacks that manipulate the
integrity and availability of training by changing the order of batches, or the order of datapoints within them. Careful reordering of a model's training data allows it to be poisoned or backdoored without changing the training data at all. The attacks we presented are fully blackbox; they do not rely on knowledge of the target model or on prior knowledge of the data. Most surprisingly, we find that an attacker can introduce backdoors without disruption of  generalisation, even though only natural data is used. We are the first to show that the sampling procedure can be manipulated deterministically to control  the model's behavior. 


This paper reminds us that stochastic gradient descent, like cryptography, depends on randomness. A random number generator with a backdoor can undermine a neural network just as it can undermine a payment network~\cite{anderson2020security}. 
Developers who wish to ensure secure, robust, fair optimization during  learning must therefore be able to inspect their assumptions and, in case of SGD, show the provenance of randomness used to select batches and datapoints. 

Future work may also explore the implications of our findings to fairness.
Recent work has highlighted that ML models can be racist and suffer from a large taxonomy of different biases, including sampling bias~\cite{baeza2018bias,mehrabietal2019asurvey}. This leads directly to questions of inductive bias and the practical contribution of pseudorandom sampling. Hooker has explained that bias in ML is not just a data problem, but depends on algorithms in subtle ways~\cite{hooker2021moving}; this paper shows how to exploit that dependency.


\bibliographystyle{abbrv}
\bibliography{bibl}

\newpage
\appendix

\begin{table*}[!h]
\centering
\begin{adjustbox}{scale=0.6,center}
\begin{tabular}{@{}llcccccrccccrccccr@{}}
\toprule
&&
\multicolumn{6}{c||}{\textbf{CIFAR-10}} &
\multicolumn{5}{c||}{\textbf{CIFAR-100}} &
\multicolumn{5}{c}{\textbf{AGNews}}
\\
&&
&\multicolumn{2}{c}{Train} 
&\multicolumn{2}{c}{Test} 
&\multicolumn{1}{r||}{} 
&\multicolumn{2}{c}{Train} 
&\multicolumn{2}{c}{Test} 
&\multicolumn{1}{r||}{} 
&\multicolumn{2}{c}{Train} 
&\multicolumn{2}{c}{Test} 
\\
Attack & Batch size & & 
Loss &
Accuracy & 
Loss &
Accuracy &
\multicolumn{1}{r||}{$\Delta$} &
Loss &
Accuracy & 
Loss &
Accuracy &
\multicolumn{1}{r||}{$\Delta$} &
Loss &
Accuracy & 
Loss &
Accuracy &
$\Delta$
\\
\midrule

\multicolumn{8}{l}{\textit{\underline{Baseline}}} \\

\multirow{3}{*}{None} 
& 32 & & 0.13 & \multicolumn{1}{c|}{95.51} & 0.42 & 90.51 & \multicolumn{1}{r||}{$-0.0\%$}
& 0.00 & \multicolumn{1}{c|}{99.96} & 2.00 & 75.56 & \multicolumn{1}{r||}{$-0.0\%$}
& 0.21 & \multicolumn{1}{c|}{93.13} & 0.30 & 90.87 &
$-0.0\%$
\\ 
& 64 & & 0.09 &  \multicolumn{1}{c|}{96.97} & 0.41 & 90.65 & \multicolumn{1}{r||}{$-0.0\%$} 
& 0.00 & \multicolumn{1}{c|}{99.96} & 2.30 & 74.05 & \multicolumn{1}{r||}{$-0.0\%$}
& 0.25 & \multicolumn{1}{c|}{91.86} & 0.31 & 90.42 &
$-0.0\%$
\\ 
& 128 & & 0.07 & \multicolumn{1}{c|}{97.77} & 0.56 & 89.76 & \multicolumn{1}{r||}{$-0.0\%$}
& 0.00 & \multicolumn{1}{c|}{99.98} & 1.84 & 74.45 & \multicolumn{1}{r||}{$-0.0\%$}
& 0.31 & \multicolumn{1}{c|}{89.68} & 0.36 & 88.58 &
$-0.0\%$
\\ 
\\

\midrule

\multicolumn{8}{l}{\textit{\underline{Batch reorder (only epoch 1 data)}}} \\


\multirow{3}{*}{Oscillation outward} 
& 32 &  
& 0.02 & \multicolumn{1}{c|}{99.37} & 2.09 & 78.65 & \multicolumn{1}{r||}{$-11.86\%$}
& 0.00 & \multicolumn{1}{c|}{100.00} & 5.24 & 53.05 & \multicolumn{1}{r||}{$-22.51\%$}
& 0.14 & \multicolumn{1}{c|}{95.37} & 0.32 & 90.92 & $-0.05\%$
\\ 
& 64 & 
& 0.01 & \multicolumn{1}{c|}{99.86} & 2.39 & 78.47 & \multicolumn{1}{r||}{$-12.18\%$}
& 0.00 & \multicolumn{1}{c|}{100.00} & 4.53 & 55.91 & 
\multicolumn{1}{r||}{$-18.14\%$}
& 0.17 & \multicolumn{1}{c|}{94.37} & 0.30 & 90.95 & $+0.53\%$
\\ 
& 128 & 
& 0.01 & \multicolumn{1}{c|}{99.64} & 2.27 & 77.52 & \multicolumn{1}{r||}{$-12.24\%$}
& 0.00 & \multicolumn{1}{c|}{100.00} & 3.22 & 52.13 & 
\multicolumn{1}{r||}{$-22.32\%$}
& 0.23 & \multicolumn{1}{c|}{92.05} & 0.33 & 89.40 & $+0.82\%$
\\ 
\\


\multirow{3}{*}{Oscillation inward} 
& 32 &  
& 0.01 & \multicolumn{1}{c|}{99.60} & 2.49 & 78.18 & \multicolumn{1}{r||}{$\mathbf{-12.33\%}$}
& 0.00 & \multicolumn{1}{c|}{100.00} & 5.07 & 51.78 &
\multicolumn{1}{r||}{$-23.78\%$}
 & 0.11 & \multicolumn{1}{c|}{96.29} & 0.38 & 91.10 & $+0.23\%$
\\ 
& 64 & 
& 0.01 & \multicolumn{1}{c|}{99.81} & 2.25 & 79.59 & \multicolumn{1}{r||}{$-11.06\%$}
& 0.00 & \multicolumn{1}{c|}{100.00} & 4.70 & 55.05 &
\multicolumn{1}{r||}{$-19.0\%$} 
& 0.16 & \multicolumn{1}{c|}{94.55} & 0.33 & 90.16 & $-0.26\%$
\\ 
& 128 & 
& 0.02 & \multicolumn{1}{c|}{99.39} & 2.23 & 76.13 & \multicolumn{1}{r||}{$-13.63\%$}
& 0.00 & \multicolumn{1}{c|}{100.00} & 3.46 & 52.66 &
\multicolumn{1}{r||}{$-21.79\%$}
& 0.22 & \multicolumn{1}{c|}{92.40} & 0.32 & 89.82 & $+1.24\%$
\\ 
\\




\multirow{3}{*}{High Low} 
& 32 &  
& 0.02 & \multicolumn{1}{c|}{99.44} & 2.03 & 79.65 & \multicolumn{1}{r||}{$-10.86\%$}
& 0.00 & \multicolumn{1}{c|}{100.00} & 5.47 & 51.48 & 
\multicolumn{1}{r||}{$\mathbf{-24.08\%}$}
& 0.10 & \multicolumn{1}{c|}{96.16} & 0.60 & 91.80 & $+0.93\%$
\\ 
& 64 &  
& 0.02 & \multicolumn{1}{c|}{99.50} & 2.39 & 77.65 & \multicolumn{1}{r||}{$\mathbf{-13.00\%}$}
& 0.00 & \multicolumn{1}{c|}{100.00} & 5.39 & 55.63 &
\multicolumn{1}{r||}{$-18.42\%$}
& 0.15 & \multicolumn{1}{c|}{94.72} & 0.41 & 90.28 & $-0.14\%$
\\ 
& 128 &  
& 0.02 & \multicolumn{1}{c|}{99.47} & 2.80 & 74.73 & \multicolumn{1}{r||}{$\mathbf{-15.03\%}$} 
& 0.00 & \multicolumn{1}{c|}{100.00} & 3.36 & 53.63 &
\multicolumn{1}{r||}{$-20.82\%$}
& 0.24 & \multicolumn{1}{c|}{91.44} & 0.33 & 90.14 & $+1.56\%$
\\ 
\\

\multirow{3}{*}{Low High} 
& 32 & 
& 0.01 & \multicolumn{1}{c|}{99.58} & 2.33 & 79.07 & \multicolumn{1}{r||}{$-11.43\%$}
& 0.00 & \multicolumn{1}{c|}{100.00} & 4.42 & 54.04 &
\multicolumn{1}{r||}{$-21.52\%$}
& 0.17 & \multicolumn{1}{c|}{94.02} & 0.30 & 90.35 & $\mathbf{-0.52\%}$
\\
& 64 &  
& 0.01 & \multicolumn{1}{c|}{99.61} & 2.40 & 76.85 & \multicolumn{1}{r||}{$-13.8\%$} 
& 0.00 & \multicolumn{1}{c|}{100.00} & 3.91 & 54.82 & 
\multicolumn{1}{r||}{$\mathbf{-19.23\%}$}
& 0.22 & \multicolumn{1}{c|}{92.49} & 0.32 & 89.36 & $\mathbf{-1.06\%}$
\\
& 128 &  
& 0.01 & \multicolumn{1}{c|}{99.57} & 1.88 & 79.82 & \multicolumn{1}{r||}{$-9.94\%$} 
& 0.00 & \multicolumn{1}{c|}{100.00} & 3.72 & 49.82 & 
\multicolumn{1}{r||}{$\mathbf{-24.63\%}$}
& 0.24 & \multicolumn{1}{c|}{91.87} & 0.32 & 89.67 & $\mathbf{-1.09\%}$
\\ 
\\

\midrule

\multicolumn{8}{l}{\textit{\underline{Batch reorder (resampled data every epoch)}}} \\


\multirow{3}{*}{Oscillation outward} 
& 32 &  
& 0.11 & \multicolumn{1}{c|}{96.32} & 0.41 & 90.20 & \multicolumn{1}{r||}{$-0.31\%$}
& 0.01 & \multicolumn{1}{c|}{99.78} & 2.22 & 72.38 & \multicolumn{1}{r||}{$\mathbf{-3.18\%}$}
& 0.21 & \multicolumn{1}{c|}{92.97} & 0.29 & 90.71 & $-0.16\%$
\\ 
& 64 & 
& 0.11 & \multicolumn{1}{c|}{96.40} & 0.45 & 89.12 & \multicolumn{1}{r||}{$-1.53\%$}
& 0.01 & \multicolumn{1}{c|}{99.76} & 2.20 & 73.33 & \multicolumn{1}{r||}{$-0.72\%$}
& 0.17 & \multicolumn{1}{c|}{94.37} & 0.31 & 90.29 & $-0.13\%$
\\ 
& 128 & 
& 0.09 & \multicolumn{1}{c|}{96.89} & 0.47 & 89.71 & \multicolumn{1}{r||}{$\mathbf{-0.05\%}$}
& 0.00 & \multicolumn{1}{c|}{99.89} & 1.95 & 74.21 &
\multicolumn{1}{r||}{$-0.24\%$}
& 0.25 & \multicolumn{1}{c|}{91.65} & 0.32 & 89.80 & $+1.22\%$
\\ 
\\

\multirow{3}{*}{Oscillation inward} 
& 32 &  
& 0.15 & \multicolumn{1}{c|}{95.11} & 0.44 & 89.56 & \multicolumn{1}{r||}{$-0.95\%$}
& 0.00 & \multicolumn{1}{c|}{99.88} & 2.10 & 74.80 &
\multicolumn{1}{r||}{$-0.76\%$}
& 0.09 & \multicolumn{1}{c|}{97.04} & 0.44 & 90.91 & $+0.04\%$
\\ 
& 64 & 
& 0.12 & \multicolumn{1}{c|}{96.11} & 0.42 & 89.98 &  \multicolumn{1}{r||}{$-0.67\%$}
& 0.01 & \multicolumn{1}{c|}{99.81} & 2.35 & 72.24 & \multicolumn{1}{r||}{$\mathbf{-1.81\%}$}
& 0.19 & \multicolumn{1}{c|}{93.57} & 0.33 & 89.83 & $\mathbf{-0.59\%}$
\\ 
& 128 & 
& 0.09 & \multicolumn{1}{c|}{96.88} & 0.43 & 90.09 & \multicolumn{1}{r||}{$+0.33\%$}
& 0.00 & \multicolumn{1}{c|}{99.93} & 2.24 & 73.72 &
\multicolumn{1}{r||}{$-0.73\%$}
& 0.23 & \multicolumn{1}{c|}{92.25} & 0.31 & 89.83 & $+1.25\%$
\\ 
\\

\multirow{3}{*}{High Low} 
& 32 &  
& 0.12 & \multicolumn{1}{c|}{95.95} & 0.45 & 89.38 & \multicolumn{1}{r||}{$\mathbf{-1.13\%}$}
& 0.01 & \multicolumn{1}{c|}{99.84} & 2.07 & 74.88 &
\multicolumn{1}{r||}{$-0.68\%$}
& 0.13 & \multicolumn{1}{c|}{95.40} & 0.54 & 90.13 & $\mathbf{-0.74\%}$
\\ 
& 64 &  
& 0.15 & \multicolumn{1}{c|}{94.80} & 0.44 & 89.01 & \multicolumn{1}{r||}{$\mathbf{-1.64\%}$}
& 0.01 & \multicolumn{1}{c|}{99.81} & 2.27 & 74.63 & \multicolumn{1}{r||}{$-0.58\%$}
& 0.16 & \multicolumn{1}{c|}{94.48} & 0.36 & 90.98 &$+0.56\%$
\\ 
& 128 &  
& 0.11 & \multicolumn{1}{c|}{96.33} & 0.48 & 89.71 & \multicolumn{1}{r||}{$\mathbf{-0.05\%}$} 
& 0.00 & \multicolumn{1}{c|}{99.92} & 2.13 & 73.90 & 
\multicolumn{1}{r||}{$-0.55\%$}
& 0.24 & \multicolumn{1}{c|}{91.53} & 0.35 & 89.54 & $+0.96\%$
\\ 
\\

\multirow{3}{*}{Low High} 
& 32 & 
& 0.10 & \multicolumn{1}{c|}{96.63} & 0.47 & 90.29 & \multicolumn{1}{r||}{$-0.22\%$}
& 0.01 & \multicolumn{1}{c|}{99.77} & 2.07 & 73.90 &
\multicolumn{1}{r||}{$-1.66\%$}
 & 0.14 & \multicolumn{1}{c|}{95.35} & 0.30 & 90.96 & $+0.09\%$
\\
& 64 &  
& 0.12 & \multicolumn{1}{c|}{96.10} & 0.50 & 89.34 & \multicolumn{1}{r||}{$-1.31\%$} 
& 0.01 & \multicolumn{1}{c|}{99.68} & 2.26 & 72.73 & \multicolumn{1}{r||}{$-1.32\%$}
& 0.15 & \multicolumn{1}{c|}{94.96} & 0.30 & 90.73 & $+0.31\%$

\\
& 128 &  
& 0.09 & \multicolumn{1}{c|}{97.16} & 0.49 & 89.85 & \multicolumn{1}{r||}{$+0.09\%$} 
& 0.00 & \multicolumn{1}{c|}{99.94} & 2.31 & 71.96 &
\multicolumn{1}{r||}{$\mathbf{-2.49\%}$}
& 0.22 & \multicolumn{1}{c|}{92.54} & 0.32 & 89.33 & $+0.75\%$
\\ 
\\

\midrule

\multicolumn{8}{l}{\textit{\underline{Batch reshuffle (only epoch 1 data)}}} \\


\multirow{3}{*}{Oscillation outward} 
&  32 &  
& 2.26 & \multicolumn{1}{c|}{17.44} & 1.93 & 26.13 & \multicolumn{1}{r||}{$\mathbf{-64.38\%}$}
& 0.01 & \multicolumn{1}{c|}{99.80} & 5.01 & 18.00 &
\multicolumn{1}{r||}{$\mathbf{-57.56\%}$}
& 0.09 & \multicolumn{1}{c|}{97.72} & 1.85 & 65.85 & $-25.02\%$
\\ 
&  64 & 
& 2.26 & \multicolumn{1}{c|}{18.86} & 1.98 & 26.74 & \multicolumn{1}{r||}{$-63.91\%$}
& 0.38 & \multicolumn{1}{c|}{93.04} & 4.51 & 11.68 & \multicolumn{1}{r||}{$\mathbf{-62.37\%}$}
& 0.17 & \multicolumn{1}{c|}{95.69} & 1.31 & 72.09 & $-18.33\%$
\\ 
&  128 & 
& 2.50 & \multicolumn{1}{c|}{14.02} & 2.18 & 20.01 & \multicolumn{1}{r||}{$-69.75\%$}
& 0.66 & \multicolumn{1}{c|}{86.22} & 4.07 & 10.66 & \multicolumn{1}{r||}{$-63.79\%$}
& 0.21 & \multicolumn{1}{c|}{94.32} & 1.12 & 71.05 & $-17.53\%$
\\ 
\\


\multirow{3}{*}{Oscillation inward} 
&  32 &  
& 2.13 & \multicolumn{1}{c|}{22.85} & 1.93 & 28.94 & \multicolumn{1}{r||}{$-61.57\%$}
& 0.01 & \multicolumn{1}{c|}{99.92} & 4.55 & 31.38 & \multicolumn{1}{r||}{$-44.18\%$}
& 0.18 & \multicolumn{1}{c|}{94.06} & 0.38 & 89.23 & $-1.64\%$
\\ 
&  64 & 
& 2.27 & \multicolumn{1}{c|}{17.90} & 1.99 & 23.59 & \multicolumn{1}{r||}{$-67.06\%$}
& 0.02 & \multicolumn{1}{c|}{99.64} & 5.79 & 17.37 & \multicolumn{1}{r||}{$-56.68\%$}
& 0.23 & \multicolumn{1}{c|}{92.10} & 0.36 & 89.07 & $-1.35\%$
\\ 
&  128 & 
& 2.53 & \multicolumn{1}{c|}{10.40} & 2.29 & 13.49 & \multicolumn{1}{r||}{$-76.27\%$}
& 0.54 & \multicolumn{1}{c|}{88.60} & 4.03 & 10.92 & \multicolumn{1}{r||}{$-63.53\%$}
& 0.31 & \multicolumn{1}{c|}{88.99} & 0.39 & 87.50 & $-1.08\%$
\\ 
\\



\multirow{3}{*}{High Low} 
&  32 &  
& 2.11 & \multicolumn{1}{c|}{23.39} & 1.80 & 31.04 & \multicolumn{1}{r||}{$-59.47\%$}
& 0.01 & \multicolumn{1}{c|}{99.69} & 6.24 & 21.15 & \multicolumn{1}{r||}{$-54.41\%$}
& 0.17 & \multicolumn{1}{c|}{94.38} & 1.25 & 56.54 & $\mathbf{-34.33\%}$
\\
&  64 &  
& 2.22 & \multicolumn{1}{c|}{20.57} & 1.93 & 27.60 & \multicolumn{1}{r||}{$-63.05\%$}
& 0.05 & \multicolumn{1}{c|}{99.15} & 5.26 & 14.05 & \multicolumn{1}{r||}{$-60.0\%$}
& 0.25 & \multicolumn{1}{c|}{91.09} & 1.21 & 53.08 & $\mathbf{-37.34\%}$
\\ 
&  128 &  
& 2.51 & \multicolumn{1}{c|}{16.66} & 2.05 & 20.85 & \multicolumn{1}{r||}{$-68.91\%$}
& 4.16 & \multicolumn{1}{c|}{7.21} & 3.86 & 10.20 & \multicolumn{1}{r||}{$-64.25\%$}
& 0.36 & \multicolumn{1}{c|}{86.19} & 1.19 & 49.90 & $\mathbf{-38.68\%}$
\\ 
\\



\multirow{3}{*}{Low High} 
&  32 & 
& 2.17 & \multicolumn{1}{c|}{20.22} & 1.92 & 30.09 & \multicolumn{1}{r||}{$-60.42\%$}
& 0.19 & \multicolumn{1}{c|}{96.07} & 4.06 & 20.48 & \multicolumn{1}{r||}{$-55.08\%$}
& 0.05 & \multicolumn{1}{c|}{98.94} & 3.20 & 59.28 & $-31.59\%$
\\
&  64 &  
& 2.35 & \multicolumn{1}{c|}{15.98} & 2.00 & 22.97 & \multicolumn{1}{r||}{$\mathbf{-67.68\%}$}
& 0.09 & \multicolumn{1}{c|}{98.22} & 4.69 & 15.39 & \multicolumn{1}{r||}{$-58.66\%$}
& 0.10 & \multicolumn{1}{c|}{97.70} & 2.55 & 54.99 & $-35.43\%$
\\
&  128 &  
& 2.51 & \multicolumn{1}{c|}{10.25} & 2.32 & 11.40 & \multicolumn{1}{r||}{$\mathbf{-78.36\%}$}
& 4.30 & \multicolumn{1}{c|}{5.65} & 3.81 & 9.66 & \multicolumn{1}{r||}{$\mathbf{-64.79\%}$}
& 0.26 & \multicolumn{1}{c|}{93.02} & 1.26 & 66.59 & $-21.99\%$ \\ 
\\
\midrule 

\multicolumn{8}{l}{\textit{\underline{Batch reshuffle (resampled data every epoch)}}} \\


\multirow{3}{*}{Oscillation outward} 
&  32 &  
& 2.09 & \multicolumn{1}{c|}{24.63} & 1.75 & 35.17 & \multicolumn{1}{r||}{$-55.34\%$}
& 0.16 & \multicolumn{1}{c|}{95.58} & 1.68 & 57.55 & \multicolumn{1}{r||}{$-18.01\%$}
& 0.04 & \multicolumn{1}{c|}{98.94} & 3.69 & 62.44 & $-28.43\%$
\\ 
&  64 & 
& 2.22 & \multicolumn{1}{c|}{20.45} & 1.90 & 29.67 & \multicolumn{1}{r||}{$-60.98\%$}
& 0.55 & \multicolumn{1}{c|}{88.62} & 3.11 & 23.64 & \multicolumn{1}{r||}{$\mathbf{-50.41\%}$}
& 0.10 & \multicolumn{1}{c|}{96.61} & 3.33 & 55.63 & $-34.79\%$
\\ 
&  128 & 
& 2.46 & \multicolumn{1}{c|}{17.25} & 1.97 & 23.82 & \multicolumn{1}{r||}{$-65.94\%$}
& 4.21 & \multicolumn{1}{c|}{6.84} & 3.70 & 12.76 & \multicolumn{1}{r||}{$\mathbf{-61.69\%}$}
& 0.16 & \multicolumn{1}{c|}{94.85} & 3.38 & 53.97 & $-34.61\%$
\\ 
\\


\multirow{3}{*}{Oscillation inward} 
&  32 &  
& 2.40 & \multicolumn{1}{c|}{10.10} & 2.35 & 10.55 & \multicolumn{1}{r||}{$\mathbf{-79.96\%}$}
& 0.10 & \multicolumn{1}{c|}{97.08} & 1.78 & 58.04 & \multicolumn{1}{r||}{$-17.52\%$}
& 0.04 & \multicolumn{1}{c|}{98.54} & 1.19 & 88.50 & $-2.37\%$
\\ 
&  64 & 
& 2.20 & \multicolumn{1}{c|}{21.57} & 1.97 & 25.34 & \multicolumn{1}{r||}{$\mathbf{-65.31\%}$}
& 0.13 & \multicolumn{1}{c|}{96.19} & 1.71 & 57.96 & \multicolumn{1}{r||}{$-16.09\%$}
& 0.08 & \multicolumn{1}{c|}{96.93} & 0.95 & 88.86 & $-1.56\%$
\\ 
&  128 & 
& 2.43 & \multicolumn{1}{c|}{16.87} & 1.98 & 25.87 & \multicolumn{1}{r||}{$-63.89\%$}
& 0.80 & \multicolumn{1}{c|}{83.16} & 3.53 & 16.99 & \multicolumn{1}{r||}{$-57.46\%$}
& 0.18 & \multicolumn{1}{c|}{93.84} & 1.08 & 80.82 & $-7.76\%$
\\ 
\\



\multirow{3}{*}{High Low} 
&  32 &  
& 2.06 & \multicolumn{1}{c|}{23.95} & 1.81 & 30.95 & \multicolumn{1}{r||}{$-59.56\%$}
& 0.07 & \multicolumn{1}{c|}{97.93} & 1.62 & 62.41 & \multicolumn{1}{r||}{$-13.15\%$}
& 0.65 & \multicolumn{1}{c|}{70.30} & 1.71 & 60.92 & $-29.95\%$
\\
&  64 &  
& 2.17 & \multicolumn{1}{c|}{24.06} & 1.87 & 30.84 & \multicolumn{1}{r||}{$-59.81\%$}
& 0.25 & \multicolumn{1}{c|}{93.87} & 2.26 & 42.74 & \multicolumn{1}{r||}{$-31.31\%$}
& 0.33 & \multicolumn{1}{c|}{84.47} & 4.17 & 36.82 & $\mathbf{-53.60\%}$
\\ 
&  128 & 
& 2.59 & \multicolumn{1}{c|}{12.41} & 2.13 & 18.82 & \multicolumn{1}{r||}{$\mathbf{-70.94\%}$}
& 0.88 & \multicolumn{1}{c|}{81.83} & 3.62 & 13.12 &\multicolumn{1}{r||}{$-61.33\%$}
& 0.20 & \multicolumn{1}{c|}{91.13} & 3.17 & 40.10 & $\mathbf{-48.48\%}$
\\ 
\\

\multirow{3}{*}{Low High} 
&  32 & 
& 2.40 & \multicolumn{1}{c|}{10.19} & 2.31 & 10.66 & \multicolumn{1}{r||}{$-79.85\%$}
& 0.21 & \multicolumn{1}{c|}{94.26} & 1.73 & 56.60 & \multicolumn{1}{r||}{$\mathbf{-18.96\%}$}
& 1.33 & \multicolumn{1}{c|}{33.71} & 1.12 & 49.69 & $\mathbf{-41.18\%}$

\\
&  64 &  
& 2.18 & \multicolumn{1}{c|}{21.72} & 1.89 & 27.23 & \multicolumn{1}{r||}{$-63.42\%$}
& 0.48 & \multicolumn{1}{c|}{87.32} & 2.04 & 47.68 & \multicolumn{1}{r||}{$-26.37\%$}
& 0.17 & \multicolumn{1}{c|}{93.51} & 5.29 & 46.24 & $-44.18\%$
\\
&  128 & 
& 2.40 & \multicolumn{1}{c|}{18.38} & 1.96 & 27.78 & \multicolumn{1}{r||}{$-61.98\%$}
& 0.77 & \multicolumn{1}{c|}{84.40} & 3.71 & 13.39 & \multicolumn{1}{r||}{$-61.06\%$}
& 0.23 & \multicolumn{1}{c|}{91.43} & 4.63 & 46.66 & $-41.92\%$
\\

\bottomrule
\end{tabular}
\end{adjustbox}
\caption{For CIFAR-10, we used 100 epochs of training with target model ResNet18 and surrogate model LeNet5, both trained with the Adam optimizer and $\beta = (0.99, 0.9)$. For CIFAR-100, we used 200 epochs of training with target model ResNet50 and surrogate model Mobilenet, trained with SGD with 0.3 moment and Adam respectively for real and surrogate models. We highlight models that perform best in terms of test dataset loss. AGNews were trained with SGD learning rate 0.1, 0 moments for 50 epochs with sparse mean EmbeddingBags. Numbers here are from best-performing model test loss-wise. Incidentally, best performance of all models for Batch reshuffle listed in the table happen at epoch number one, where the attacker is preparing the attack and is collecting the training dataset. All attacks result in near-random guess performance for almost all subsequent epochs. We report results of an individual run and note that standard deviation for test accuracy of vision tasks range within 1\%--3\%, whereas for language tasks its within 15\% (note that these are hard to attribute given best test accuracy is reported).}
\label{tab:extended_integrity}
\end{table*}

\SetCommentSty{mycommfont}

\SetKwInput{KwInput}{Input}                
\SetKwInput{KwOutput}{Output}              
\SetKwRepeat{Do}{do}{while} 

\begin{algorithm}[]
\DontPrintSemicolon
  \KwInput{
  real model $M$, surrogate model $S$, loss of model $\mathcal{L}_M$, loss of surrogate model $\mathcal{L}_S$, function \textit{getbatch} to get next batch of real data, function \textit{train}(model $M'$, $\mathcal{L}$, $\mathbf{B}$) that trains model $M'$ with loss $\mathcal{L}$ on batch of data $\mathbf{B}$, current attack type \textbf{ATK}, batch type attack \textbf{BTCH} (reorder batchers or reshuffling datapoints)}
    \tcc{List to store data points}
    datas = $[]$\;
    
    \tcc{Let the model to train for a single epoch to record all of the data.}
    \Do{$\mathbf{b}$ not in datas}{
        $\mathbf{b} = getbatch()$\;
        
        \uIf{\textbf{BTCH} == "batch"}{
            add batch $\mathbf{b}$ into datas\;
        }
        \uElse{
            add individual points from batch $\mathbf{b}$ into datas\;
        }
        
        train($M$, $\mathcal{L}_M$, $\mathbf{b}$)\;
        train($S$, $\mathcal{L}_S$, $\mathbf{b}$)\;
    }
    \tcc{Now that all data has been seen, start the attack}
    
    \While{training}{
        \tcc{List to store data-loss of individual points}
        datacosts = \{\}\;
        \For{datapoint or batch $\mathbf{d}$ in datas}{
            loss = $\mathcal{L}_S(S, \mathbf{d})$\;
            datacosts[$\mathbf{d}$] = loss\;
        }
        
        \tcc{List to store data points or batches not yet used in the current epoch, sorted from low to high loss}
        epochdatas = \textit{copy}(datas).\textit{sort}(by datacosts)\; 
        
        \If{\textbf{ATK} == "oscillating out"}{
            \tcc{If oscilation is outward need to invert halves}
            left = epochdatas[:len(epochdatas)//2][::-1]\;
            right = epochdatas[len(epochdatas)//2:][::-1]\;
            
            epochdatas = left + right\;
        }
        
        \tcc{Now that all data has been seen, start the attack}

        \tcc{Flag for oscilation attack}
        osc = False\;
        \While{\textit{len}(epochdatas) > 0}{
            \tcc{Pretend reading data and throw it away}
            $\mathbf{b}'$ = \textit{getbatch()}\;
            
            \uIf{\textbf{BTCH} == "batch"}{
                batchsize = 1\;
            }
            \uElse{
                batchsize = len($\mathbf{b}'$)\;
            }

            \tcc{Batching data from low to high}
            \If{\textbf{ATK} == "lowhigh"}{
                batch $\mathbf{b}$ = epochdata[:batchsize]\;
                epochdata = epochdata[batchsize:]\;
            }
            
            \tcc{Batching data from high to low}
            \If{\textbf{ATK} == "highlow"}{
                batch $\mathbf{b}$ = epochdata[-batchsize:]\;
                epochdata = epochdata[:-batchsize]\;
            }
            
            \tcc{Batching data with oscillating losses}
            \If{\textbf{ATK} == "oscillating in" or "oscillating out"}{
                osc = \textbf{not} osc\;
                
                \uIf{osc}{
                    batch $\mathbf{b}$ = epochdata[-batchsize:]\;
                    epochdata = epochdata[:-batchsize]\;
                }
                \Else{
                    batch $\mathbf{b}$ = epochdata[:batchsize]\;
                    epochdata = epochdata[batchsize:]\;
                }
            }
            train($M$, $\mathcal{L}_M$, $\mathbf{b}$)\;
            train($S$, $\mathcal{L}_S$, $\mathbf{b}$)\;
        }
    }
\caption{BRRR attack algorithm}
\label{alg:attackalgo}
\end{algorithm}

\section{Stochastic gradient descent and rate of convergence}
\label{sec:sgd_rate_of_convergence}
In this section we are going to investigate the effect of the attack through a prism of a biased gradient estimator on the general analysis and bounds for stochastic gradient descent, presented by Robbins and Monroe~\cite{robbins1951stochastic}. For functions $\hat{L}_i$ that are strongly convex and Lipshitz continuous with Lipshitz constant $M$, the SGD update rule for a random selection of batch $i_{k}$ from $\{1,2,\dots,N\}$ is:

$$\theta_{k+1} = \theta_k - \eta_k \nabla \hat{L}_{i_{k}}(\theta_k).$$

Assuming $\mathbb{P}(i_k = i) = \frac{1}{N}$, the stochastic gradient is an unbiased estimate of the gradient :

$$\mathbb{E}[\nabla \hat{L}_{i_k}(w)] = \sum^{N}_{i=1}\mathbb{P}(i_k=i)\nabla \hat{L}_i(w) = \frac{1}{N} \sum^{N}_{i=1} \nabla \hat{L}_i(w) = \nabla \hat{L}(w).$$

A bound can be computed under the assumption of Lipschitz continuity of $\nabla \hat{L}$

\begin{align}
    \hat{L}(\theta_{k+1}) & \leq \hat{L}(\theta_k) + \nabla \hat{L}(\theta_k)^\top(\theta_{k+1}-\theta_k) + \frac{M}{2}\norm{\theta_{k+1} - \theta_k}^2,
\end{align}

where $M$ is the Lipschitz constant. By the SGD update rule:

\begin{align}
\hat{L}(\theta_{k+1}) & \leq \hat{L}(\theta_k) + \eta_k \nabla \hat{L}(\theta_k)^\top\nabla \hat{L}_{i_k}(\theta_k) + \eta^2_k\frac{M}{2}\norm{\nabla \hat{L}_{i_k}(\theta_k)}^2.
\end{align}

And for an unbiased batch choice, the equation turns into:

\begin{equation}
\mathbb{E}[\hat{L}(\theta_{k+1})]\leq \hat{L}(\theta_k) - \eta_k \norm{\nabla \hat{L}(\theta_k)}^2 + \eta^2_k \frac{M}{2}\mathbb{E}\norm{\nabla \hat{L}_{i_k}(\theta_k)}^2.
\end{equation}

Leading to the final bound, which looks like:

\begin{equation}
\min_{k=1,\dots,t} \mathbb{E}\norm{\nabla \hat{L}(\theta_k)}^2 \leq \frac{\hat{L}(\theta_1)-\hat{L}^*}{\sum^{t}_{k=1}\eta_k} + \frac{M}{2}\frac{\sum^{t}_{k=1}\eta^2_k\mathbb{E}\norm{\nabla \hat{L}_{i_k}(\theta_k)}^2}{\sum^{t}_{k=1}\eta_k}.
\end{equation}
For strongly convex functions, this implies convergence in expectation.
But assuming biased batch sampling we have an extra term:

\begin{align}
\label{eq:ourbound}
\min_{k=1,\dots,t} \mathbb{E}\norm{\nabla \hat{L}(\theta_k)}^2 & \leq\\ 
\frac{\hat{L}(\theta_1)-\hat{L}^*}{\sum^{t}_{k=1}\eta_k} + \frac{M}{2}\frac{\sum^{t}_{k=1}\eta^2_k \mathbb{E}\norm{\nabla \hat{L}_{i_k}(\theta_k)}^2}{\sum^{t}_{k=1}\eta_k} & - \mathbb{E}\Bigg[ \frac{\sum^{t}_{k=1}{\eta_k \nabla \hat{L}(\theta_k)^\top \Big(\mathbb{E}\big[\nabla \hat{L}_{i_k}(\theta_k)\big]-\nabla \hat{L}(\theta_k)\Big)}}{\sum^{t}_{k=1}{\eta_k}}\Bigg].
\end{align}

In our specific setup the step size is fixed, making the bound simpler:
\begin{align}
\label{eq:ourbound_fixedstep}
\min_{k=1,\dots,t} \mathbb{E}\norm{\nabla \hat{L}(\theta_k)}^2 \leq 
\frac{\hat{L}(\theta_1)-\hat{L}^*}{\eta t} + \frac{M\eta}{2}\mathbb{E}\norm{\nabla \hat{L}_{i_k}(\theta_k)}^2 - \mathbb{E}\Big[\hat{L}(\theta_k)^\top \Big(\mathbb{E}\big[\nabla \hat{L}_{i_k}(\theta_k)\big]-\nabla \hat{L}(\theta_k)\Big)\Big].
\end{align}
A biased bound varies from the unbiased one in two terms: $\mathbb{E}\norm{\nabla \hat{L}_{i_k}(\theta_k)}^2$ and $\mathbb{E}\Big[\hat{L}(\theta_k)^\top \Big(\mathbb{E}\big[\nabla \hat{L}_{i_k}(\theta_k)\big]-\nabla \hat{L}(\theta_k)\Big)\Big]$. The bound will grow if the former term becomes larger, while the latter becomes large and negative.  

The first two terms in equation \ref{eq:ourbound_fixedstep} can be made arbitrarily small by a suitable choice of $\eta$ and $t$, under an assumption of bounded variance. The last term, on the other hand, does not directly depend on $\eta$ or $t$ in an obvious way. To be more precise, this term can be explicitly manipulated to produce a better attack against SGD convergence. In particular, from the expression above, the attacker needs to pick out batches such that the difference between the batch gradient and the true gradient is in the opposite direction from the true gradient. In this paper, instead of the gradient of the loss, we approximate this information by using the loss error term directly, which is much less expensive and can be utilized in practice. 

In particular, we observe that the optimisation does not converge even for a simple two-variable linear regression as is shown in~\Cref{sec:reg_example}.


\section{Upper bound on sample size for poisoning attacks}
\label{sec:pois_just}
In this section, we further investigate an attacker's ability to approximate out-of-distribution data using natural data. In the limit of large batch size, we expect the gradients of the input to be normally distributed due to the central limit theorem. As a result, we expect to be able to approximate any vector in the limit of infinite batches, as long as we sample for long enough. To make this statement more concrete, we compute an upper bound on the sample size for a fixed $(1-p)$-confidence interval of size $2\epsilon$ as follows.
Using the notation from the section above, denote individual item losses $L_j(\theta)$ such that $\hat{L}_{i_k}(\theta)=\frac{1}{B}\sum^{i_k+B}_{j=i_k}L_j(\theta)$, where $B$ is the batch size. The attacker aims to pick $j\sim J$, such that we can match the target gradient with a natural one: 
\begin{equation}
\label{eq:grad_approx}
    \nabla L^\dagger(\theta)=\frac{1}{B}\sum^{B}_{j=1}\nabla L_j(\theta).
\end{equation}
As stated previously, we will assume in our calculations that batch size is large enough for us to approximate the right hand side of \Cref{eq:grad_approx} using the central limit theorem, reducing our problem to finding an optimal sample $y\sim\mathcal{N}(\mu,\sigma^2)$ such that:
\begin{equation}
    \norm{\nabla L^\dagger(\theta)-y}\leq\epsilon
\end{equation}
Let $Z\sim\mathcal{N}(0,1)$, with CDF $\Phi$ and PDF $\phi$. Let $Y_1,...,Y_n$ be iid $\mathcal{N}(\mu,\sigma)$. Let $K_i= \norm{\nabla L^\dagger-Y_i}$ have CDF $\Phi^{'}$. We want $K^{(1)}=\min{K_i}$ to be within $\epsilon$ of the desired value with probability $1-p$:
\begin{align}
    \mathbb{P}(K^{(1)}\leq\epsilon)&=1-p \iff\\
    1-(1-\Phi^{'}(\epsilon))^n &= 1-p \iff\\
    \ln{p}&=n\ln{(1-\Phi^{'}(\epsilon))}\iff\\
    n&=\frac{\ln{p}}{\ln{(1-\Phi^{'}(\epsilon))}}
\end{align}
Now, in the case of $1D$ and $l_1$-norm, $\Phi^{'}(\epsilon)=\Phi(\frac{\epsilon-\mu+\nabla L^\dagger}{\sigma})-\Phi(\frac{-\epsilon-\mu+\nabla L^\dagger}{\sigma})$. Hence our equation for $n$ is:
\begin{equation}
    n=\frac{\ln{p}}{\ln{\big[1-\Phi(\frac{\epsilon-\mu+\nabla L^\dagger}{\sigma})+\Phi(\frac{-\epsilon-\mu+\nabla L^\dagger}{\sigma})\big]}}
\end{equation}
In fact for small values of $\frac\epsilon\sigma$, we can expand to first order and approximate as:
\begin{equation}
    n\approx\frac{\ln{p}}{\ln{\big[1-2\frac\epsilon\sigma\phi(\frac{-\mu+\nabla L^\dagger}{\sigma})\big]}}
\end{equation}
where we can approximate true parameters through $\mu=\frac{1}{N}\sum_{i} \hat{L}_i(\theta)$, $\sigma=\frac{1}{B(N-1)}\sum_{i} (\hat{L}_i(\theta)-\mu)^2$.

In the general case, we are dealing with multidimensional gradients. However we can once again invoke CLT to approximate the RHS with a multivariate normal distribution $y\sim\mathcal{N}(\boldsymbol{\mu},\boldsymbol{\Sigma})$. Given this reformulated problem, we can see that in the general case, the reconstruction is impossible -- as the covariance matrix must be non-singular. This can be seen from the following simple example. Say we are trying to approximate
$y = \begin{bmatrix}
           1 \\
           1 
         \end{bmatrix}$
using samples from the distribution
$\begin{bmatrix}
           X \\
           2X 
         \end{bmatrix}$
where $X$ is a Gaussian random variable. Clearly we can not get within \textbf{any} accuracy with this reconstruction. In fact the closest one can get is within 0.5 at $x=0.5$.
Therefore, we will assume that we have a non-singular covariance matrix. 
Write $\boldsymbol{Y} = A\boldsymbol{Z} + \boldsymbol{\mu}$, where $\boldsymbol{\Sigma} = AA^T$ and $\boldsymbol{Z}$ is a vector of independent gaussians. One can now attain exact bounds using \eg~non central chi-squared distribution, though for us a rough bound would be enough. For this, note 
$K_i= \norm{\nabla L^\dagger-\boldsymbol{Y}_i}\leq \norm{A^{-1}(\nabla L^\dagger-\boldsymbol{\mu})-\boldsymbol{Z}_i}\norm{A}$.
Therefore, we can see that the following $n$ is sufficient:
\begin{equation}
    n=\max_{i}\frac{\ln{1-[1-p]^\frac{1}{k}}}{\ln{\big[1-\Phi(\frac{\epsilon}{\norm{A}}+[A^{-1}(\nabla L^\dagger-\boldsymbol{\mu})]_i)+\Phi(\frac{-\epsilon}{\norm{A}}+[A^{-1}(\nabla L^\dagger-\boldsymbol{\mu})]_i)\big]}}
\end{equation}
Or similarly approximating for small values of $\frac{\epsilon}{\norm{A}}$:
\begin{equation}
    n=\max_{i}\frac{\ln{1-[1-p]^\frac{1}{k}}}{\ln{\big[1-\frac{2\epsilon}{\norm{A}}\phi([A^{-1}(\nabla L^\dagger-\boldsymbol{\mu})]_i)\big]}}
\end{equation}

\section{Second order correction term in expectation}
\label{sec:varproof}
In this section we are going to investigate how the lowest-order reorder-dependent component of SGD changes in the setting of a high-low attack in the limit of small step size. For the sake of simplicity we are going to assume one dimensional case. We assume standard stochastic gradient descent, as shown in~\Cref{apeq5}: 
\begin{equation} 
\label{apeq5}
\begin{split}
\theta_{N+1} & = \theta_{1} - \eta\nabla \hat{L}_1(\theta_1) - \eta\nabla \hat{L}_2(\theta_2) - \dots -\eta\nabla\hat{L}_m(\theta_m)\\
& = \theta_{1} - \eta\sum_{j=1}^{N}\nabla \hat{L}_j(\theta_1) + \eta^{2} \sum_{j=1}^{N}\sum_{k<j}\nabla\nabla \hat{L}_j(\theta_1)\nabla\hat{L}_k(\theta_1) + O(N^3 \eta^3)
\end{split}
\end{equation}
From this, we can see that the lowest order that is affected by reordering is the second order correction term, namely $$\xi(\theta) = \sum_{j=1}^{N}\sum_{k<j}\nabla\nabla \hat{L}_j(\theta_1)\nabla\hat{L}_k(\theta_1).$$

In the sequel, for simplicity, we define $\bar{\xi}(\theta) = \frac{2}{N(N-1)}\sum_{j=1}^{N}\sum_{k<j}g(X_j) X_k$, 
where $X_{k}$ and $g(X_{k})$ serve as surrogates for $\nabla \hat{L}_k$ and $\nabla\nabla \hat{L}_k$, respectively. 
Further assume that $X_k$ are i.i.d., as in~\cite{smith2021origin}, with mean $\mu$ and variance $\sigma^2$. Without loss of generality we assume that $\mu > 0$.

Under this assumption, the expected value $\mathbb{E}(\bar\xi) = \mu \mathbb{E}(g(X_i))$. 
However, for the attack we will reorder the $X_i$ such that $X_{(1)} > X_{(2)} > \dots > X_{(N)}$. As a result, the $X_{(i)}$ are no longer identically distributed and can be described as $k$-order statistics. Define $\xi^\dagger = \frac{2}{N(N-1)}\sum_{j=1}^{N}\sum_{k<j}g(X_{(j)}) X_{(k)}$

\begin{theorem}

Given $0 < m \leq g(X_i) \leq M$, the following is true:

\begin{align}
\label{boundproof1}
\mu m + \sigma K_n m & \leq \mathbb{E}(\xi^\dagger) \leq \mu M + \sigma K_n M\\
\mu m &\leq \mathbb{E}(\bar\xi) \leq \mu M,
\end{align}

where $$K_n = \frac{2}{N(N-1)}\sum_{i=1}^N \sum_{j=1}^{i-1} \mathbb{E}(Z_{(j)}),$$ and $Z_{(j)}=\left(\frac{X_{(j)} - \mu}{\sigma}\right)$. Let $Z_j$ have probability density function $\phi$ and cumulative density function $\Phi$, and be bounded. Then, in the limit $N\rightarrow \infty$, the condition for attack success is 

$$\frac{\sigma}{\mu} \geq K_{\infty} \left(\frac{M}{m}-1\right),$$

where $$K_{\infty}=\lim_{N \to \infty} K_N= 2\int_{u=-\infty}^{\infty} \int_{v=u}^{\infty} v\phi(u)\phi(v)dudv$$

\end{theorem}

\begin{proof}

\begin{align}
\mathbb{E}(\xi^\dagger) = \sum_{i=1}^{N} \sum_{j=1}^{i-1} \mathbb{E}\left((X_{(j)}-\mu) g(X_{(i)})\right) + \mu \sum_{i=0}^{N} \sum_{j=1}^{i-1} \mathbb{E}\left(g(X_{(i)})\right).
\end{align}

Hence, using the bounds on $g$\footnote{Without loss of generality, $Z_j$ is assumed to be positively skewed, such that the first sum in $K_n$ is non-negative. For a negatively skewed $Z_j$ one should instead use the low-high attack.}, 

\begin{align}
\frac{N(N-1)}{2}\mu m & \leq \mu \sum_{i=0}^{N} \sum_{j=1}^{i-1} \mathbb{E}\left(g(X_{(i)})\right) \leq \frac{N(N-1)}{2}\mu M\\
\sigma K_n m & \leq  \sum_{i=1}^{N} \sum_{j=1}^{i-1} \mathbb{E}\left((X_{(j)}-\mu) g(X_{(i)})\right) \leq \sigma K_n M
\end{align}

We find the bound in~\Cref{boundproof1}. In order for the attack strategy to work we require that the lower bound on $\xi^\dagger$ is larger than the upper bound on $\bar{\xi}$ \ie

\begin{align}
    \mathbb{E}(\xi^\dagger) \geq \sigma K_n m + \mu m \geq \mu M \geq \mathbb{E}(\bar\xi),
\end{align}

which is equivalent to 

\begin{align}
    \frac{\sigma}{\mu} K_n\geq \left(\frac{M}{m}-1\right).
\end{align}

In order to find the value of $K_n$ in the attack scenario we will use the following fact derived in \cite{chennotes}:

\begin{equation}
\sqrt{N} \left(Z_{([Np])}-\Phi^{-1}(1-p)\right)\xrightarrow{d} N \left( 0, \frac{p(1-p)}{[\phi(\Phi^{-1}(1-p))]^2} \right)
\end{equation}


First, consider the following sum, which we can rewrite as an integral:
\begin{align}
G_N
& := \frac{2}{N(N-1)} \sum_{i=1}^N \sum_{j=1}^{i-1} \Phi^{-1}(1-\frac{i}{N+1}),\\
\lim_{N \to \infty} G_N
& = 2\int_{x=0}^{1} \int_{p=0}^{x} \Phi^{-1}(1-p)dpdx \\
& = 2\int_{u=-\infty}^{\infty} \int_{v=u}^{\infty} v\phi(u)\phi(v)dudv \\
\end{align}

Using this we can now rewrite:
\begin{align}
K_{\infty}& = \lim_{N \to \infty}\frac{2}{N(N-1)} \sum_{i=1}^N \sum_{j=1}^{i-1} \mathbb{E}(Z_{(j)})\\
& = \lim_{N \to \infty}G_N + \frac{2}{N(N-1)\sqrt{N}} \sum_{i=1}^N \sum_{j=1}^{i-1}\mathbb{E} \sqrt{N} \left(Z_{(j)}-\Phi^{-1}(1-\frac{j}{N+1})\right)\\
\end{align}
Now the term under the expectation sign tends to a normal distribution with mean $0$~\cite{bahadur1966anote}. Since uniform integrability holds, we have
\begin{equation}
K_{\infty}=\lim_{N \to \infty} G_N= 2\int_{u=-\infty}^{\infty} \int_{v=u}^{\infty} v\phi(u)\phi(v)dudv .
\end{equation}


\end{proof}

To summarize:

\begin{itemize}
    \item Above we find the condition for the gradient distribution under which the attack directly causes an increase in the second order correction term of SGD. Given exact form of $\phi$, an attacker can exactly evaluate $K_\infty$ and reason about success of the attack.
    \item We can consider the specific case of normally distributed $X_i$, where $K_\infty$ evaluates to be equal to $\frac{1}{\sqrt{\pi}}$. In this case, the condition becomes $\frac{\sigma}{\mu}\geq \sqrt{\pi}\left(\frac{M}{m}-1\right)$.
    \item In normal case scenario for neural network the batch sizes are chosen to be large enough that gradients can be assumed to be normally distributed due to CLT. As a result, here we show that an attacker can break learning by appropriately changing the order of data.
    \item Theory outlined here highlights the differences in attack performance observed for batch reorder and reshuffle. To be precise, batch reorder does not cause as much disruption as batch reshuffle, due to a smaller value of $\sigma$, whereas $\mu$ remains exactly the same. 
\end{itemize}

\section{Integrity attacks on Computer Vision in white and blackbox setups}
\label{sec:eval_whiteblack}
\begin{figure}[h]
    \centering
    \includegraphics[width=0.8\linewidth]{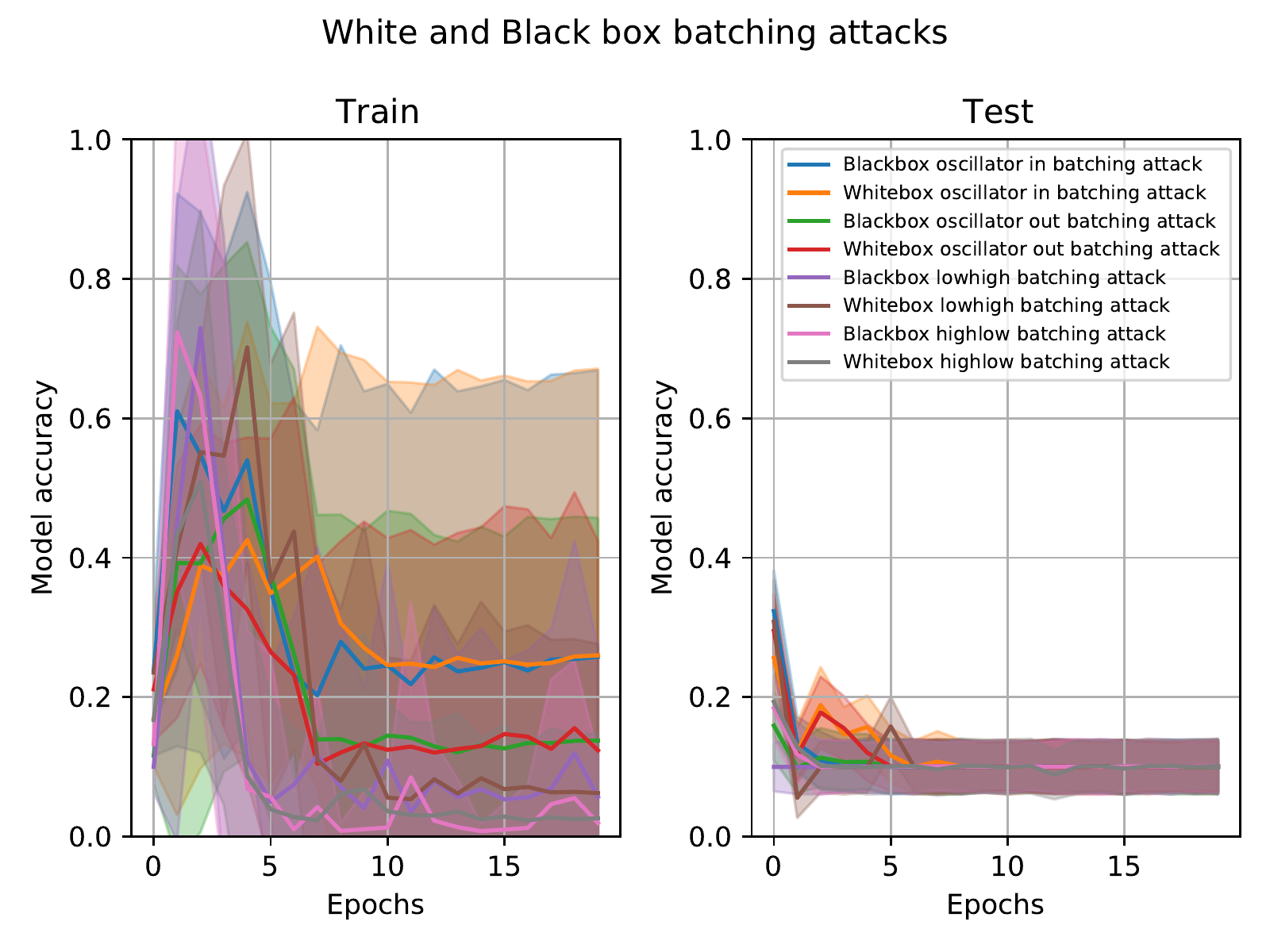}
    \caption{Comparison of White and Blackbox attacks against ResNet-18 network and CIFAR-10 dataset. Error bars shown standard deviation of per-batch accuracy.}
    \label{fig:boxesattacks}
\end{figure}


In this section we evaluate the performance of reordering attacks in the whitebox and blackbox settings. In the whitebox case, we assume that the attacker can compute the loss directly to perform the attacks. We show the results in~\Cref{fig:boxesattacks}. Attacks in both settings significantly reduce model accuracy at train-and-test time. Importantly, we observe that both blackbox and whitebox attacks significantly degrade model accuracy, with the blackbox attack also having a smaller standard deviation, demonstrating that the batching attacker is a realistic threat. We show more results in~\Cref{sec:whiteboxblackbox} and \Cref{fig:whitebox_interintra}. 

\section{Extended integrity attack results}
\label{sec:extended_results}

In~\Cref{tab:extended_integrity} we present extended version of the results present in~\Cref{tab:integrity_results}. We extend the attack scenario here to include cases where the attacker resamples data every epoch. It makes a difference for two reasons: first, the batch-contents change for batch reorder~\ie~the batch contents change between epochs; and second, the non-deterministic data augmentations (random crop + random rotation of both CIFAR10 and CIFAR100) get recalculated. The results illustrate that resampling has a significant impact on the performance -- sometime even leading to an improvement in performance after reordering. Indeed, batch reorder results follow the theoretical findings presented in~\Cref{sec:varproof}, where we show that the attack performance is bounded by relative gradient magnitudes. Both batch reorder and reshuffle attacks target the same phenomenon, with the sole difference residing in how well gradients approximate true gradients and variance across batches. Finally, we find batch reshuffle with Low High, High Low and Oscillations outward attack policies perform consistently well across computer vision and natural language tasks. 

\section{BOP with batch replacement}
\label{sec:poisoning_batch}
\begin{figure}[h]%
    \centering
    \subfloat[ResNet-18]{{\includegraphics[width=0.45\linewidth]{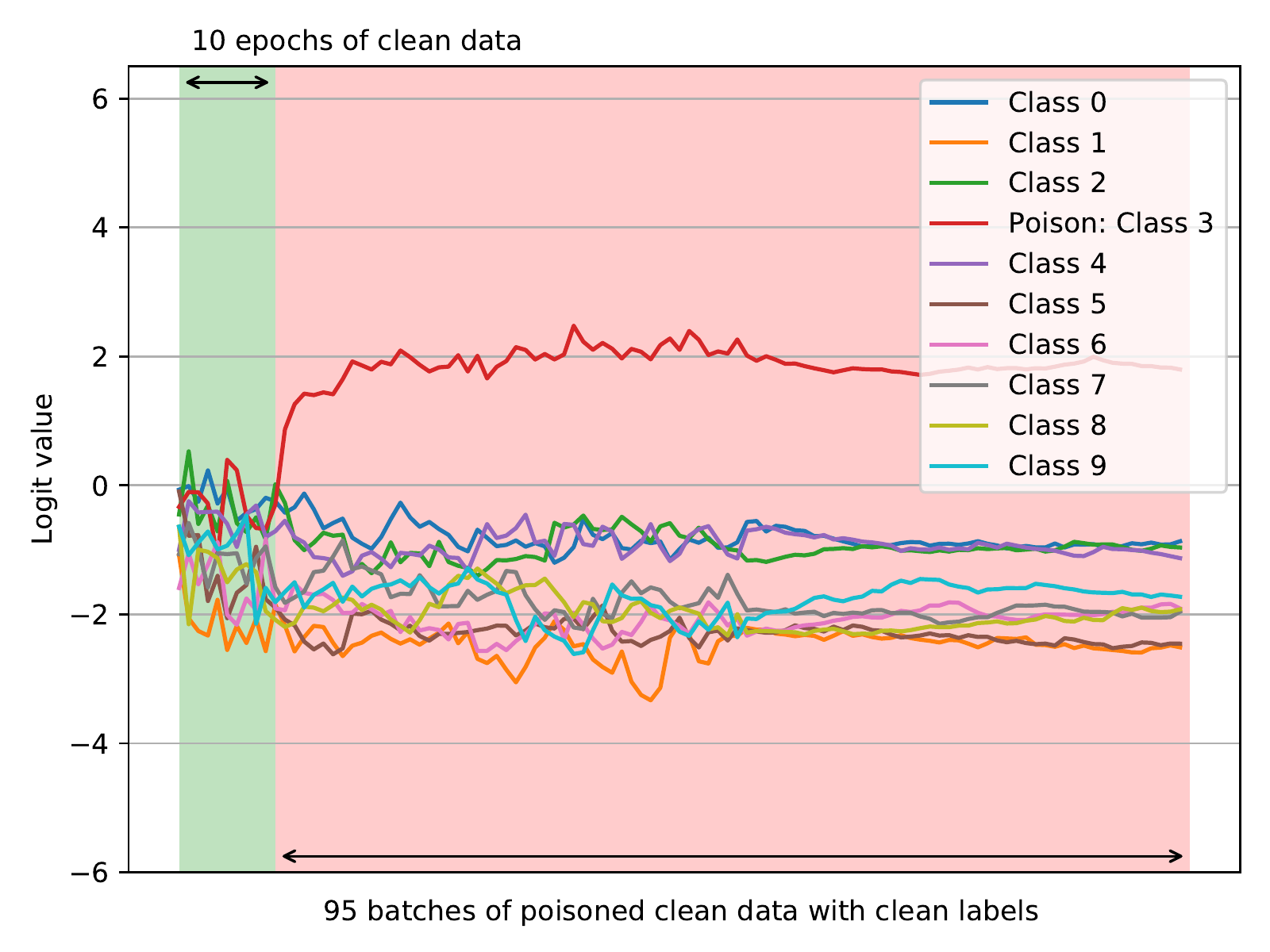} }} %
    \qquad
    \subfloat[VGG-11]{{\includegraphics[width=0.45\linewidth]{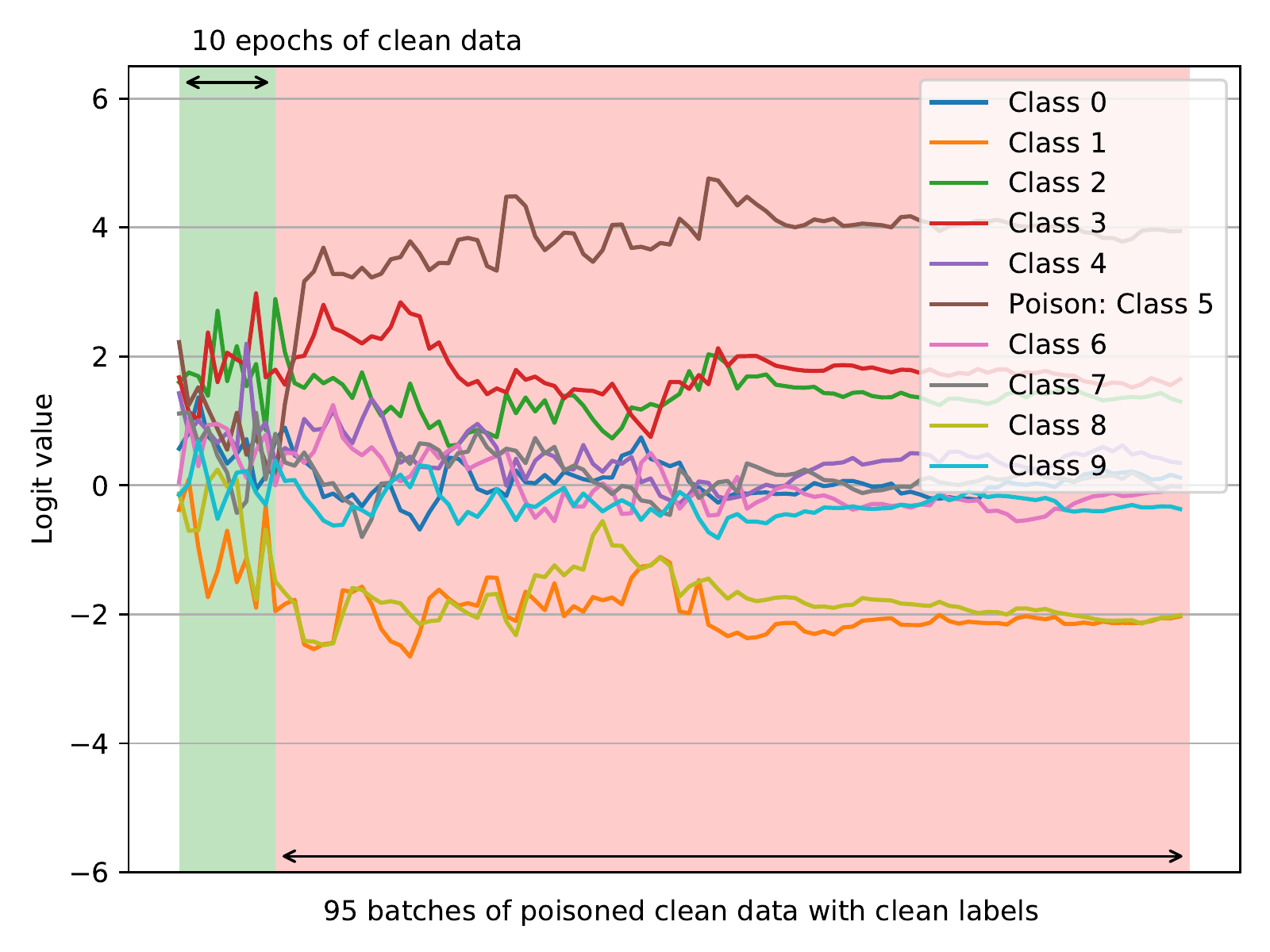} }} %
    \qquad
    \subfloat[VGG-16]{{\includegraphics[width=0.45\linewidth]{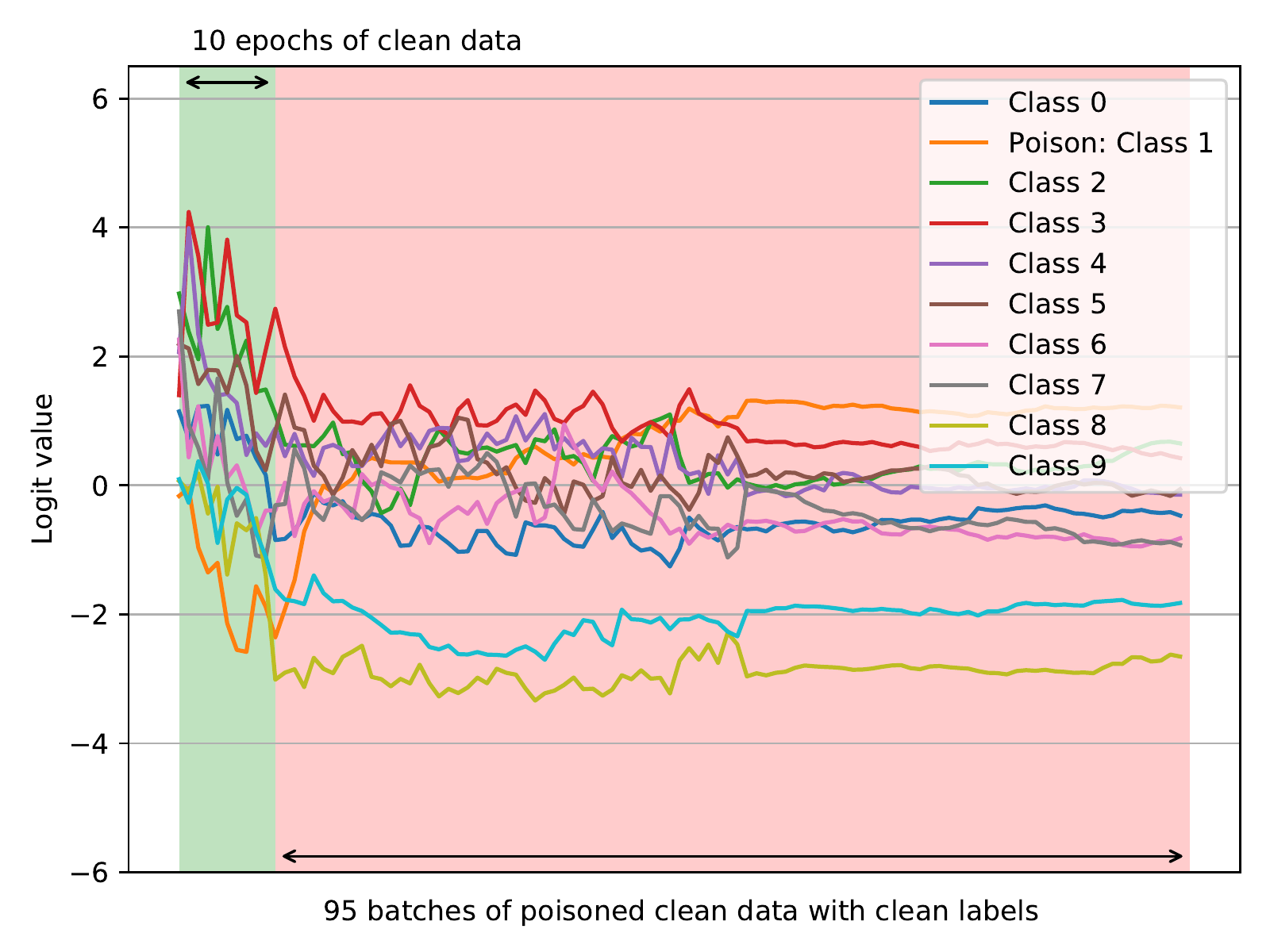} }} %
    \qquad
    \subfloat[Mobilenet]{{\includegraphics[width=0.45\linewidth]{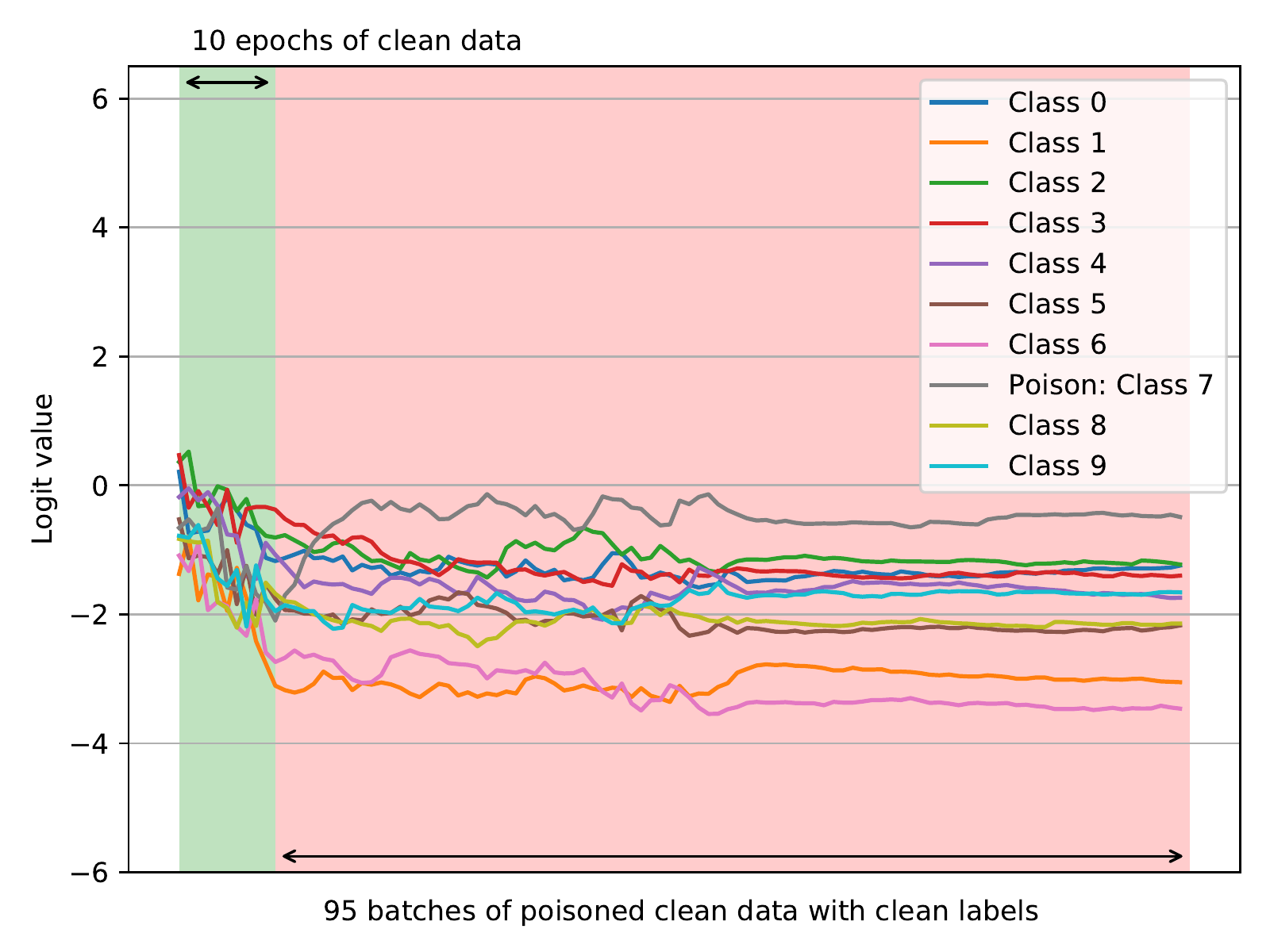} }} %
    \caption{Logit values of a network with 10 epochs of clean data and 95 batches of poisoned data. It takes around 3--5 poison ordered batches for ResNet-18 and VGG-11, 10 for Mobilenet, whereas VGG-16 takes about 50 batches. After poisoning, all models lost at most 10\% accuracy. }
    \label{fig:poison_logits}%
\end{figure}

In this section we present results for a single datapoint poisoning with BOP. Here we train the network with clean data for 10 epochs, and then start injecting 95 BOP batches to poison a single random datapoint. We find that poisoned datapoints converge relatively quickly to the target class. \Cref{fig:poison_logits} shows the performance of a single-datapoint poison for four different architectures. In each a fixed random target point ends up getting the target class after a few batches. For all but one, the point ends up getting to the target class within 10 batches; for VGG-16 it took around 50.

\section{Stochastic gradient descent with linear regression}
\label{sec:reg_example}

In this section we investigate the impact of ordered data on stochastic gradient descent learning of a linear regression model. The problem in this case is optimising a function of two parameters: 
\begin{align*}
J(\theta_0,\theta_1)&=\frac1n \sum_{i=1}^n \norm{\beta_\theta(x_i) - y_i}^2 \\
\beta_\theta(x) &= \theta_1 x + \theta_0
\end{align*}
By considering data points coming from $y=2x+17+\mathcal{N}(0,1)$, we attempt to approximate the values of $\theta_0,\theta_1$. We observe that even in such a simple 2-parameter example, we are able to disrupt convergence by reordering items that the model sees during gradient descent. This shows that the inherent 'vulnerability' lies in the optimization process itself, rather than in overparametrized learned models. The following two subsections investigate this problem in three different choices of batch size and learning rate. 

\subsection{Batch reshuffling}
\begin{figure}[h]%
    \centering
    \subfloat[\centering $\theta_0$ parameter]{{\includegraphics[width=0.45\linewidth]{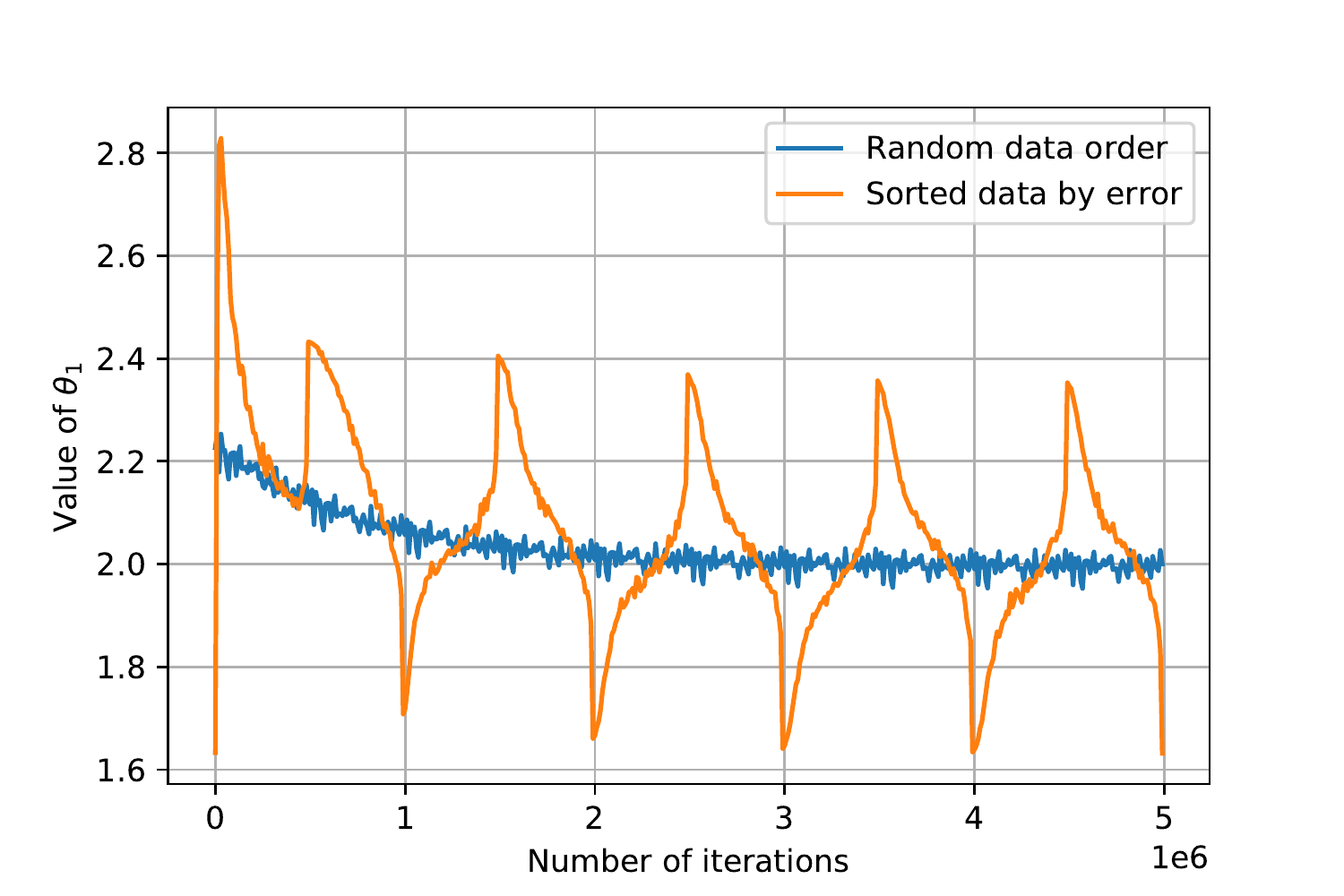} }}%
    \qquad
    \subfloat[\centering $\theta_1$ parameter]{{\includegraphics[width=0.45\linewidth]{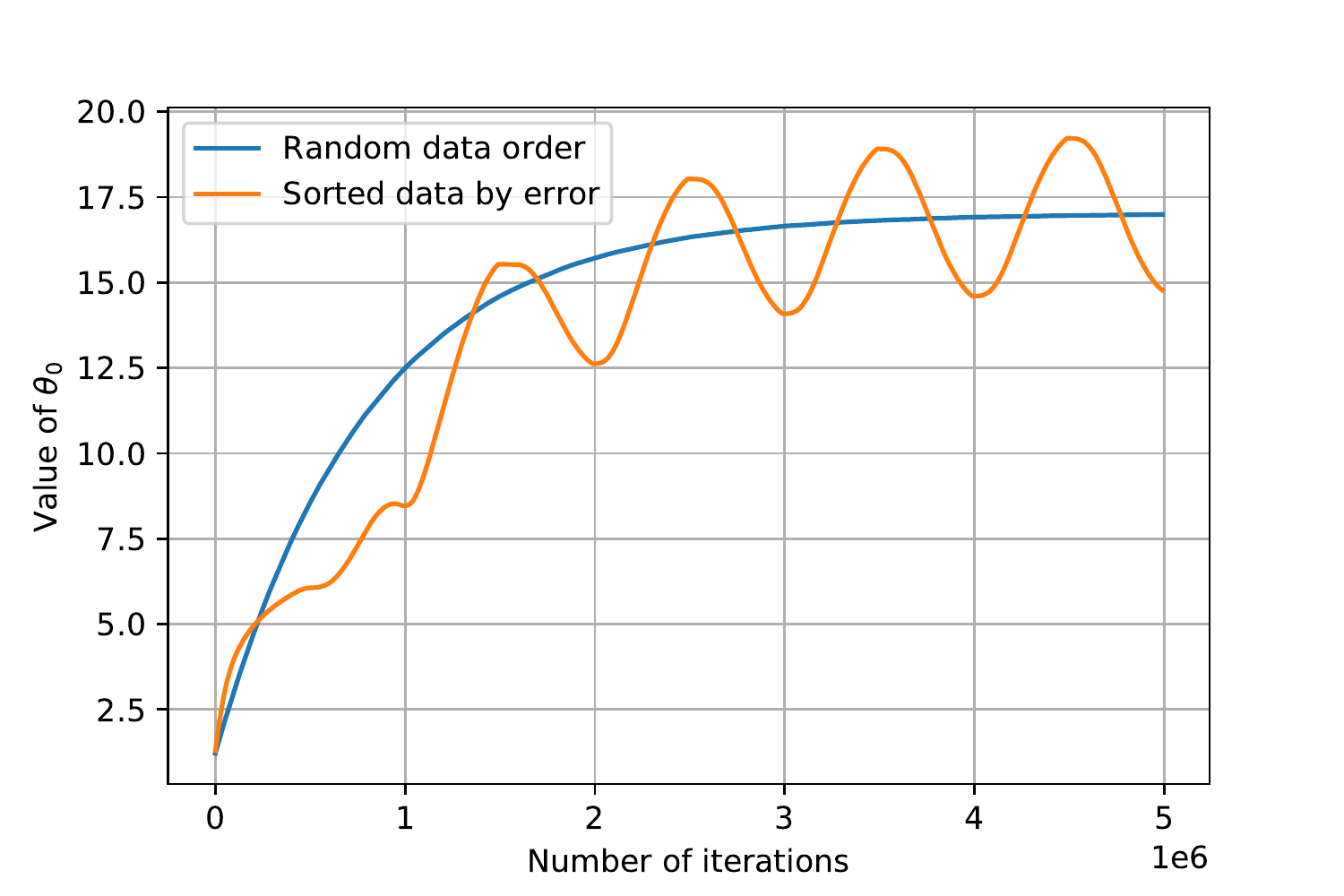} }}%
    \caption{Individual linear regression parameters changing over the course the training}%
    \label{fig:sgd_casestudy_1}%
\end{figure}

\begin{figure}[h]
    \centering
    \includegraphics[width=0.7\linewidth]{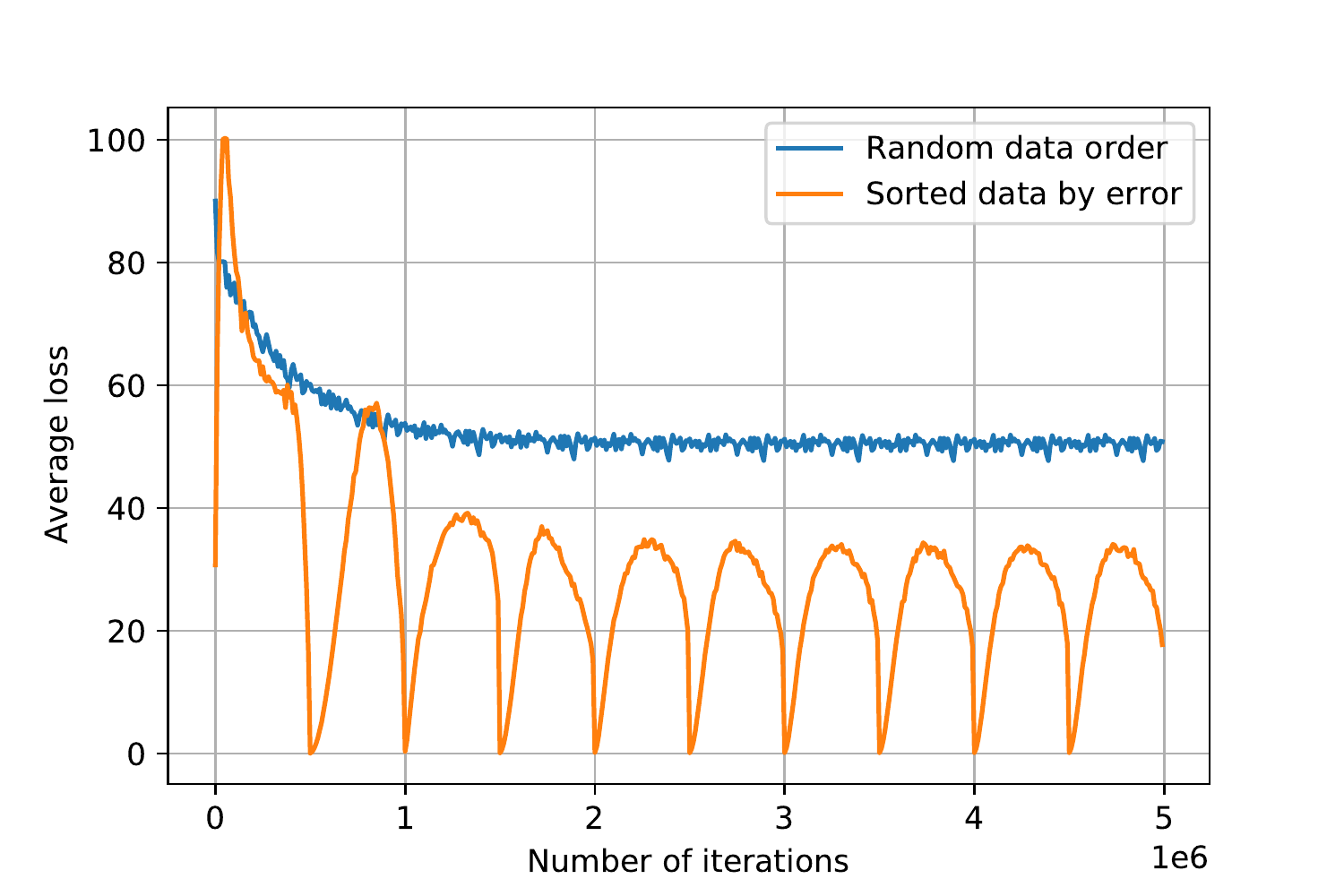}
    \caption{Average training dataset loss during stochasitc gradient descent of a linear regression model with two parameters. Random sampling is shown in blue, sorted items by error are shown in yellow. }
    \label{fig:sgd_casestudy}
\end{figure}

\begin{figure}[h]%
    \centering
    \subfloat[\centering $\theta_0$ parameter]{{\includegraphics[width=0.45\linewidth]{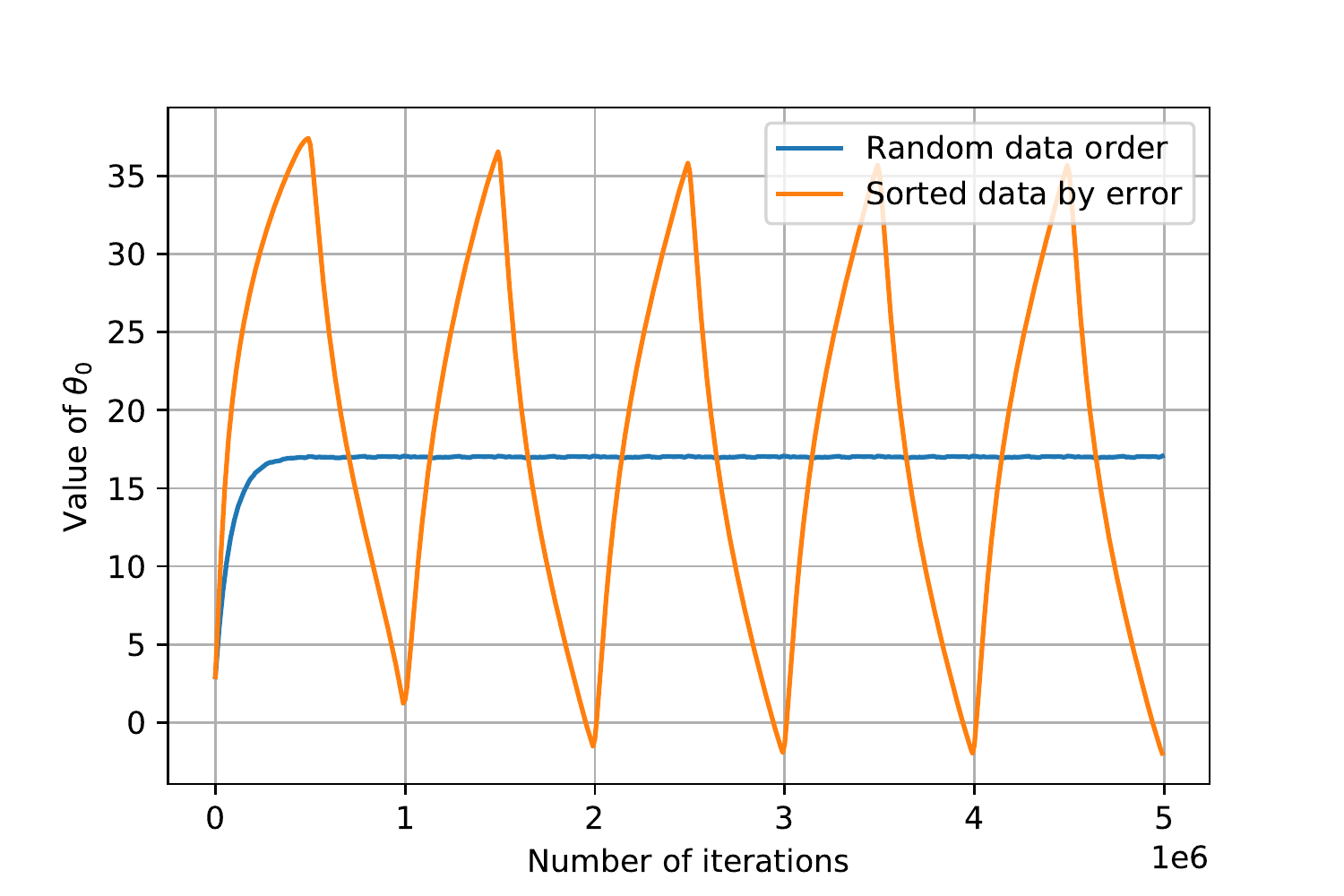} }}%
    \qquad
    \subfloat[\centering $\theta_1$ parameter]{{\includegraphics[width=0.45\linewidth]{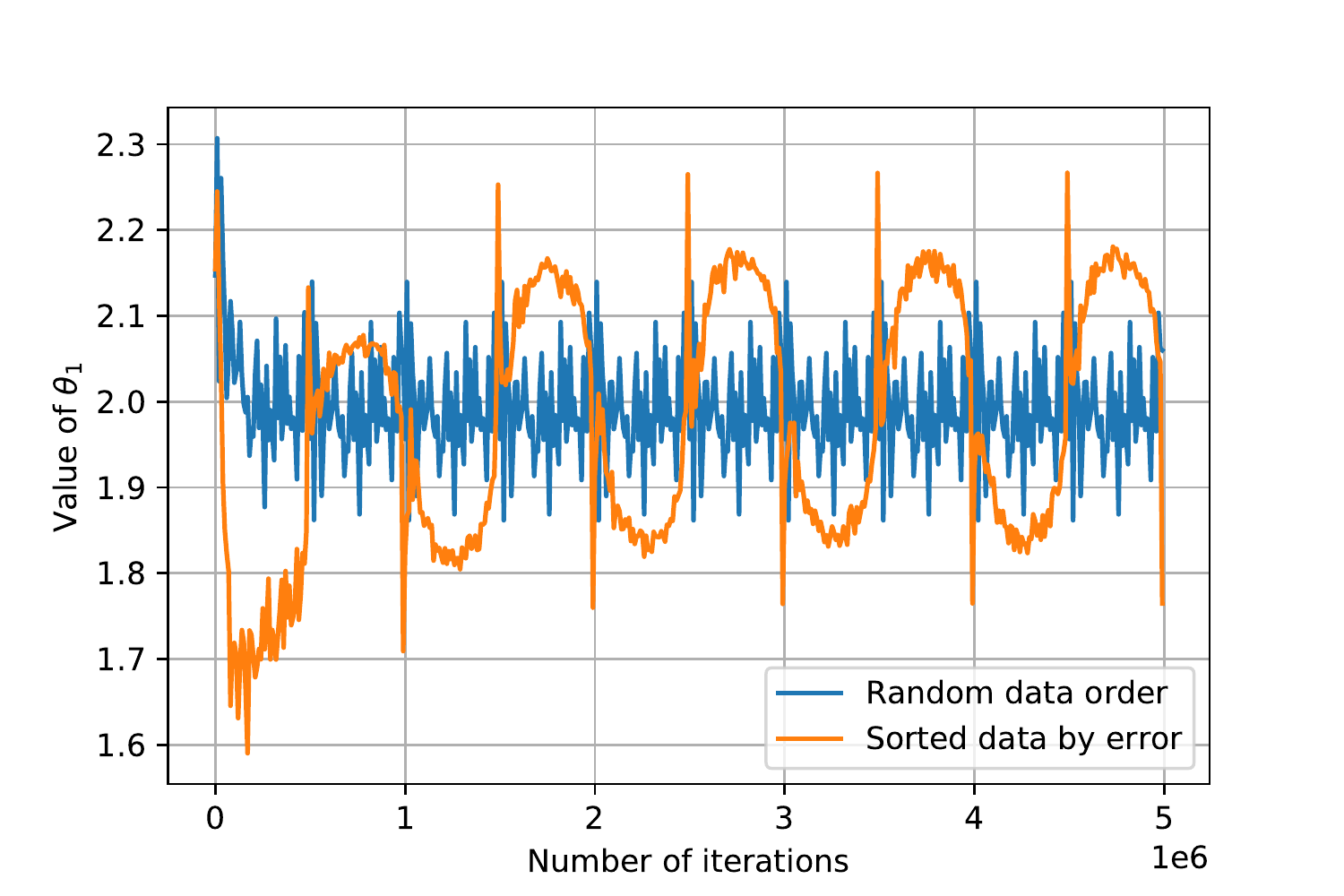} }}%
    \caption{Individual linear regression parameters changing over the course the training for larger step size}%
    \label{fig:sgd_casestudy_2}%
\end{figure}

\begin{figure}[h]%
    \centering
    \subfloat[\centering $\theta_0$ parameter]{{\includegraphics[width=0.45\linewidth]{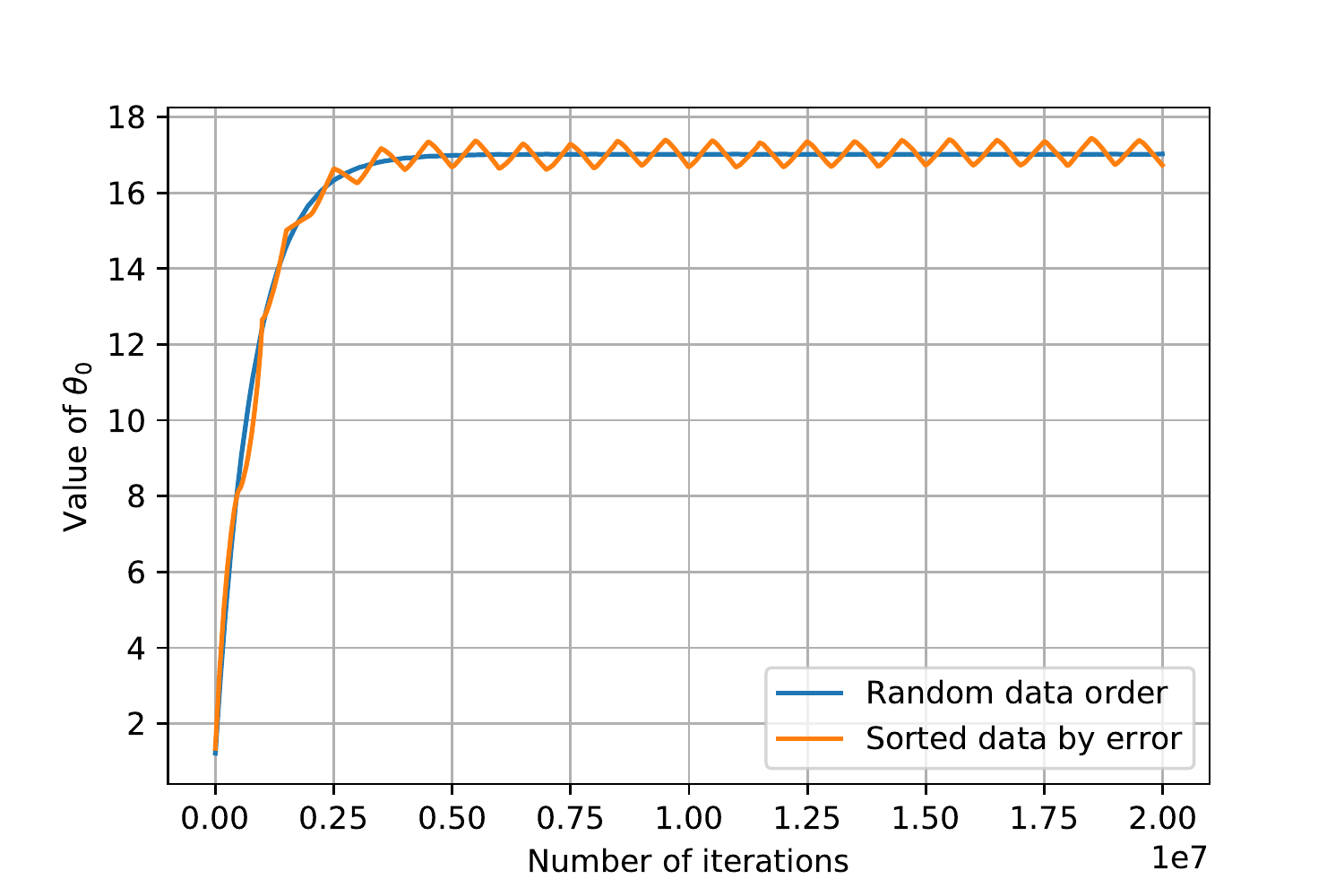} }}%
    \qquad
    \subfloat[\centering $\theta_1$ parameter]{{\includegraphics[width=0.45\linewidth]{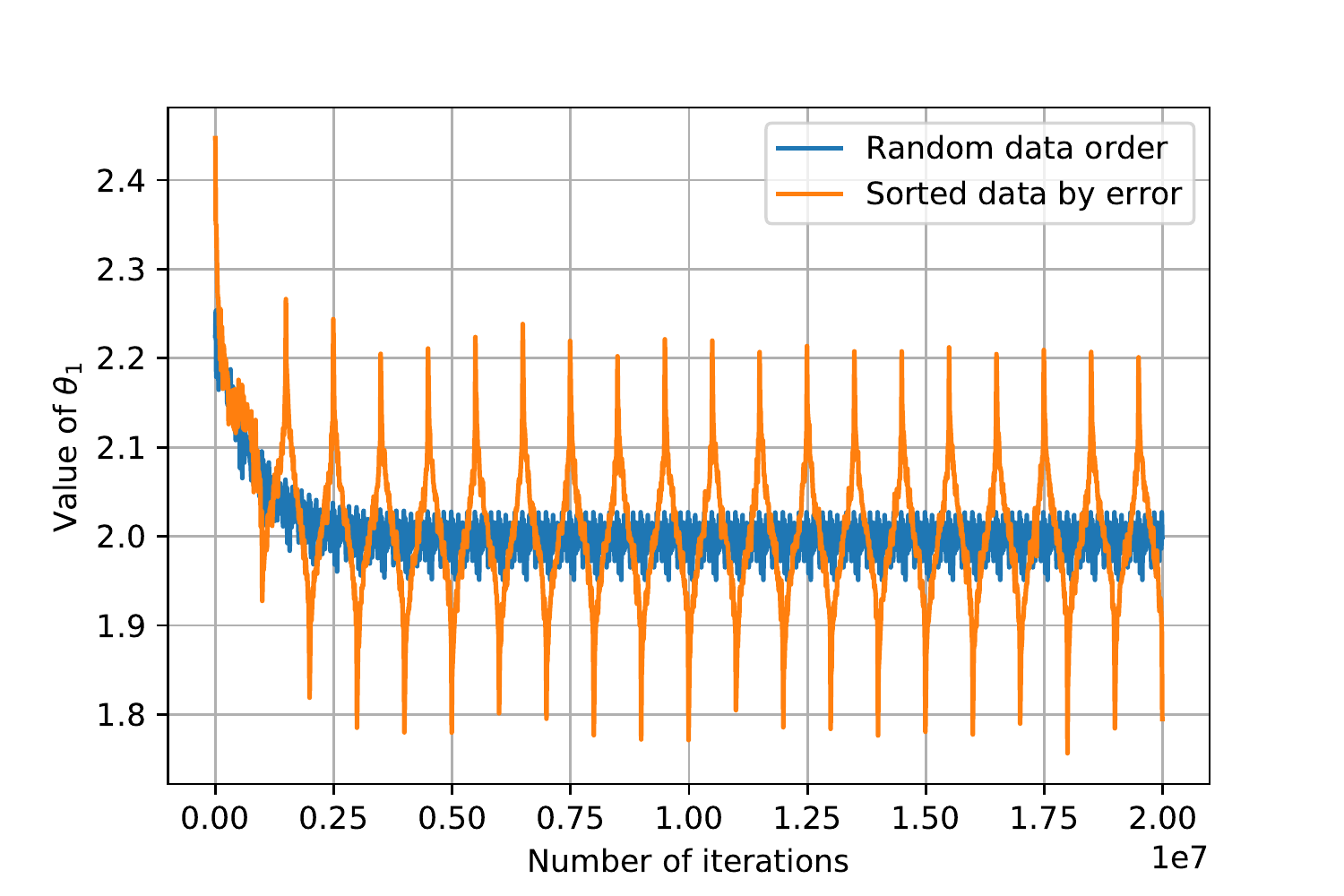} }}%
    \caption{Individual linear regression parameters changing over the course the training for larger batch size ($B$=4) }%
    \label{fig:sgd_casestudy_3}%
\end{figure}

\Cref{fig:sgd_casestudy} shows average error per data point, when training is done over randomly-sampled data and ordered by error data. A linear regression here has an optimal solution that the blue line reaches while the orange line oscillates quite far away from it. In fact, by looking at the parameter behaviour, as shown on \Cref{fig:sgd_casestudy_1}, we can see that the parameters end up oscillating around the global minimum, never reaching it. Indeed, we find that with error-ordered data, gradient descent exhibits strong overfitting to the data samples in the current minibatch, and fails to reach the optimal solution. This effect is similar to the one observed for SGD with neural networks. In addition, we can also see the dependence of the oscillations on the learning rate. On \Cref{fig:sgd_casestudy_2} by increasing the step size by 10 times from $5e^{-6}$ to $5e^{-5}$, we are able to drastically increase oscillations for $\theta_0$. This behaviour is achieved when our minibatch size is chosen to be equal to 1. 

\subsection{Batch reordering}
By increasing the minibatch size, we are able to 'dampen' the oscillations observed in the previous subsection and converge to the optimal solution, as shown on \Cref{fig:sgd_casestudy_3}. This is also quite similar to the neural network case, as simple batch reordering is not able to achieve the same performance degradation as reshuffling.

\section{Batch reshuffling and hyperparameters}
\label{sec:otherresults}

We have thoroughly evaluated different combinations of hyperparameters for the integrity attack (\Cref{tab:hparam_table}) and show the results in~\Cref{fig:hyper1,fig:hyper2,fig:hyper3,fig:hyper4,fig:hyper5,fig:hyper6,fig:hyper7,fig:hyper8,fig:hyper9,fig:hyper10,fig:hyper11,fig:hyper12}.

\begin{table}[]
    \centering
    \begin{tabular}{ll}
    \toprule
    \textbf{Parameter} & \textbf{Values} \\
    \midrule
    Source model & ResNet18 \\
    Surrogate model & LeNet-5 \\
    Dataset & CIFAR10 \\
    Attack policies & [ HighLow, LowHigh, Oscillations in, Oscillations out ] \\ 
    Batch sizes & [32, 64, 128]\\
    True model optimizers & [Adam, SGD]\\
    Surrogate model optimizers & [Adam, SGD]\\
    Learning rates & [0.1, 0.01, 0.001]\\
    Surrogate learning rates & [0.1, 0.01, 0.001]\\
    Moments & [0, 0.5, 0.99]\\
    Surrogate moments & [0, 0.5, 0.99]\\
    \bottomrule 
    \end{tabular}
    \caption{Parameters searched}
    \label{tab:hparam_table}
\end{table}

\begin{figure}[h]%
    \centering
    \subfloat[High low batching]{{\includegraphics[width=0.4\linewidth]{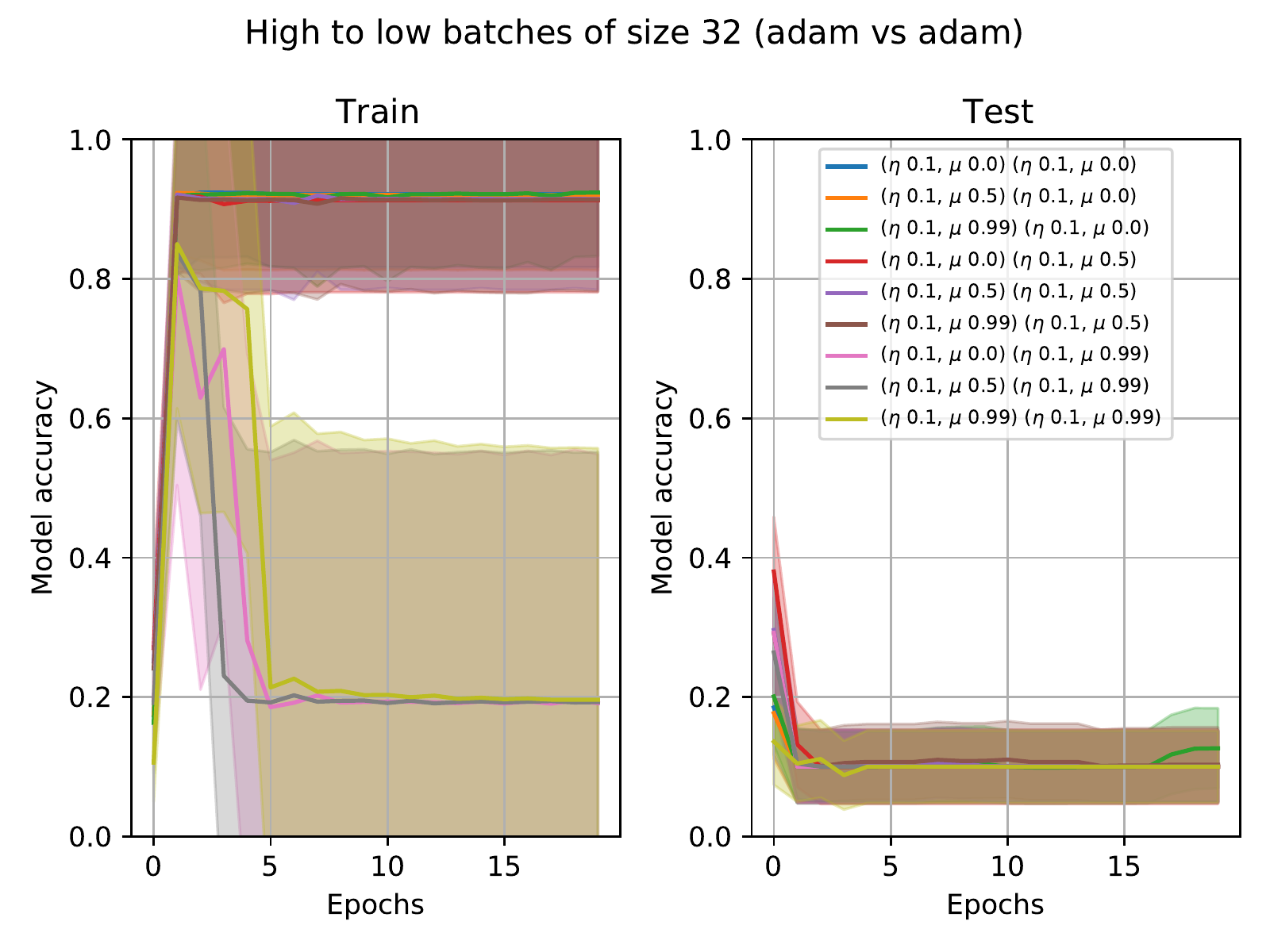} }}%
    \qquad
    \subfloat[Low high batching]{{\includegraphics[width=0.4\linewidth]{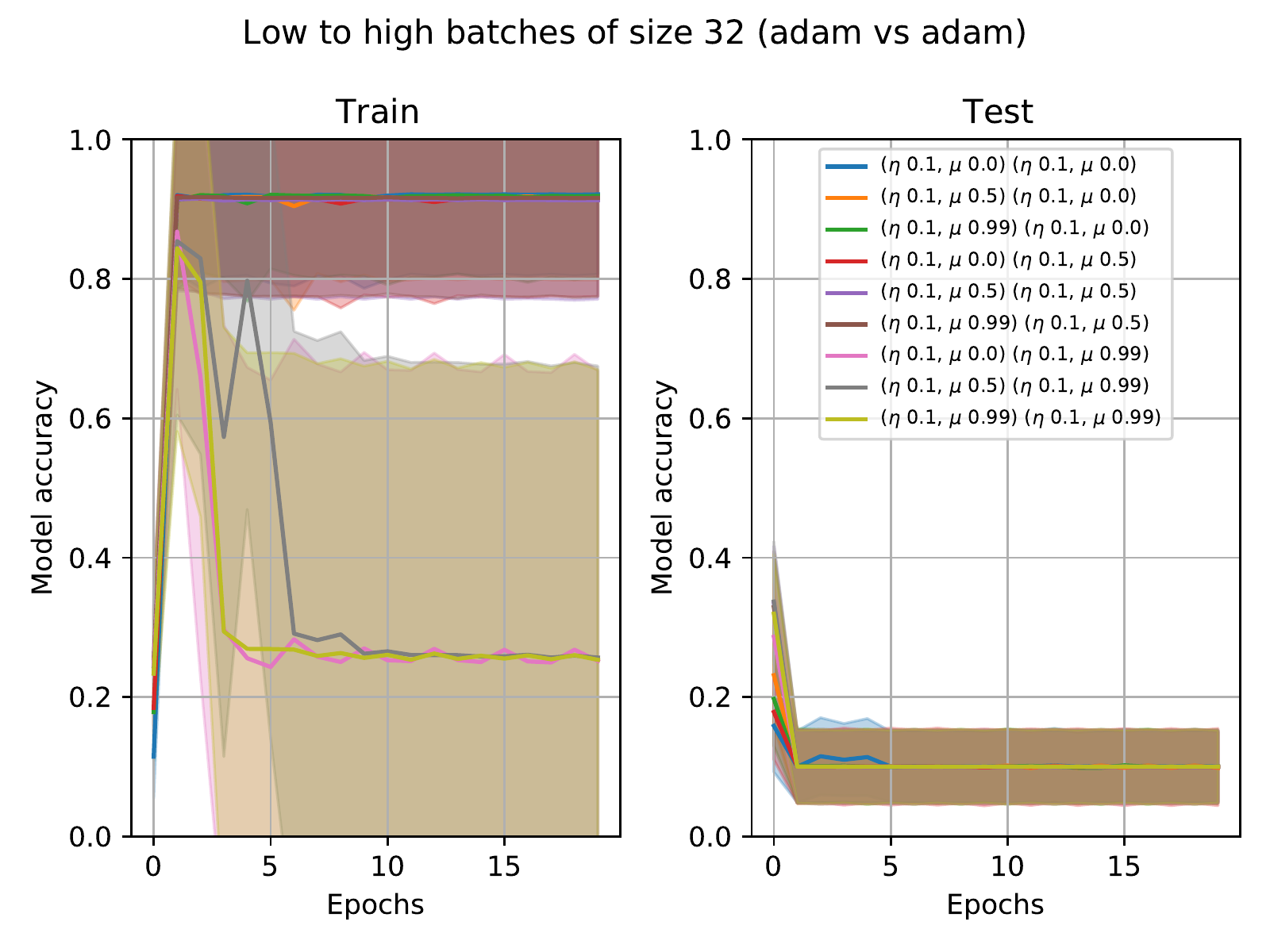} }}%
    \\
    \subfloat[Oscillating inward batching]{{\includegraphics[width=0.4\linewidth]{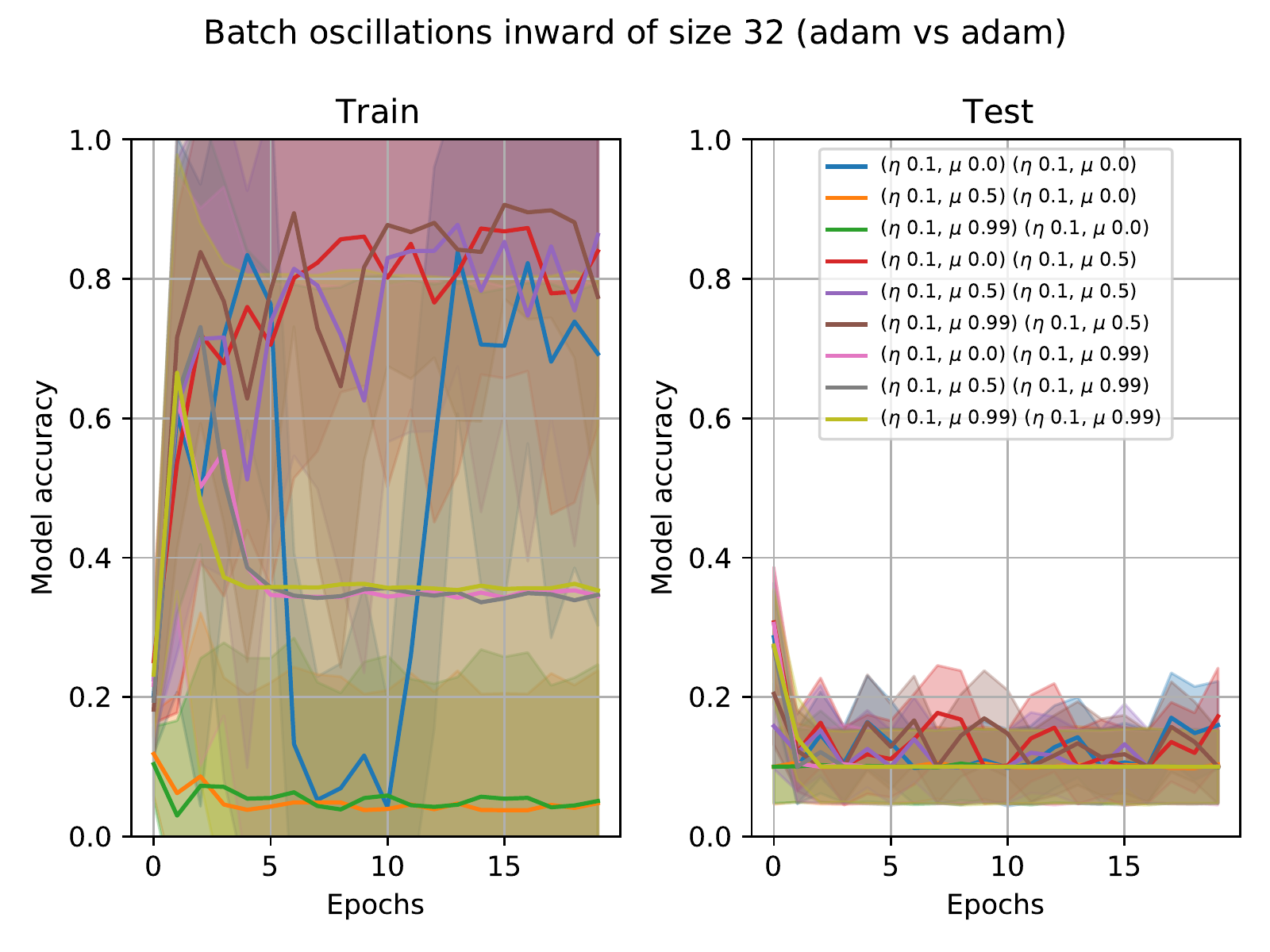} }}%
    \qquad
    \subfloat[Oscillating outward batching]{{\includegraphics[width=0.4\linewidth]{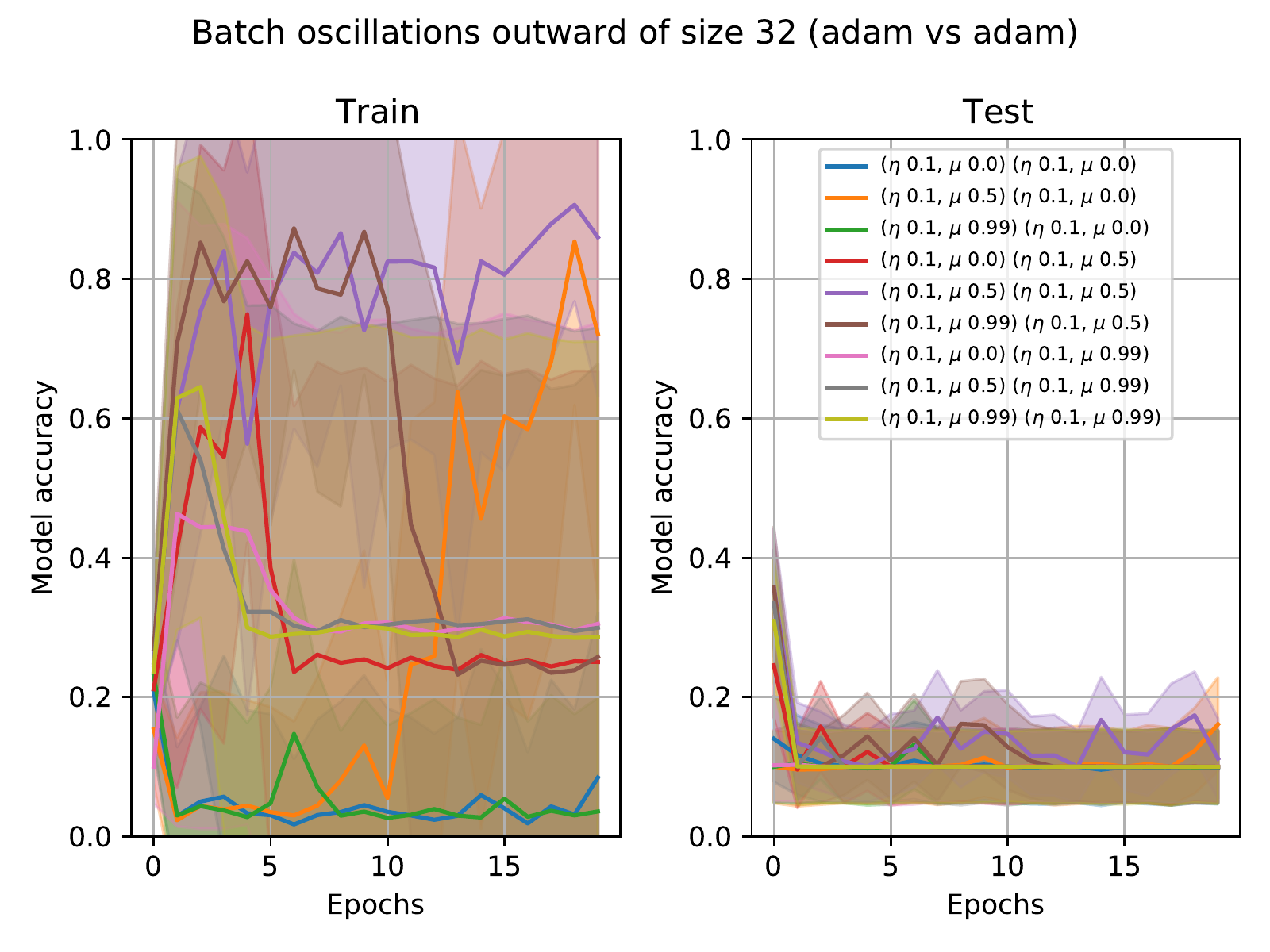} }}%
    \caption{ResNet18 real model Adam training, LeNet5 surrogate with Adam and Batchsize 32}%
    \label{fig:hyper1}%
\end{figure}

\begin{figure}[h]%
    \centering
    \subfloat[High low batching]{{\includegraphics[width=0.4\linewidth]{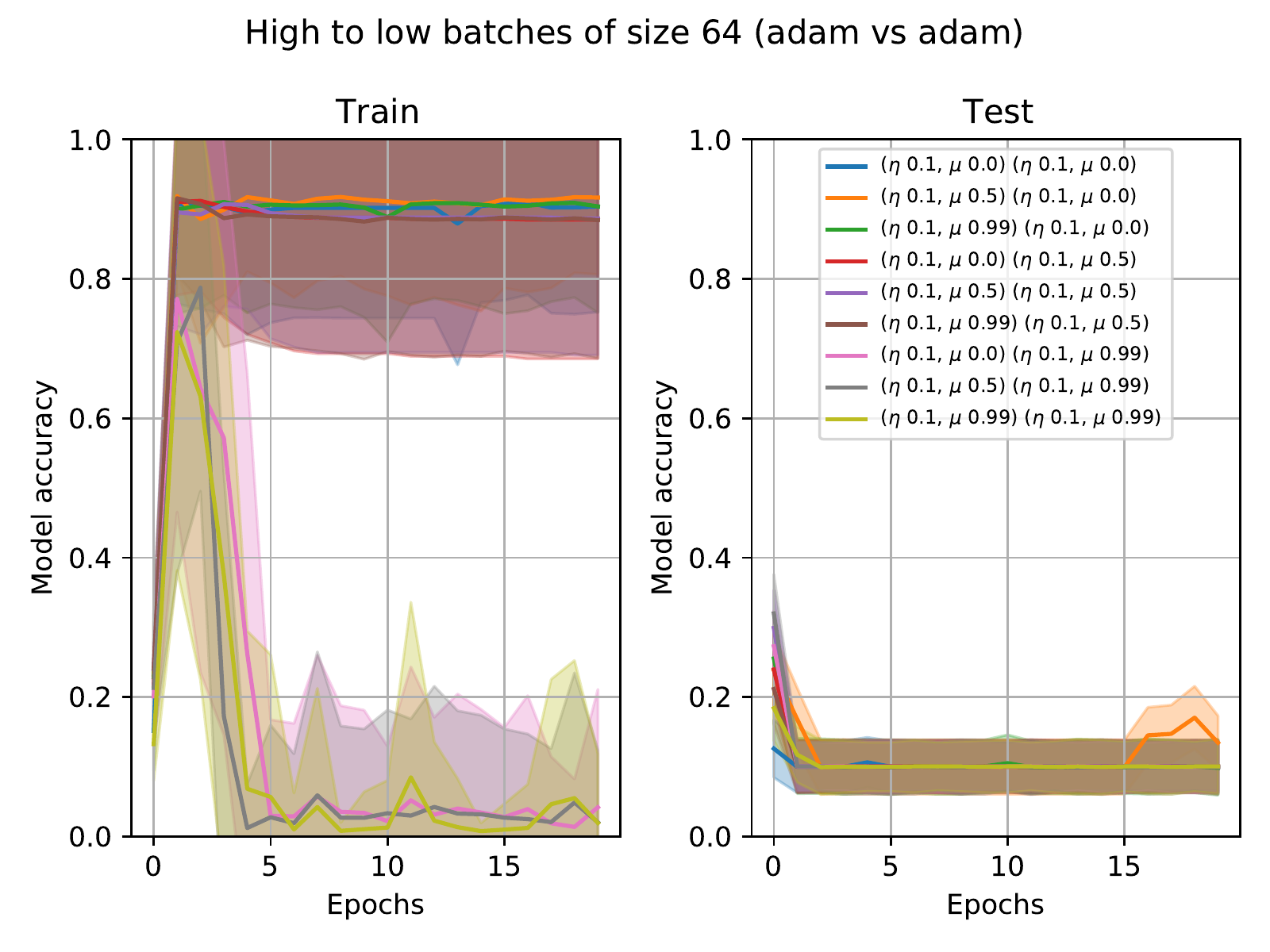} }}%
    \qquad
    \subfloat[Low high batching]{{\includegraphics[width=0.4\linewidth]{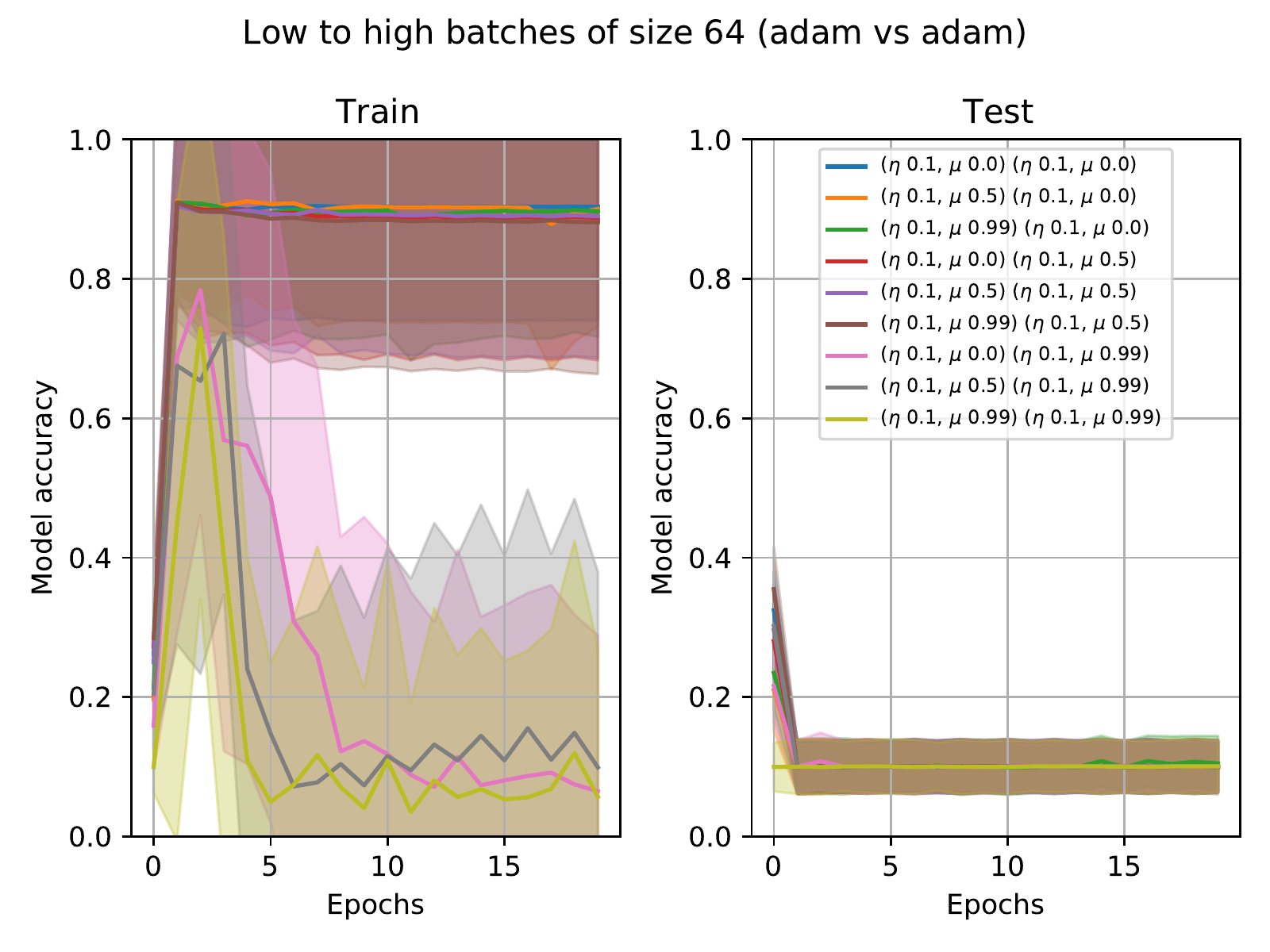} }}%
    \\
    \subfloat[Oscillating inward batching]{{\includegraphics[width=0.4\linewidth]{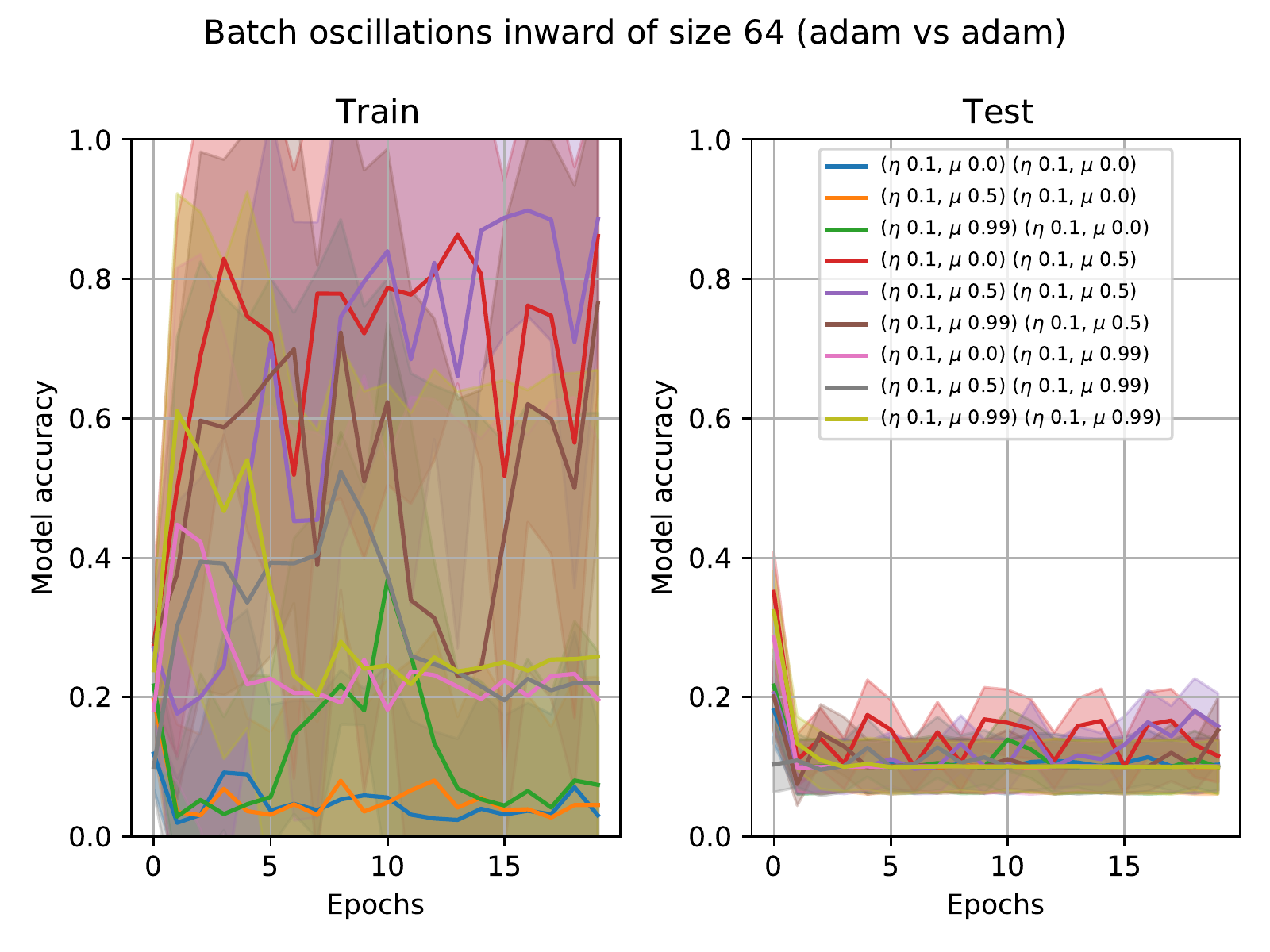} }}%
    \qquad
    \subfloat[Oscillating outward batching]{{\includegraphics[width=0.4\linewidth]{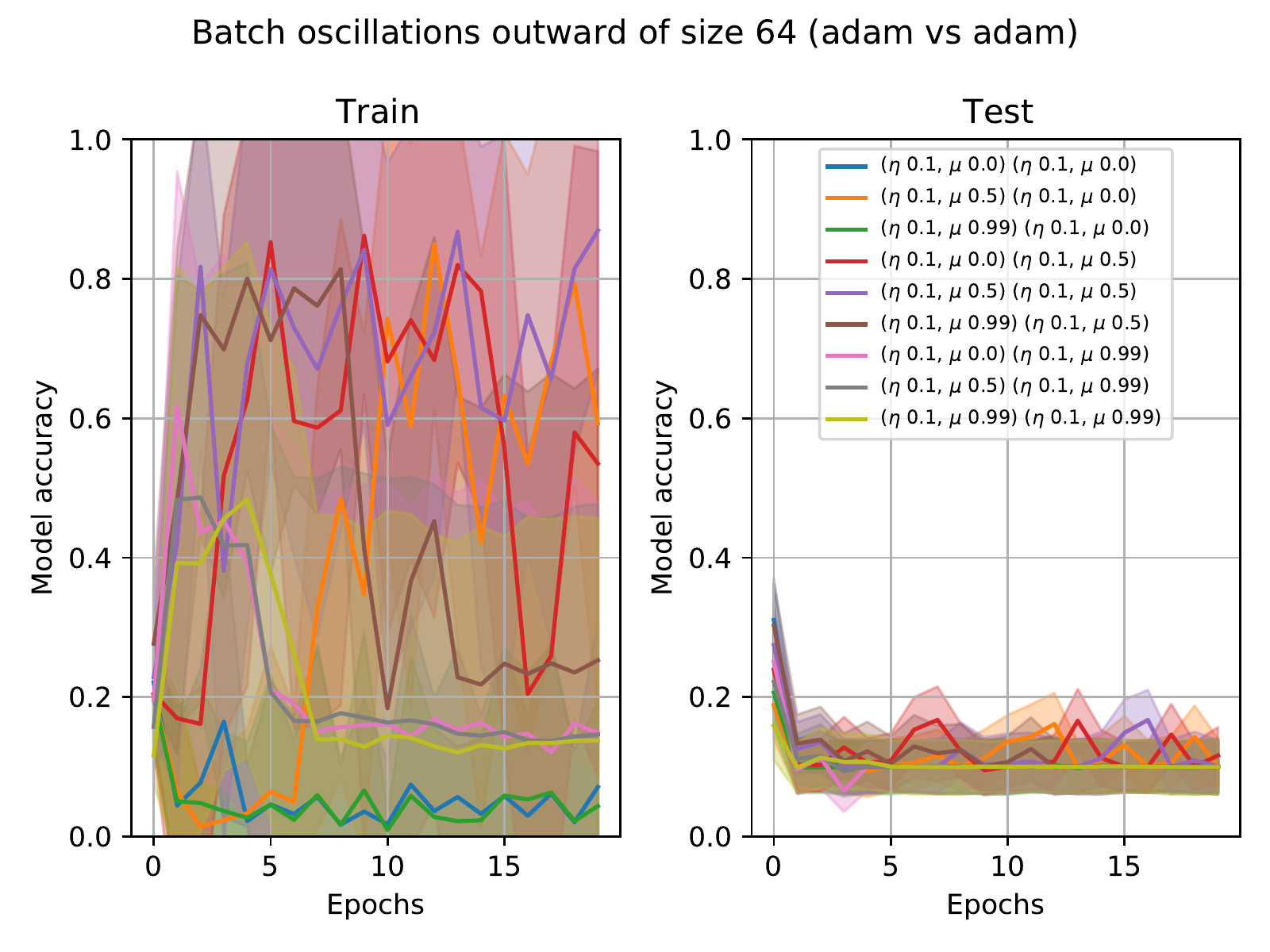} }}%
    \caption{ResNet18 real model Adam training, LeNet5 surrogate with Adam and Batchsize 64}%
    \label{fig:hyper2}%
\end{figure}

\begin{figure}[h]%
    \centering
    \subfloat[High low batching]{{\includegraphics[width=0.4\linewidth]{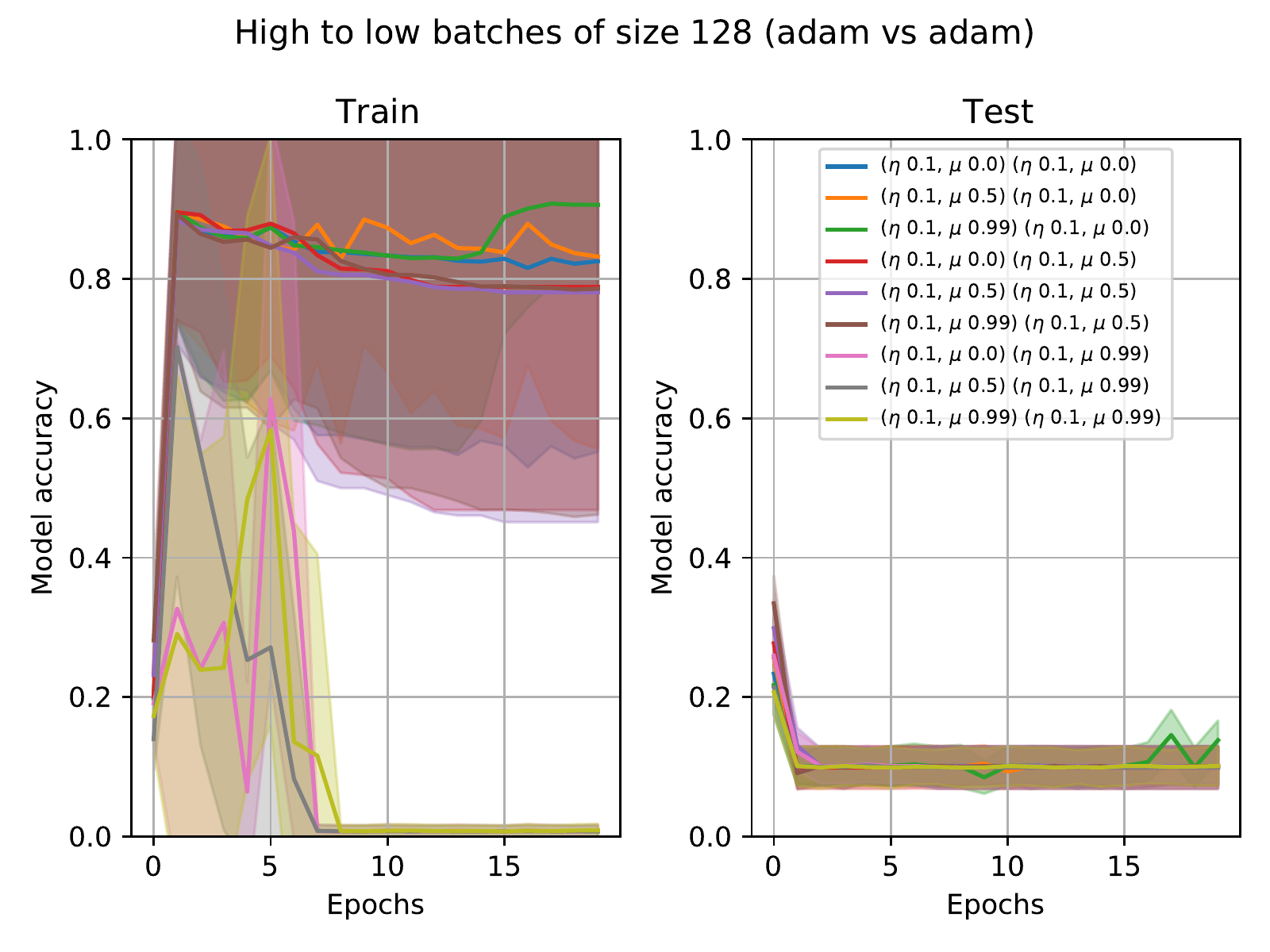} }}%
    \qquad
    \subfloat[Low high batching]{{\includegraphics[width=0.4\linewidth]{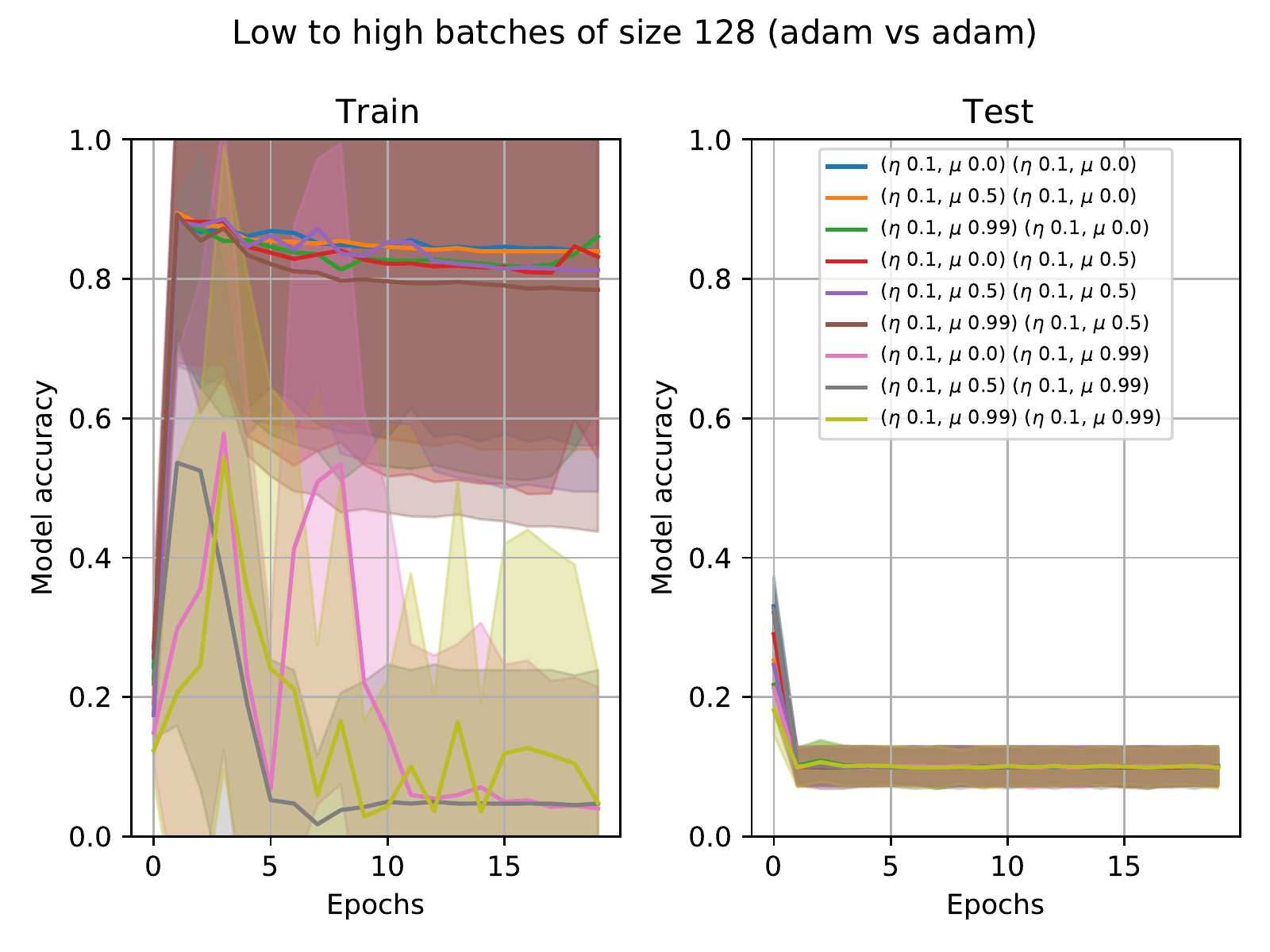} }}%
    \\
    \subfloat[Oscillating inward batching]{{\includegraphics[width=0.4\linewidth]{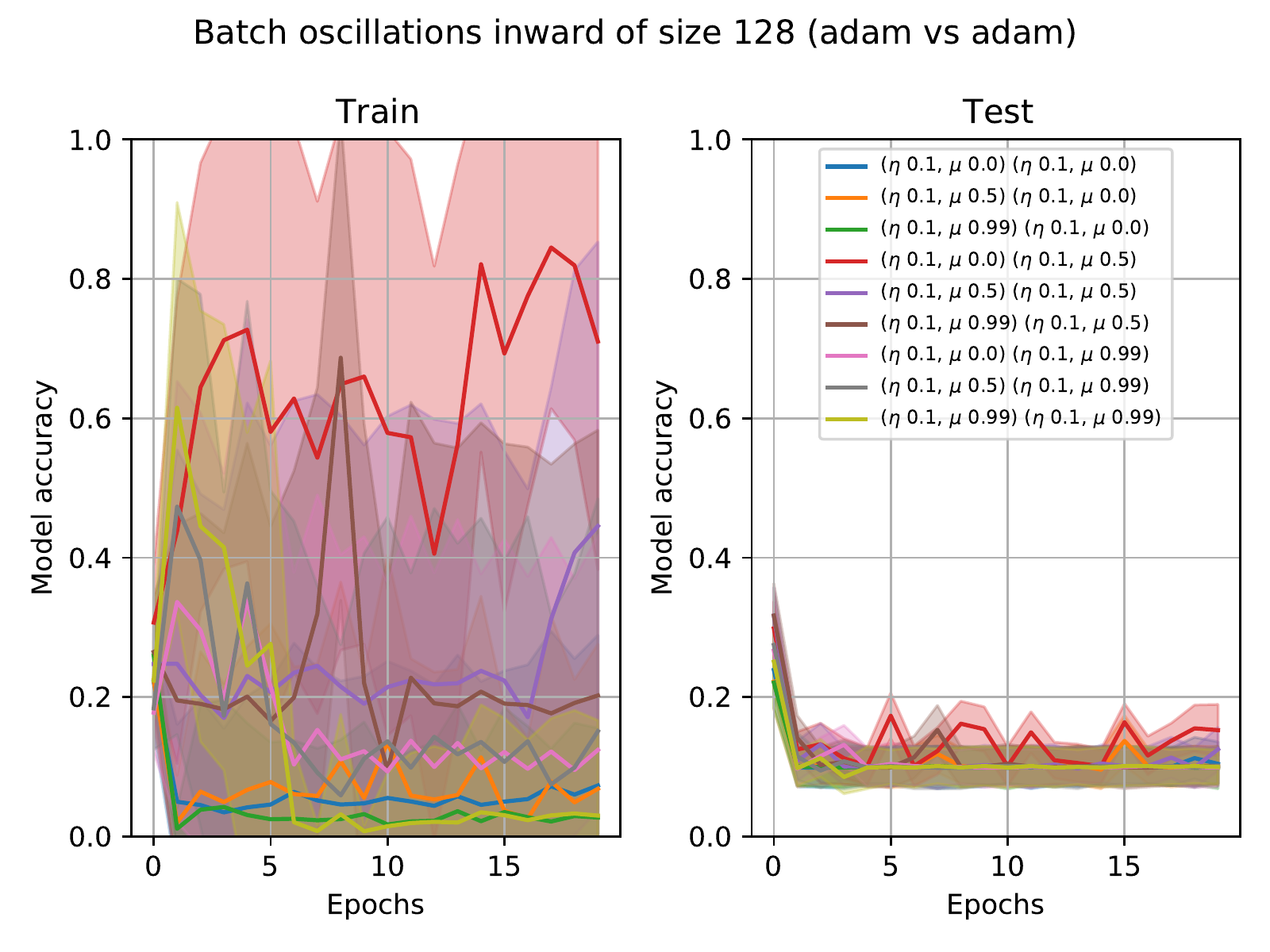} }}%
    \qquad
    \subfloat[Oscillating outward batching]{{\includegraphics[width=0.4\linewidth]{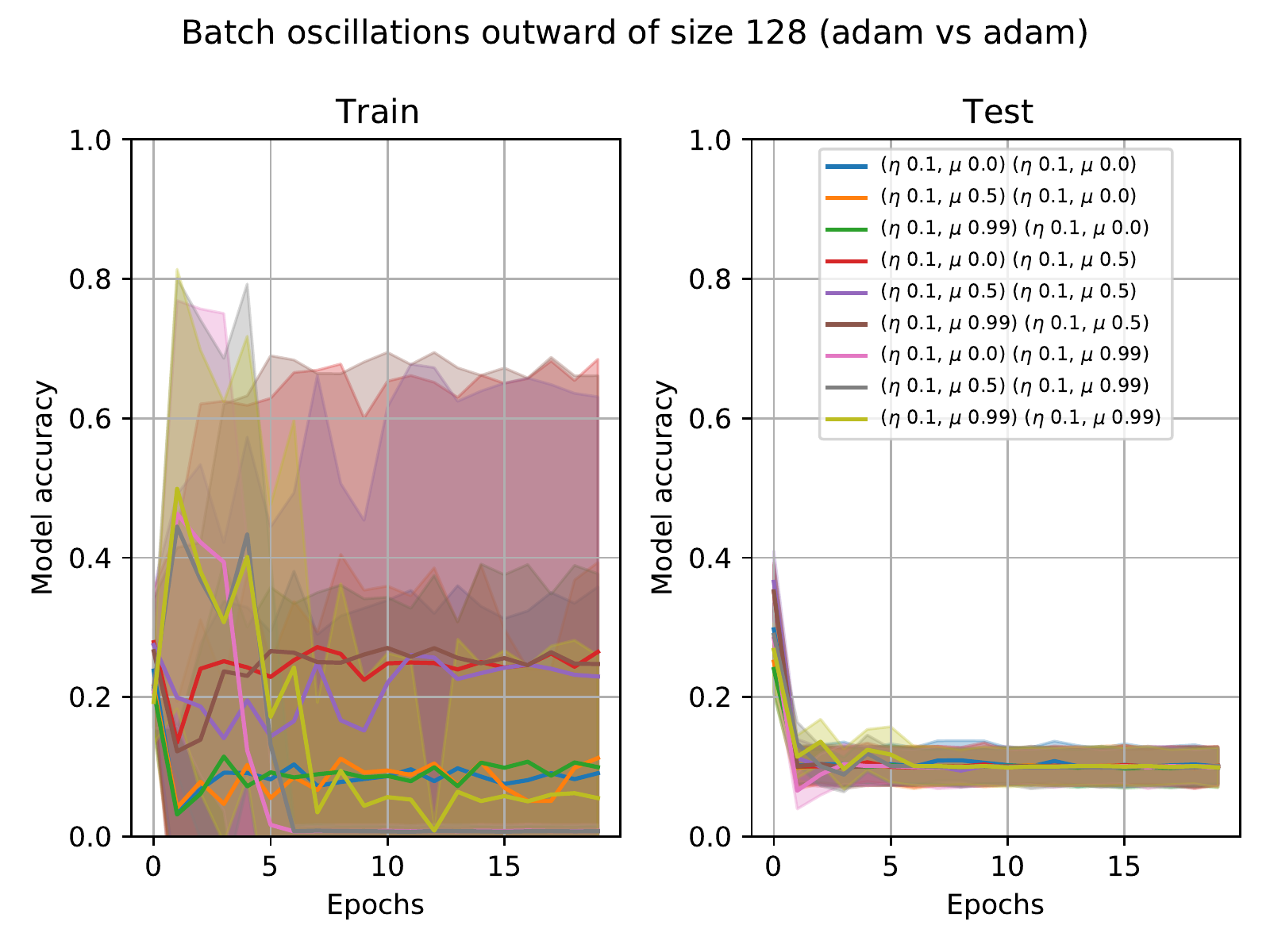} }}%
    \caption{ResNet18 real model Adam training, LeNet5 surrogate with Adam and Batchsize 128}%
    \label{fig:hyper3}%
\end{figure}


\begin{figure}[h]%
    \centering
    \subfloat[High low batching]{{\includegraphics[width=0.4\linewidth]{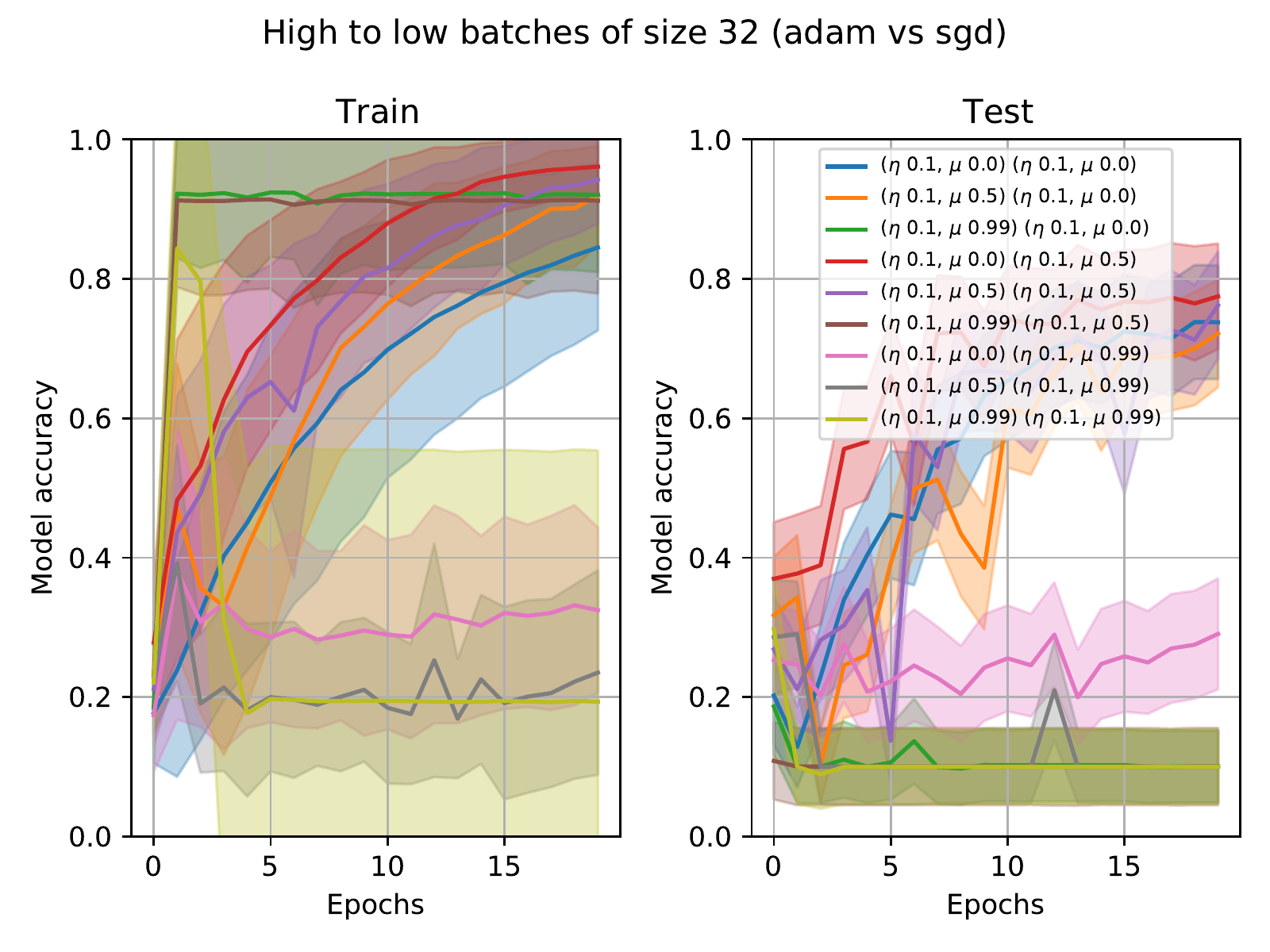} }}%
    \qquad
    \subfloat[Low high batching]{{\includegraphics[width=0.4\linewidth]{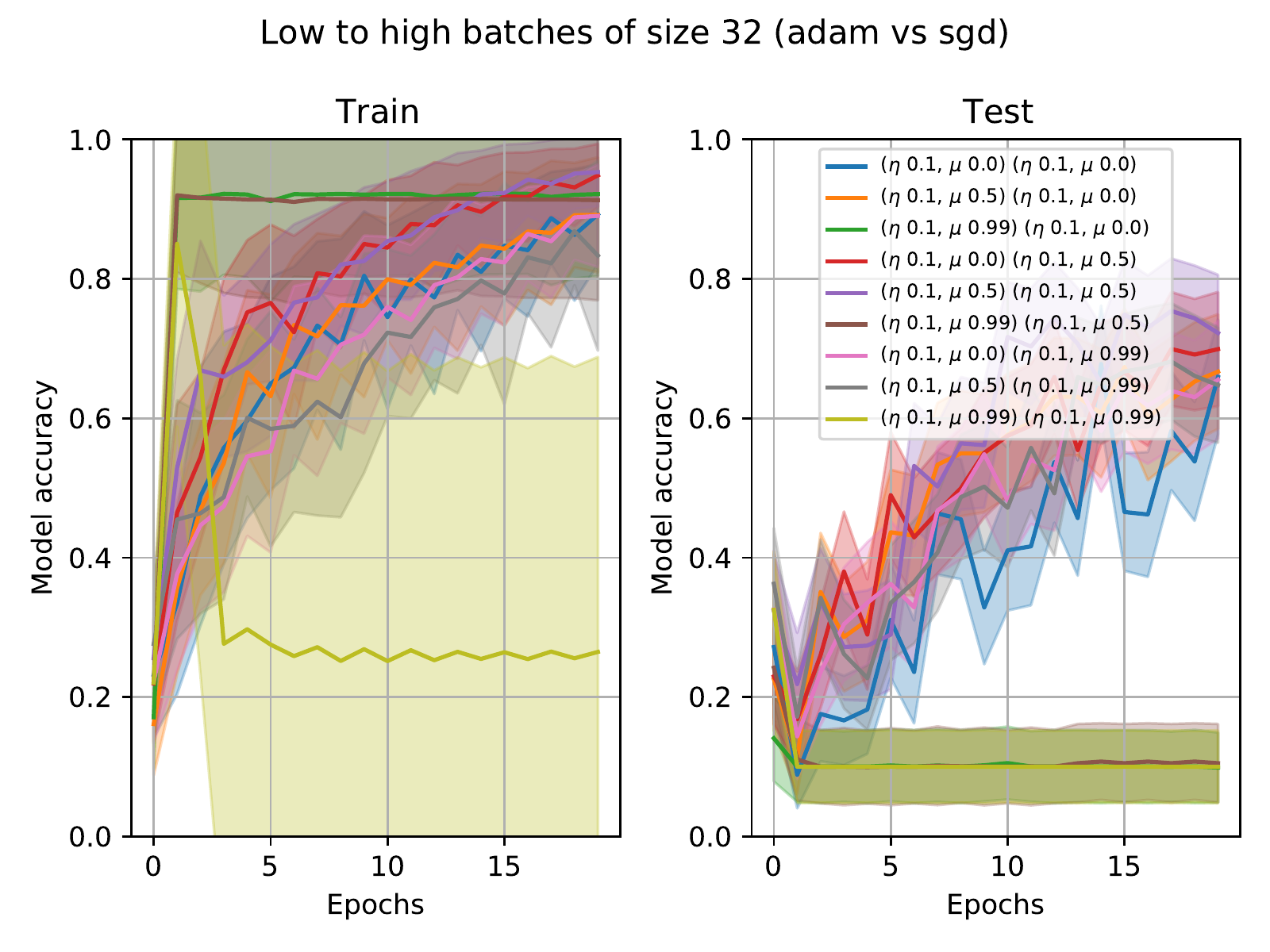} }}%
    \\
    \subfloat[Oscillating inward batching]{{\includegraphics[width=0.4\linewidth]{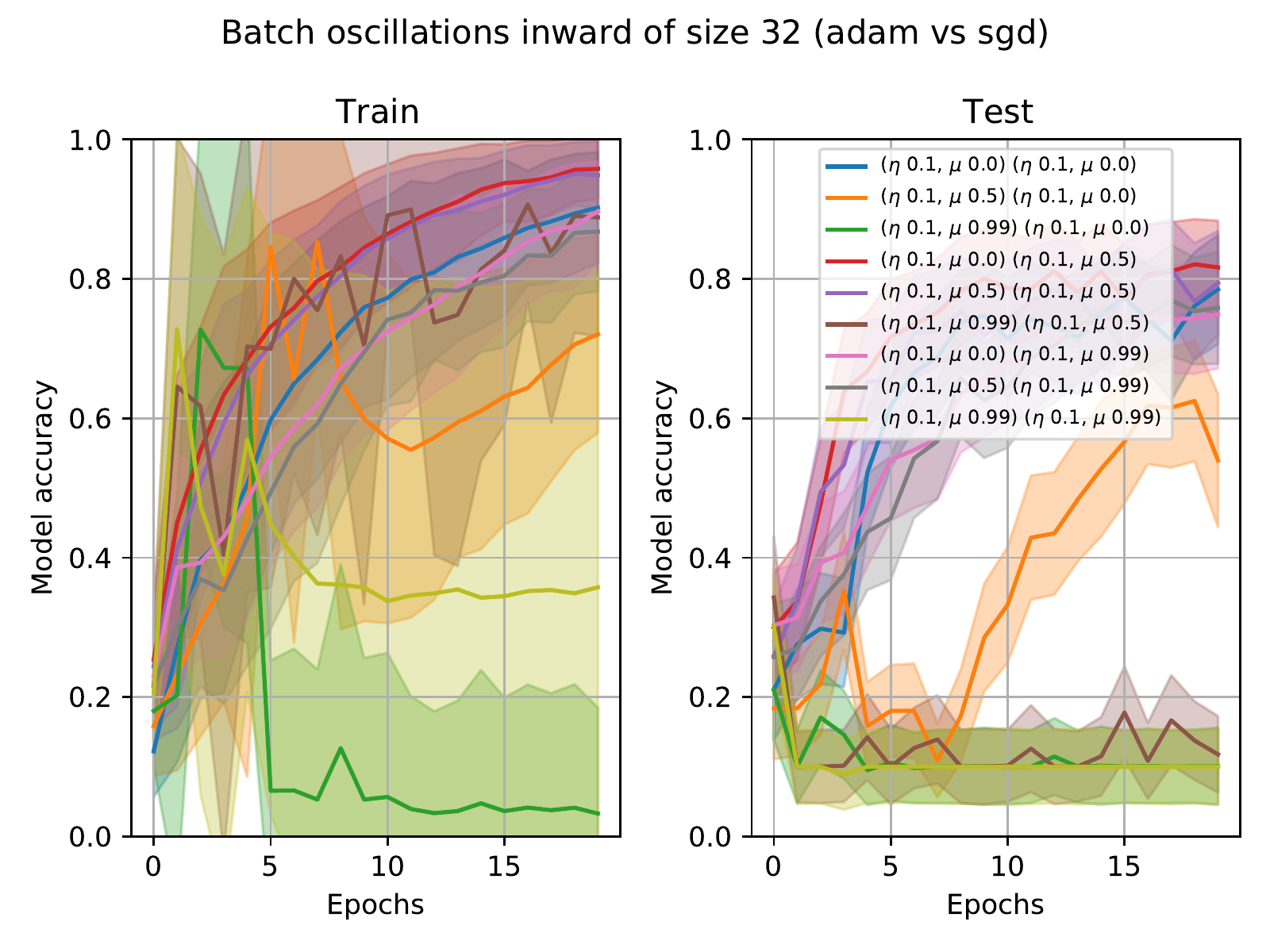} }}%
    \qquad
    \subfloat[Oscillating outward batching]{{\includegraphics[width=0.4\linewidth]{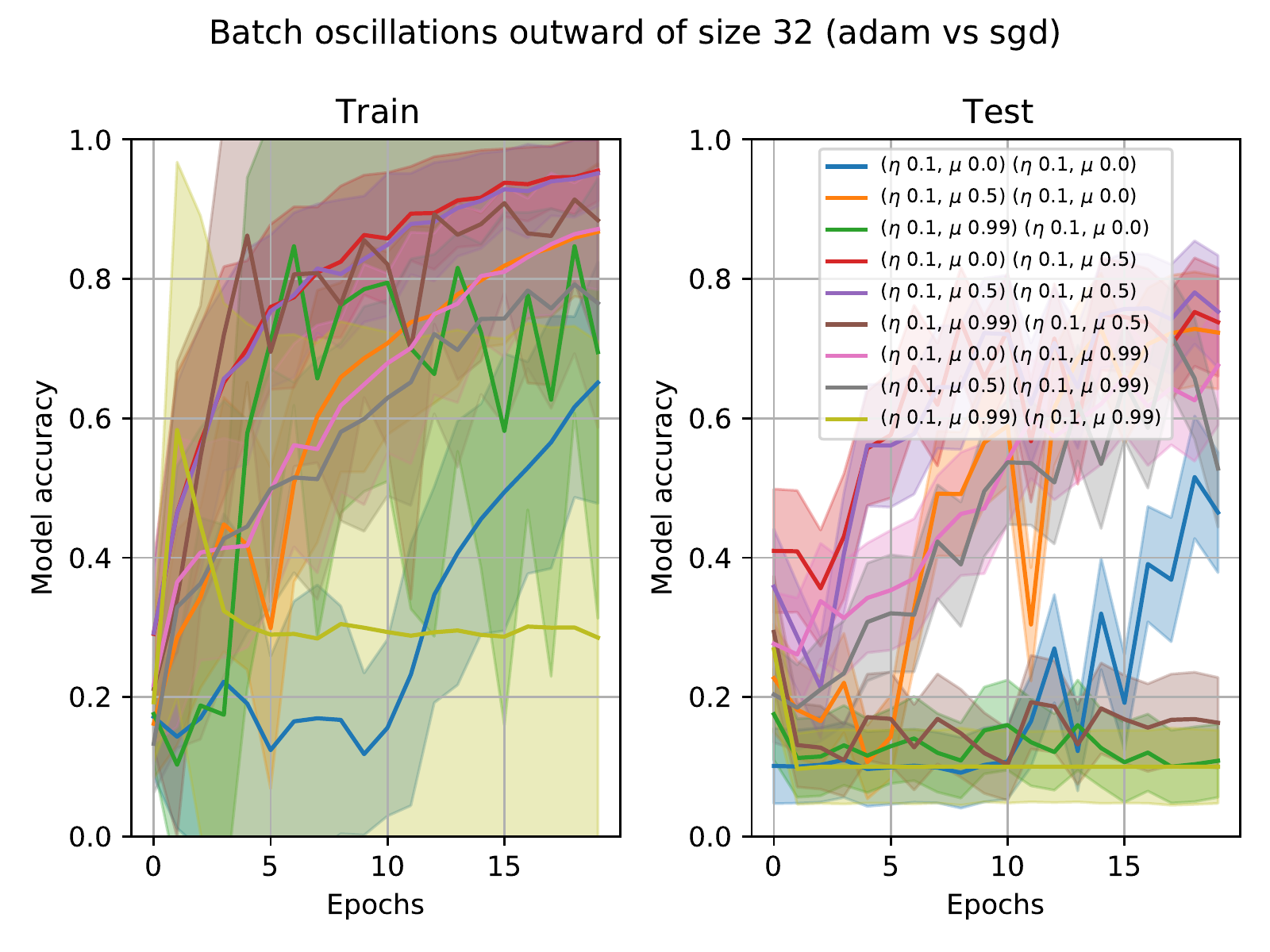} }}%
    \caption{ResNet18 real model Adam training, LeNet5 surrogate with SGD and Batchsize 32}%
    \label{fig:hyper4}%
\end{figure}

\begin{figure}[h]%
    \centering
    \subfloat[High low batching]{{\includegraphics[width=0.4\linewidth]{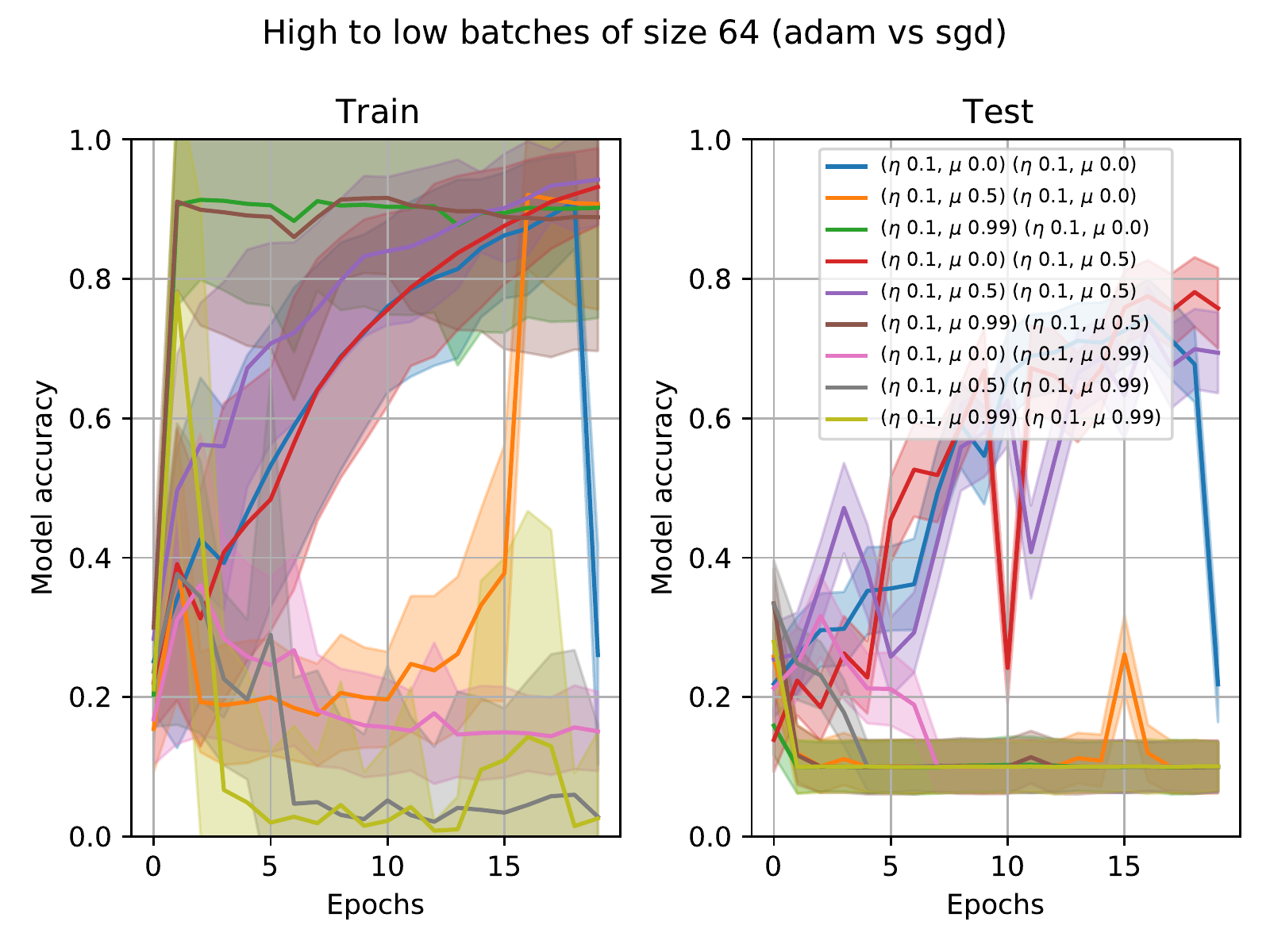} }}%
    \qquad
    \subfloat[Low high batching]{{\includegraphics[width=0.4\linewidth]{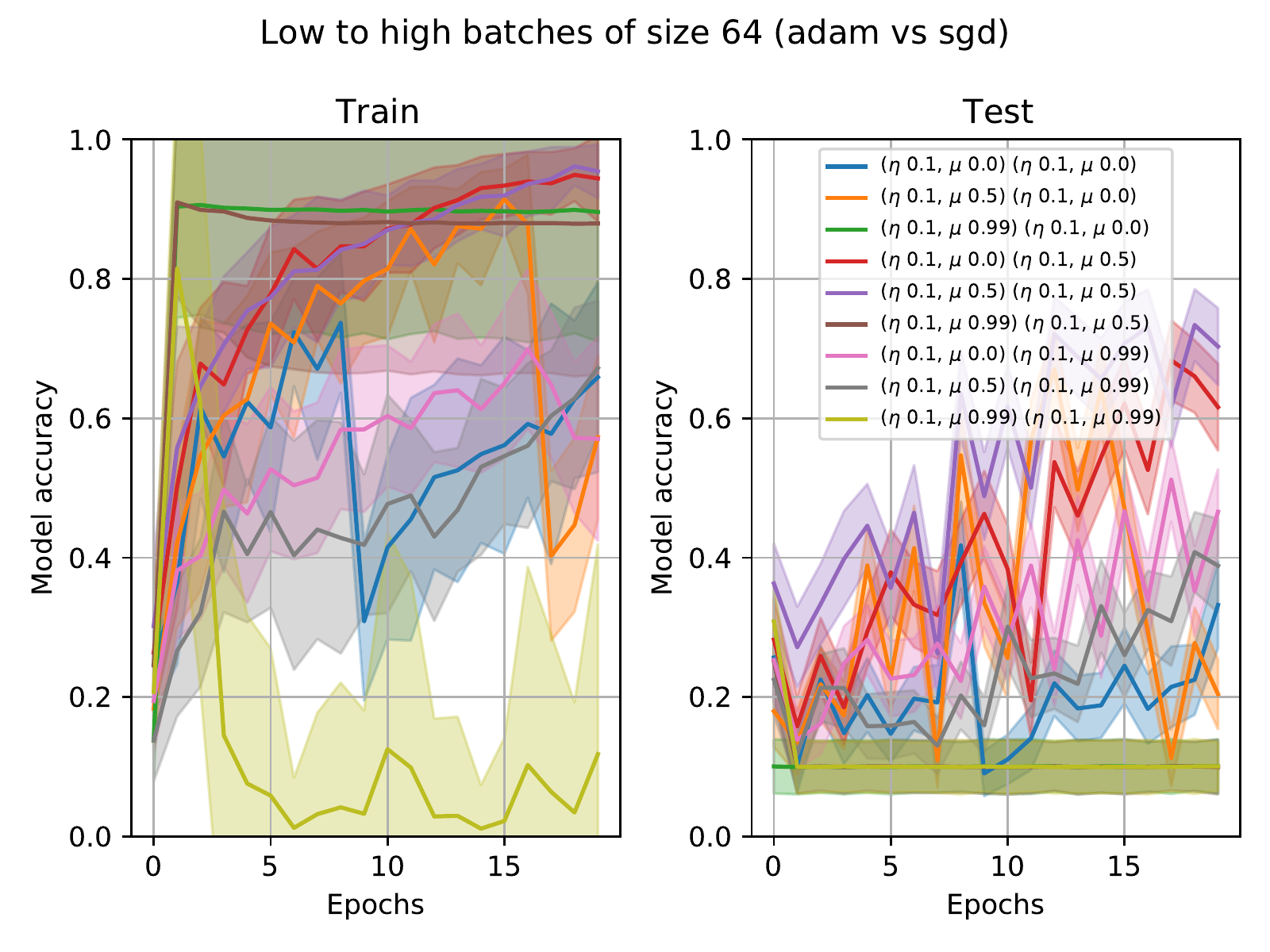} }}%
    \\
    \subfloat[Oscillating inward batching]{{\includegraphics[width=0.4\linewidth]{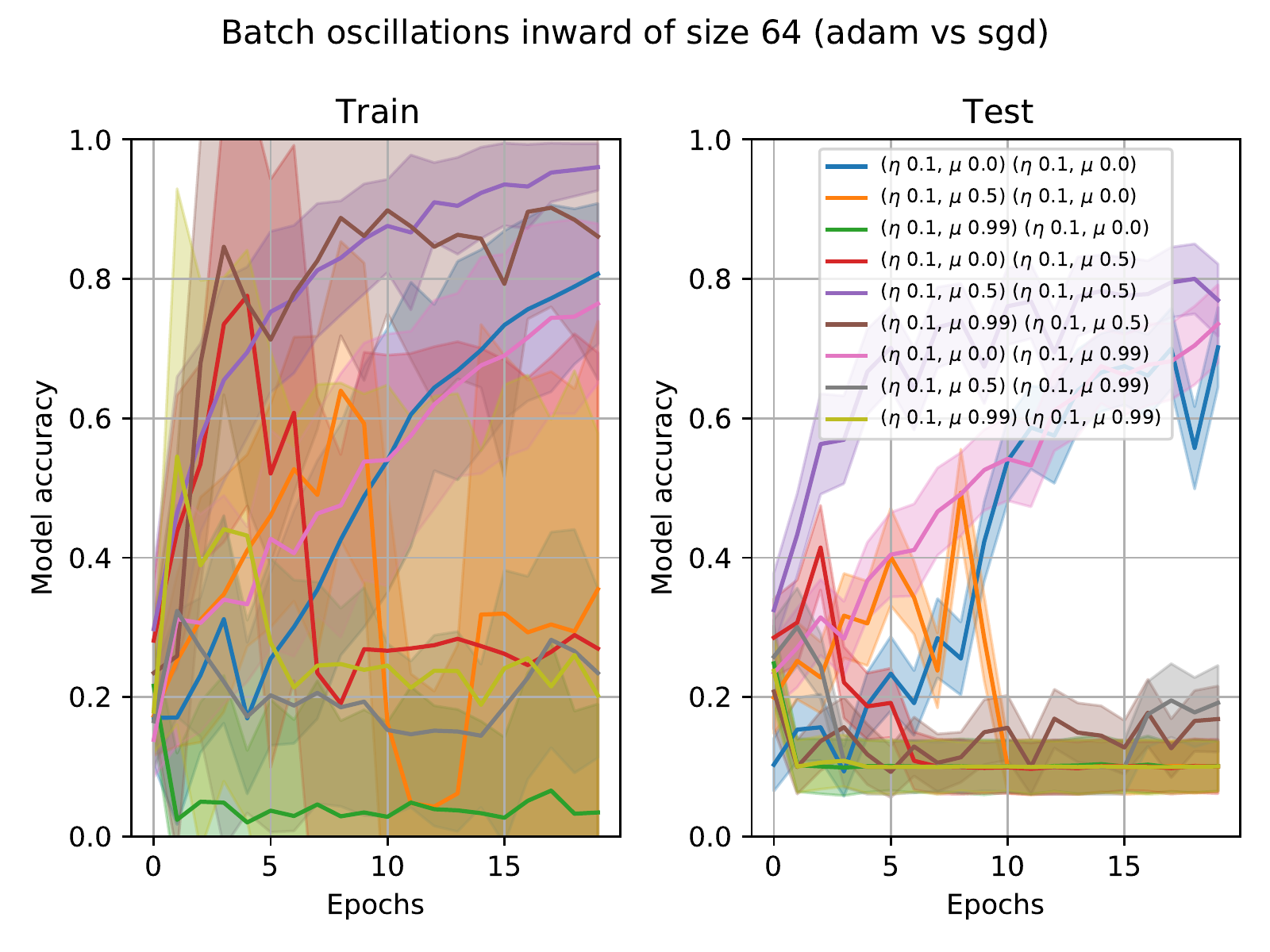} }}%
    \qquad
    \subfloat[Oscillating outward batching]{{\includegraphics[width=0.4\linewidth]{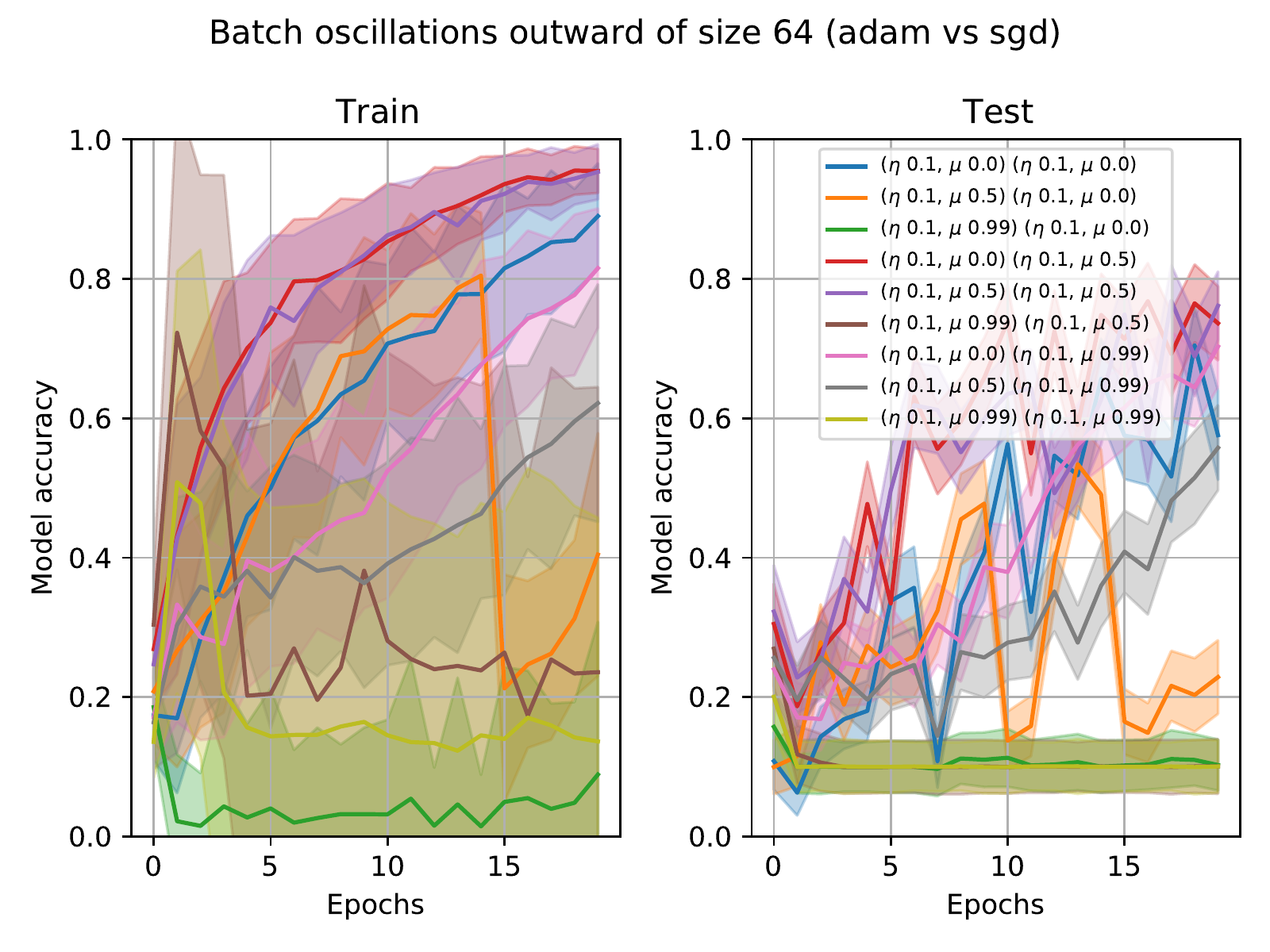} }}%
     \caption{ResNet18 real model Adam training, LeNet5 surrogate with SGD and Batchsize 64}%
    \label{fig:hyper5}%
\end{figure}

\begin{figure}[h]%
    \centering
    \subfloat[High low batching]{{\includegraphics[width=0.4\linewidth]{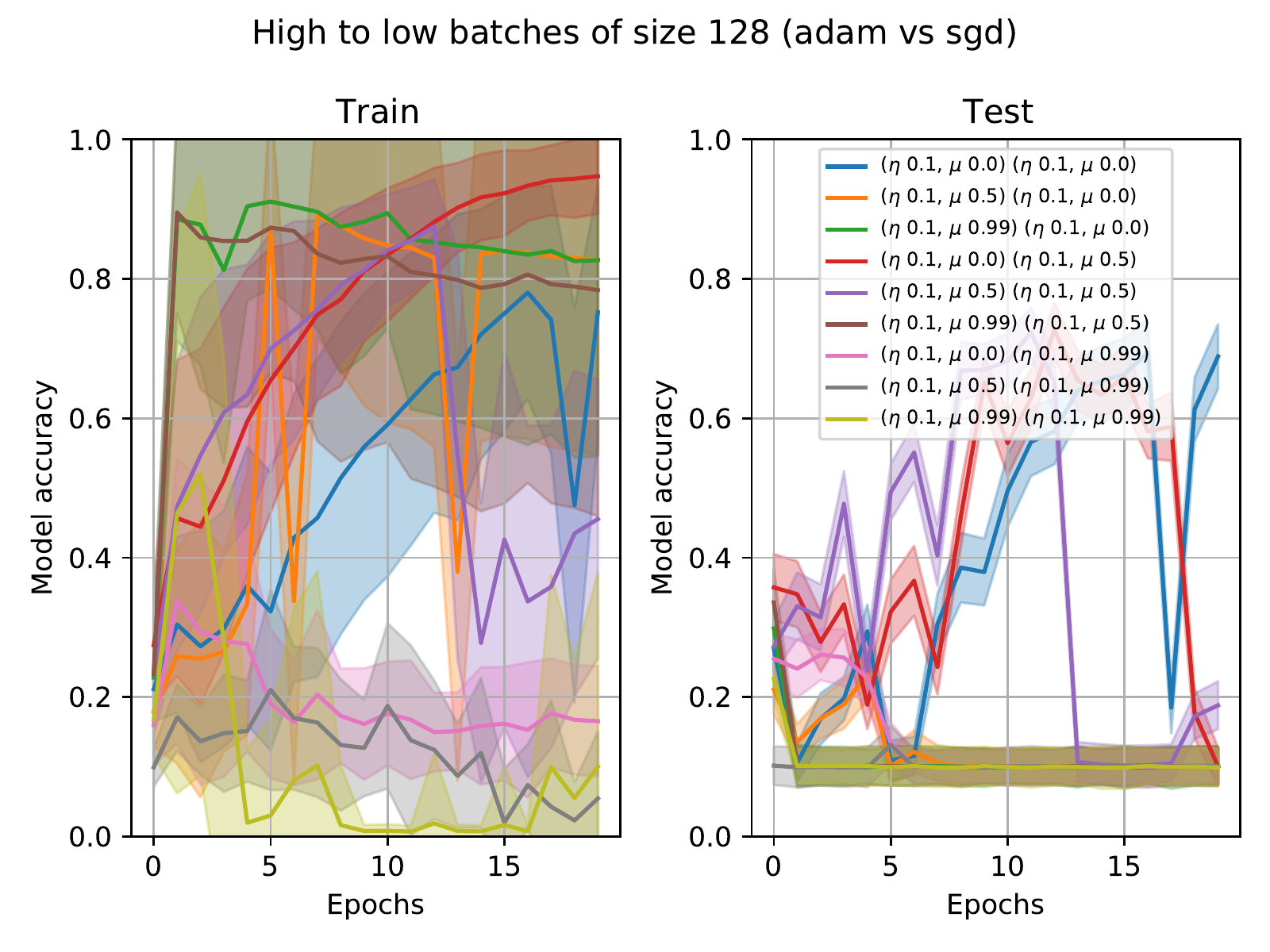} }}%
    \qquad
    \subfloat[Low high batching]{{\includegraphics[width=0.4\linewidth]{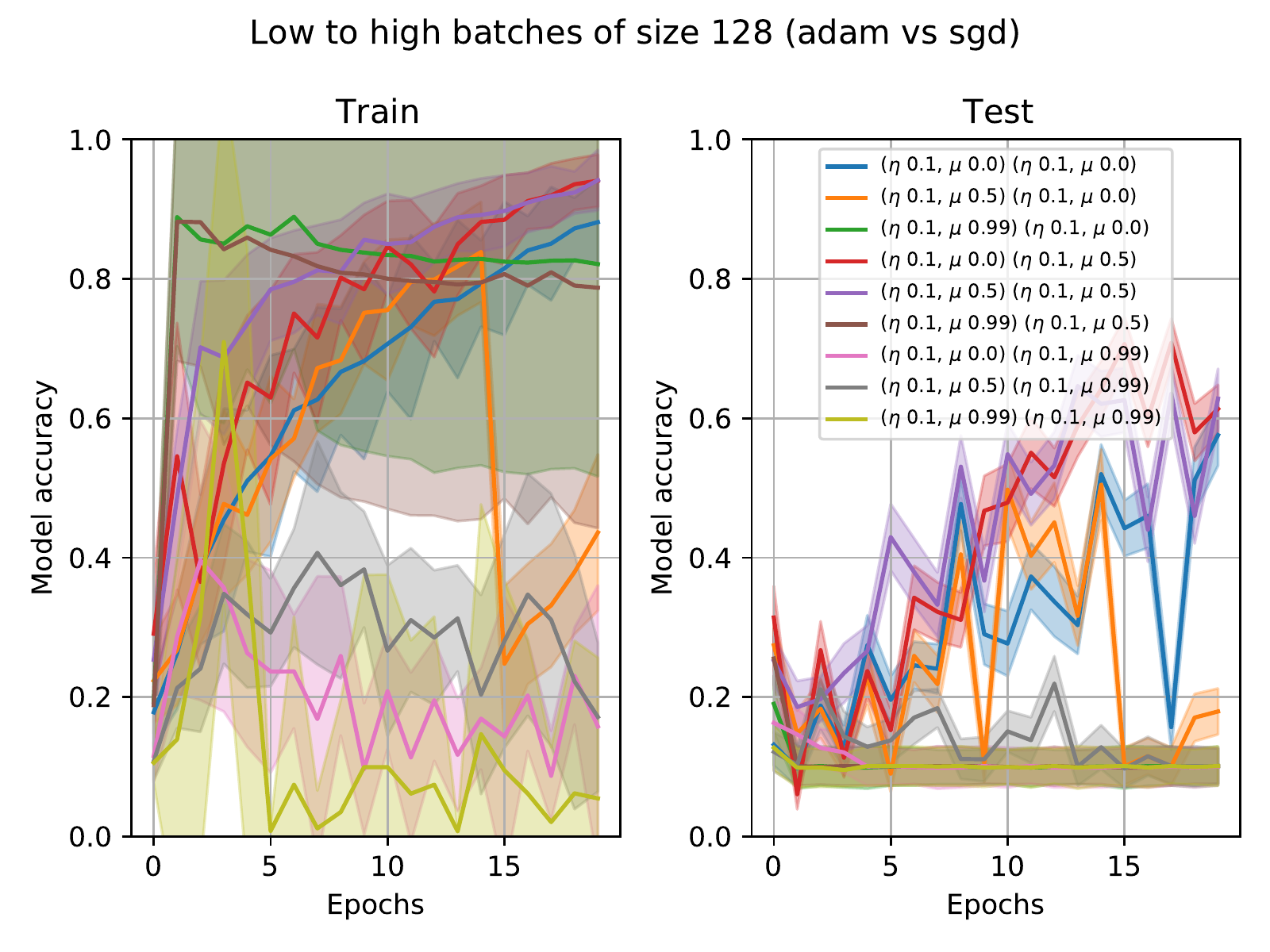} }}%
    \\
    \subfloat[Oscillating inward batching]{{\includegraphics[width=0.4\linewidth]{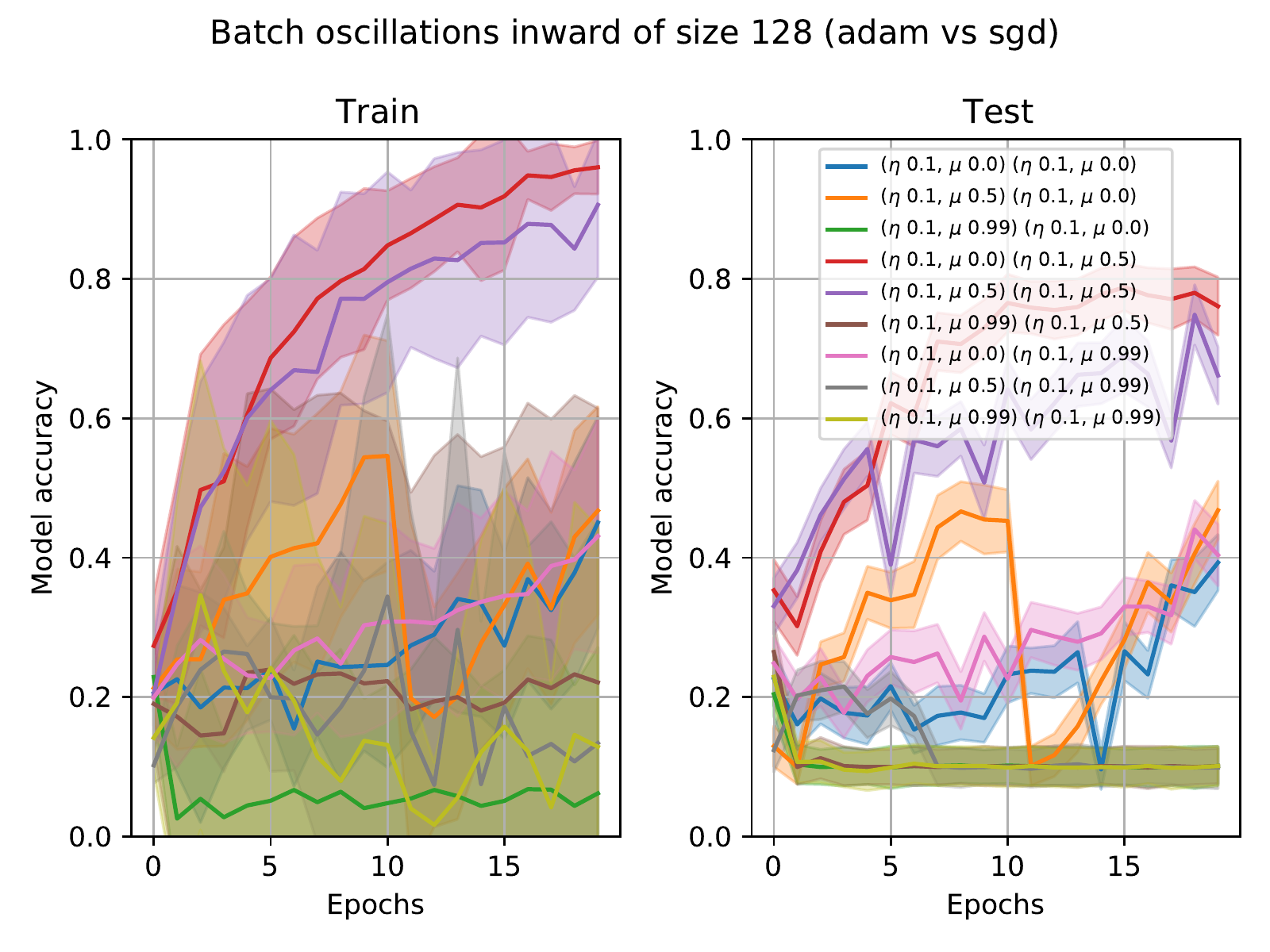} }}%
    \qquad
    \subfloat[Oscillating outward batching]{{\includegraphics[width=0.4\linewidth]{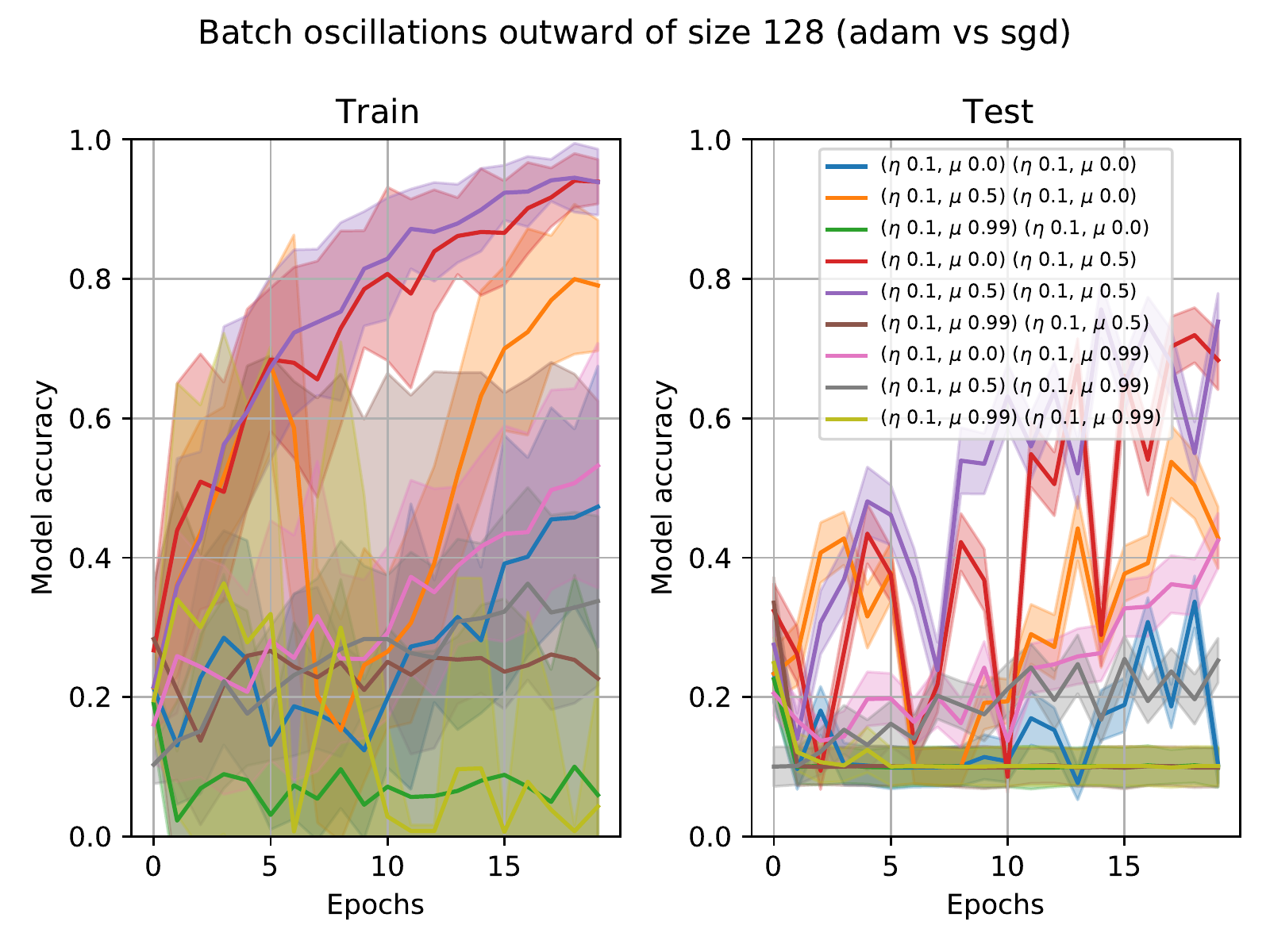} }}%
     \caption{ResNet18 real model Adam training, LeNet5 surrogate with SGD and Batchsize 128}
    \label{fig:hyper6}%
\end{figure}


\begin{figure}[h]%
    \centering
    \subfloat[High low batching]{{\includegraphics[width=0.4\linewidth]{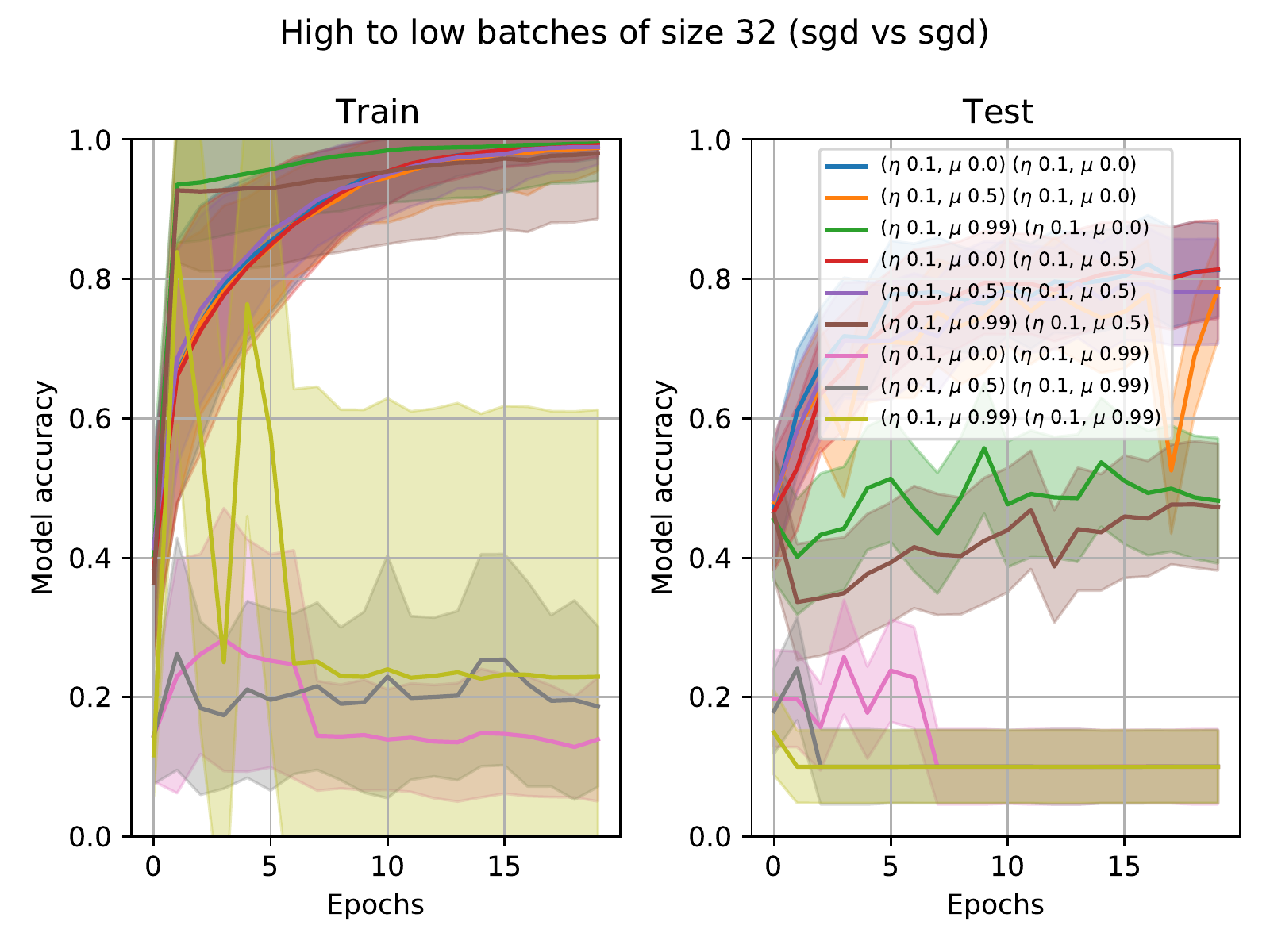} }}%
    \qquad
    \subfloat[Low high batching]{{\includegraphics[width=0.4\linewidth]{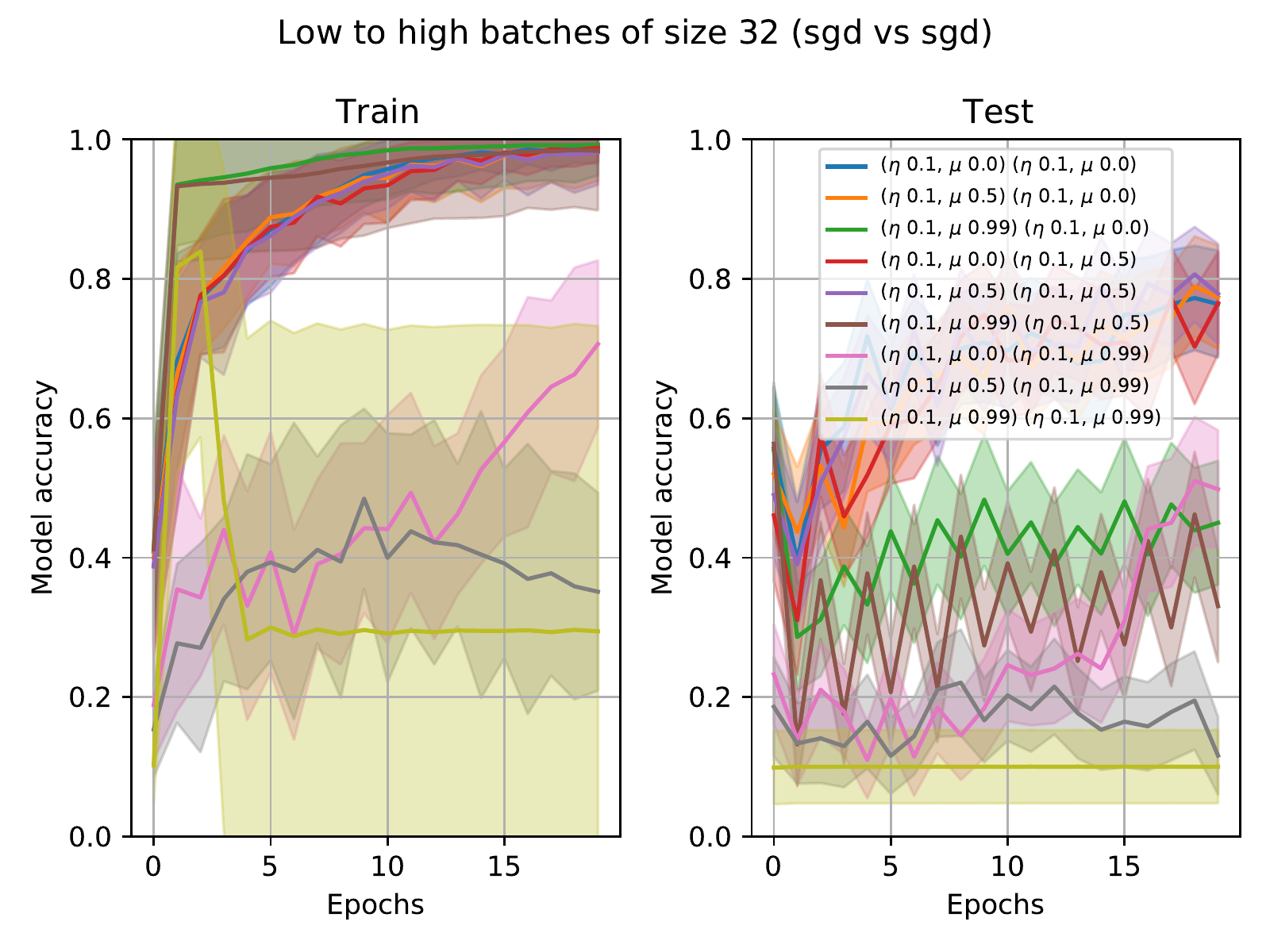} }}%
    \\
    \subfloat[Oscillating inward batching]{{\includegraphics[width=0.4\linewidth]{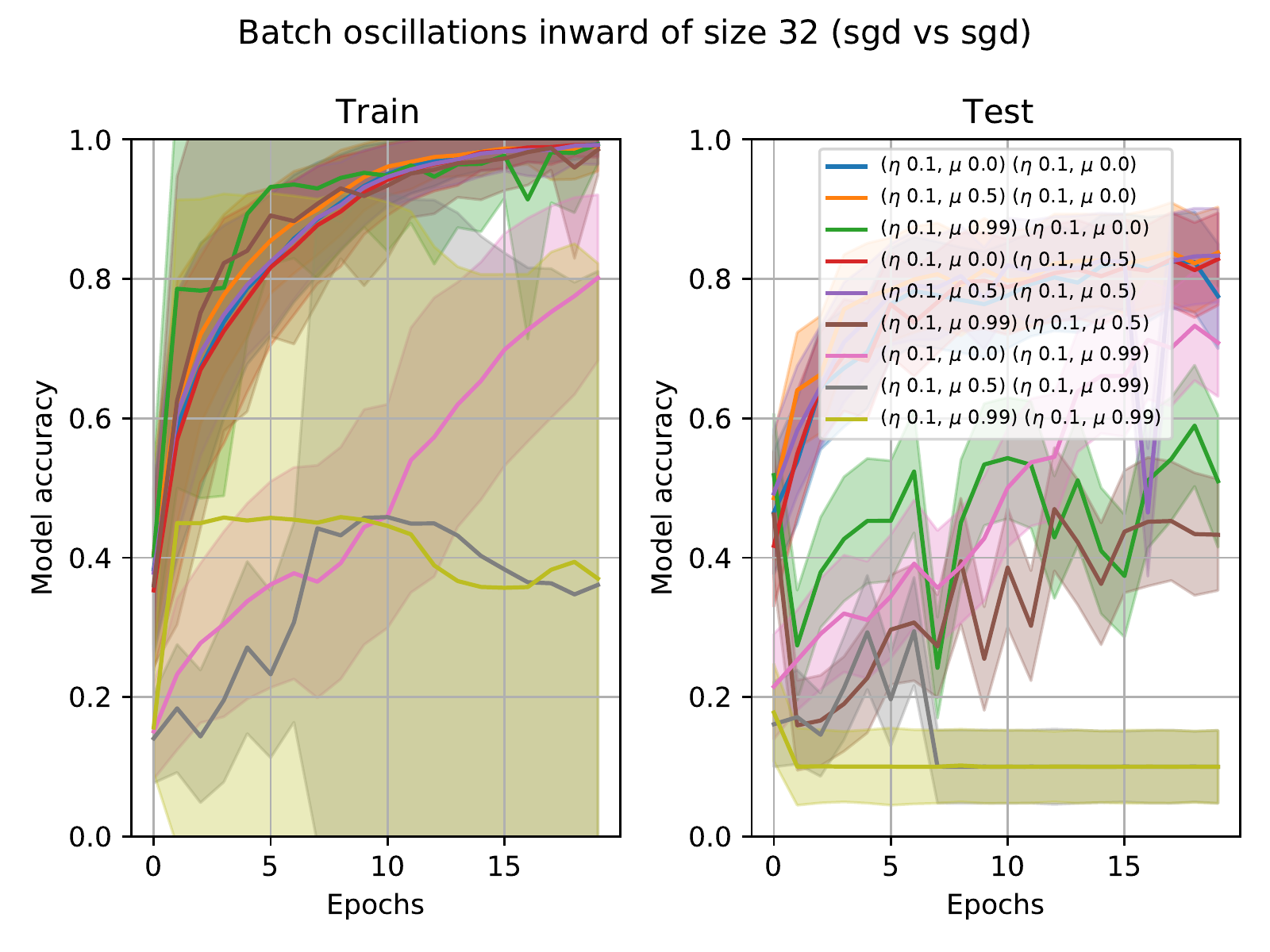} }}%
    \qquad
    \subfloat[Oscillating outward batching]{{\includegraphics[width=0.4\linewidth]{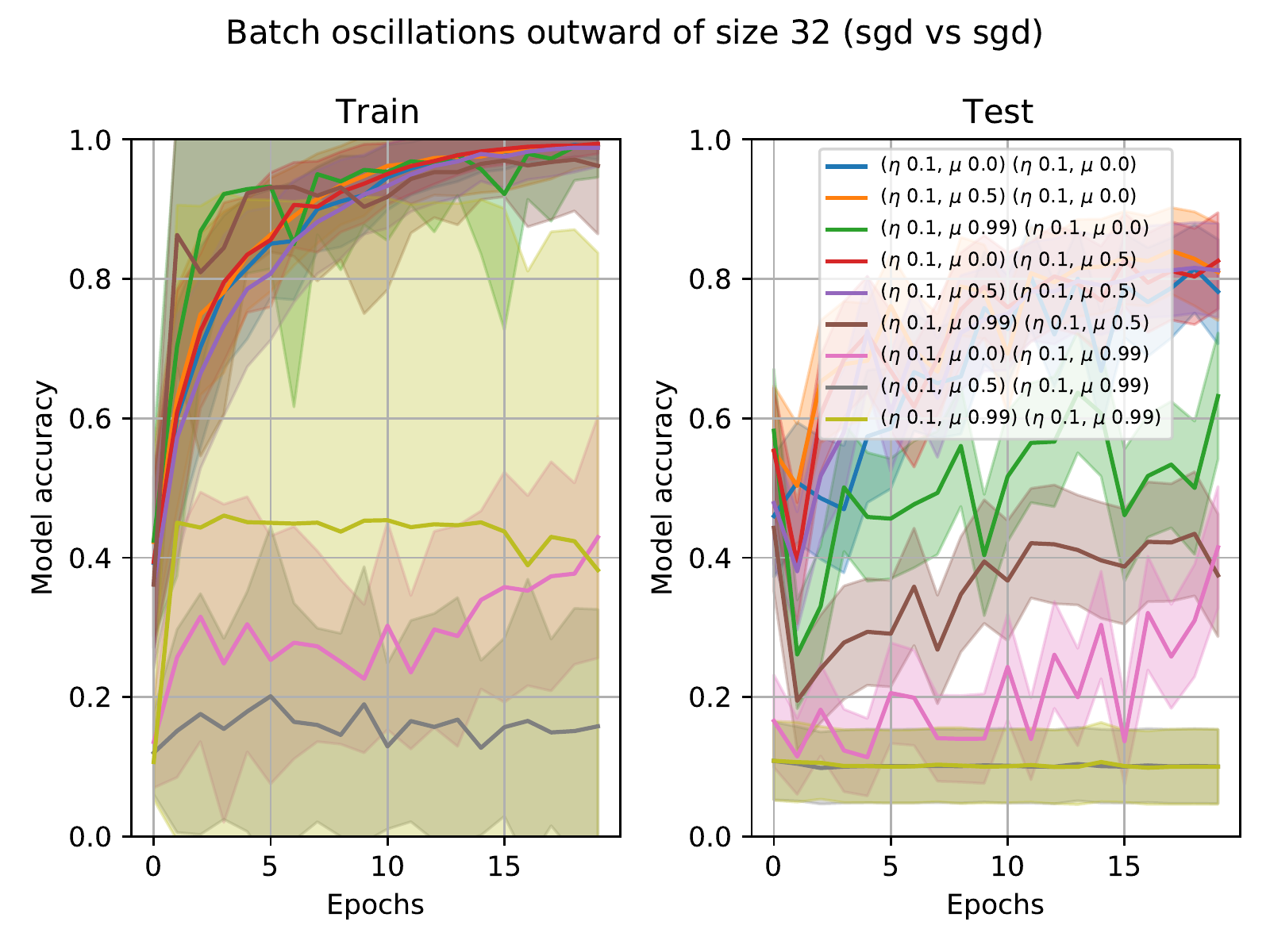} }}%
   \caption{ResNet18 real model SGD training, LeNet5 surrogate with SGD and Batchsize 32}%
    \label{fig:hyper7}%
\end{figure}

\begin{figure}[h]%
    \centering
    \subfloat[High low batching]{{\includegraphics[width=0.4\linewidth]{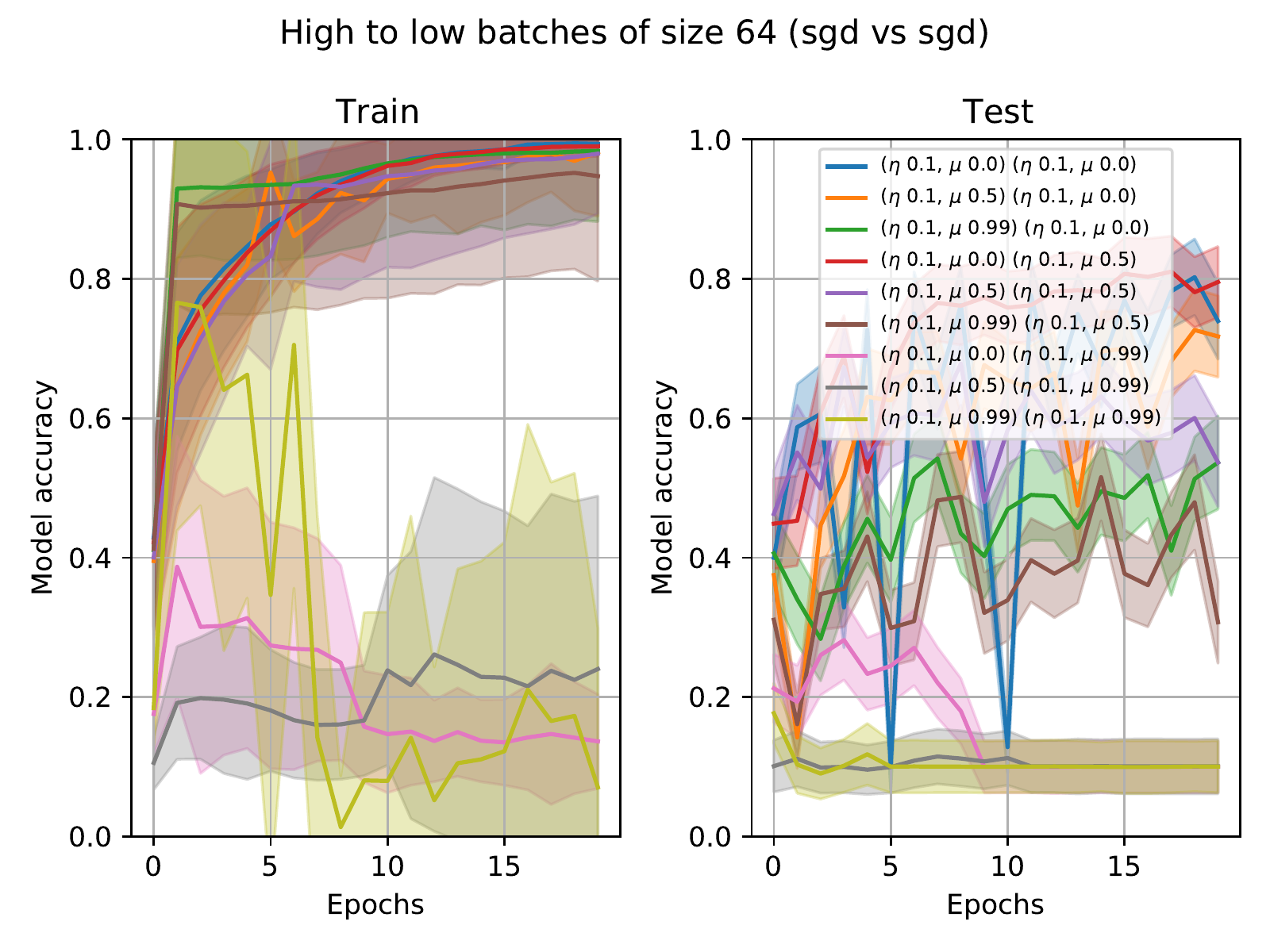} }}%
    \qquad
    \subfloat[Low high batching]{{\includegraphics[width=0.4\linewidth]{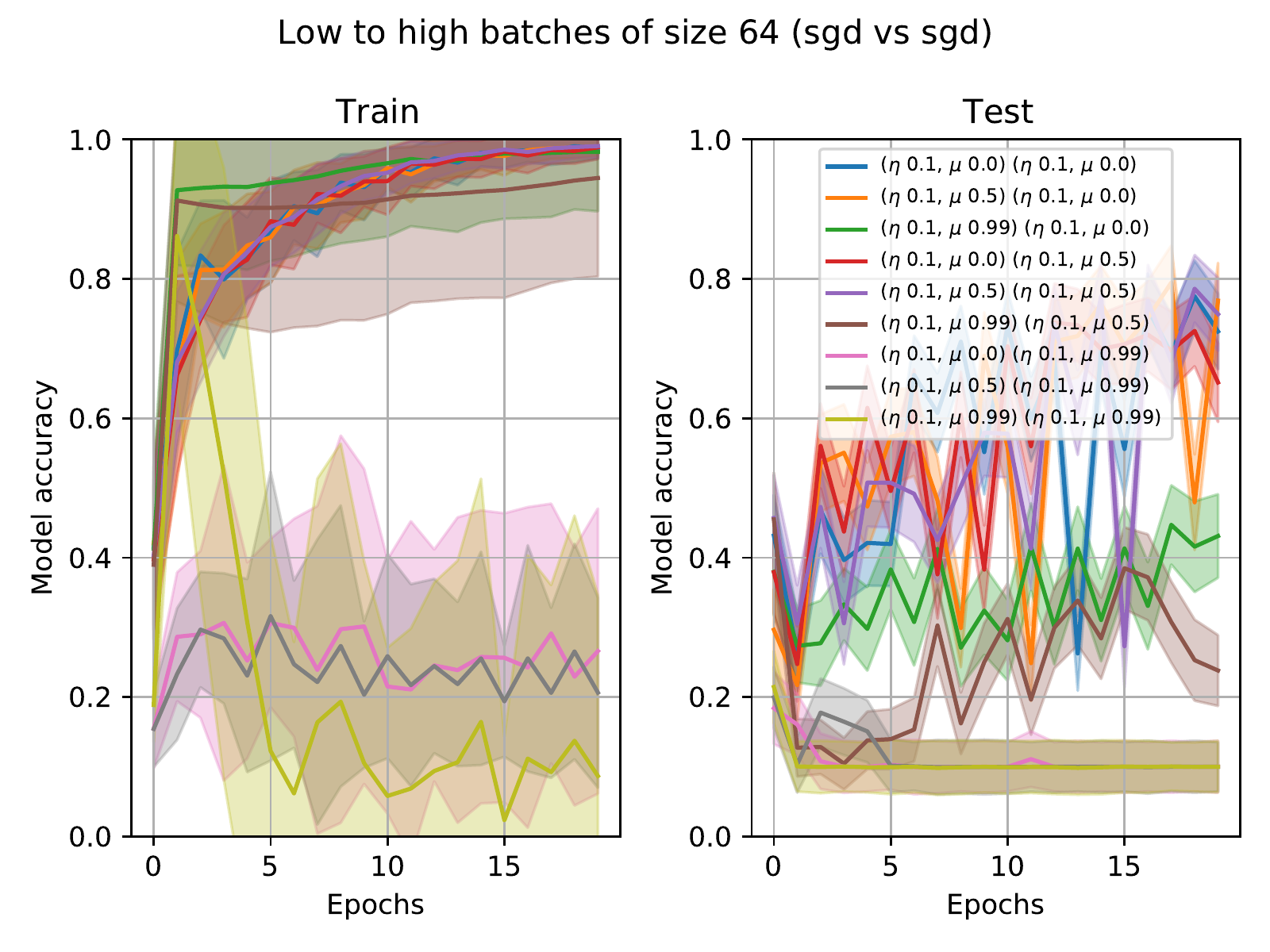} }}%
    \\
    \subfloat[Oscillating inward batching]{{\includegraphics[width=0.4\linewidth]{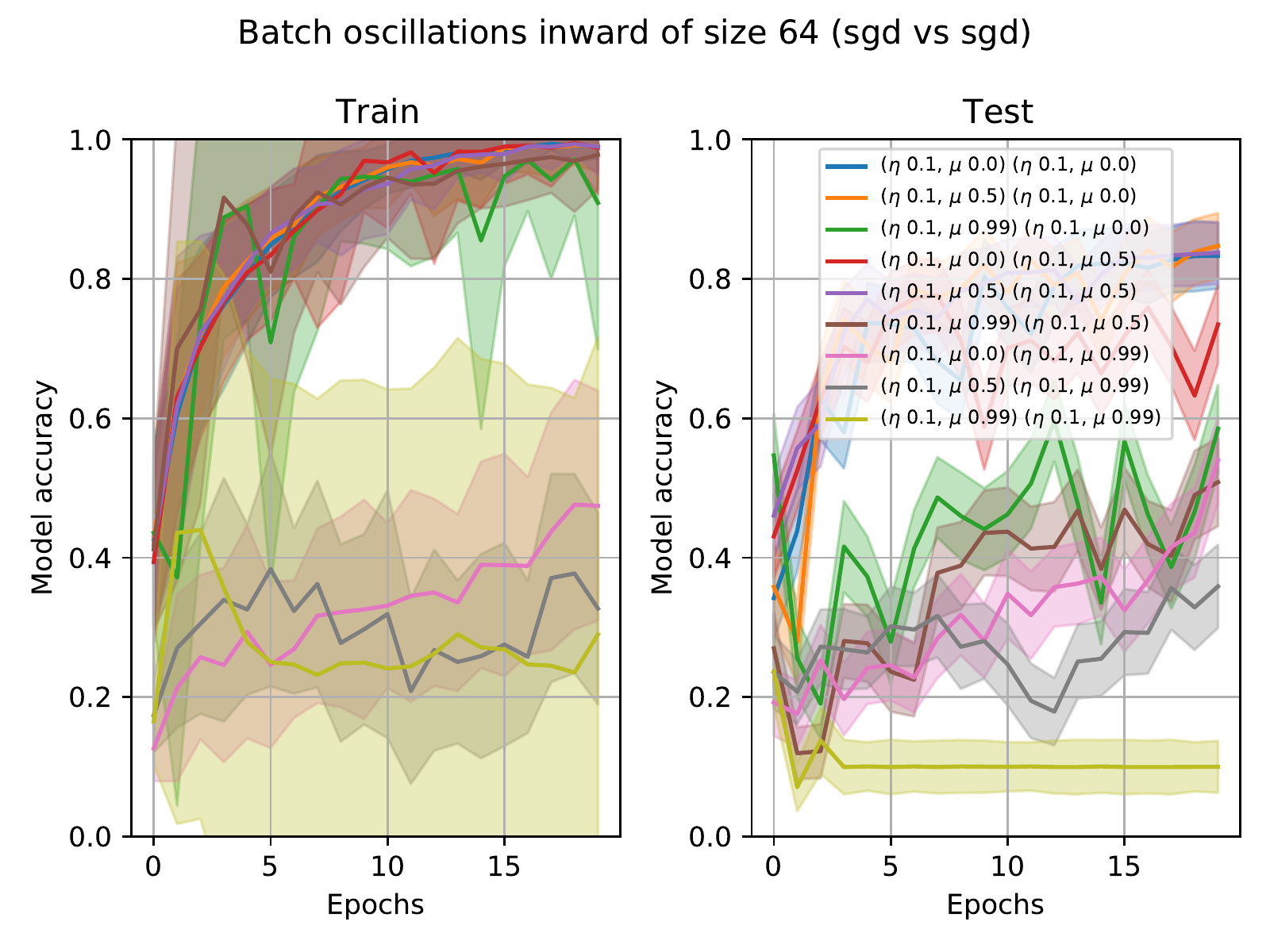} }}%
    \qquad
    \subfloat[Oscillating outward batching]{{\includegraphics[width=0.4\linewidth]{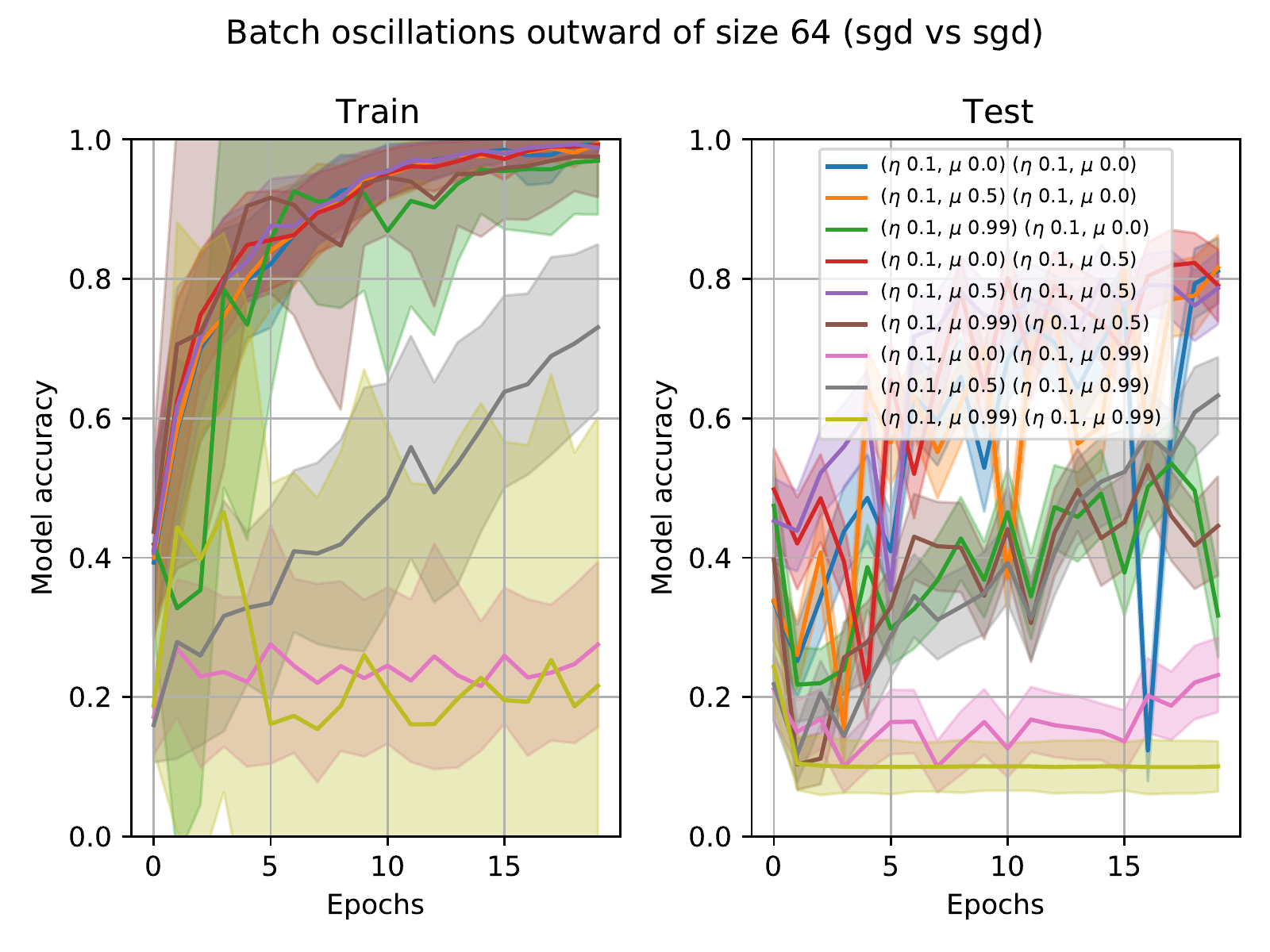} }}%
   \caption{ResNet18 real model SGD training, LeNet5 surrogate with SGD and Batchsize 64}%
    \label{fig:hyper8}%
\end{figure}

\begin{figure}[h]%
    \centering
    \subfloat[High low batching]{{\includegraphics[width=0.4\linewidth]{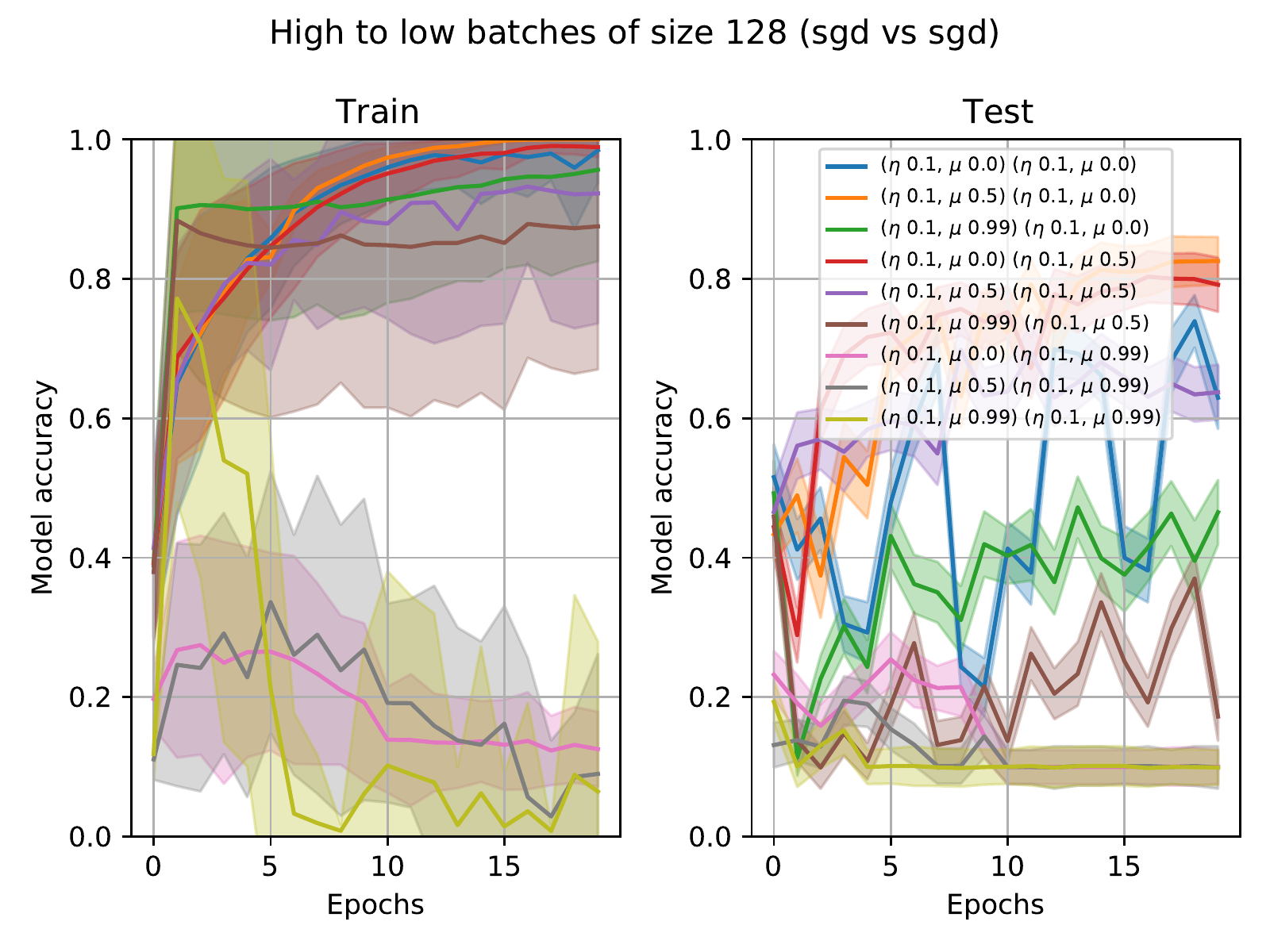} }}%
    \qquad
    \subfloat[Low high batching]{{\includegraphics[width=0.4\linewidth]{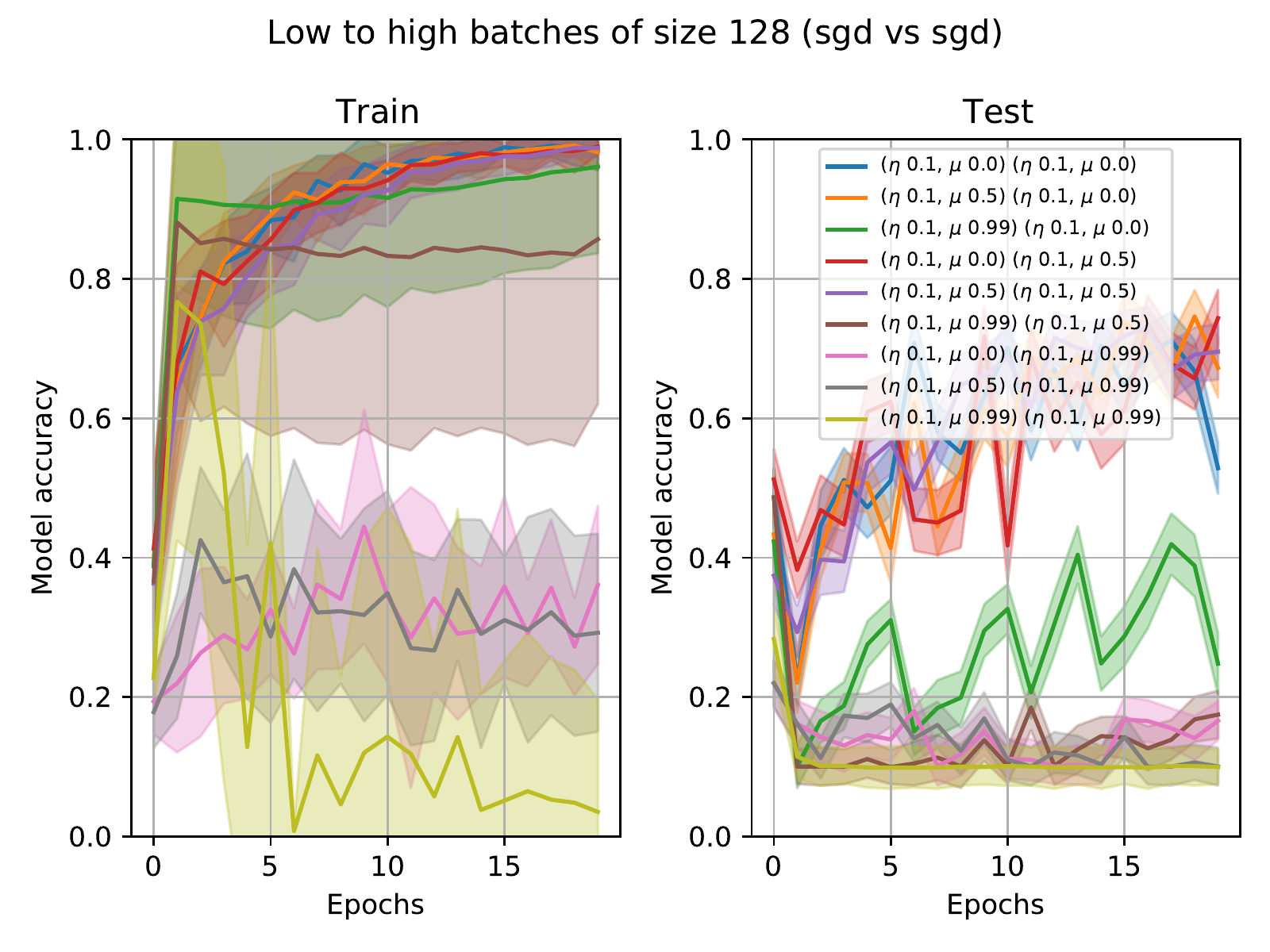} }}%
    \\
    \subfloat[Oscillating inward batching]{{\includegraphics[width=0.4\linewidth]{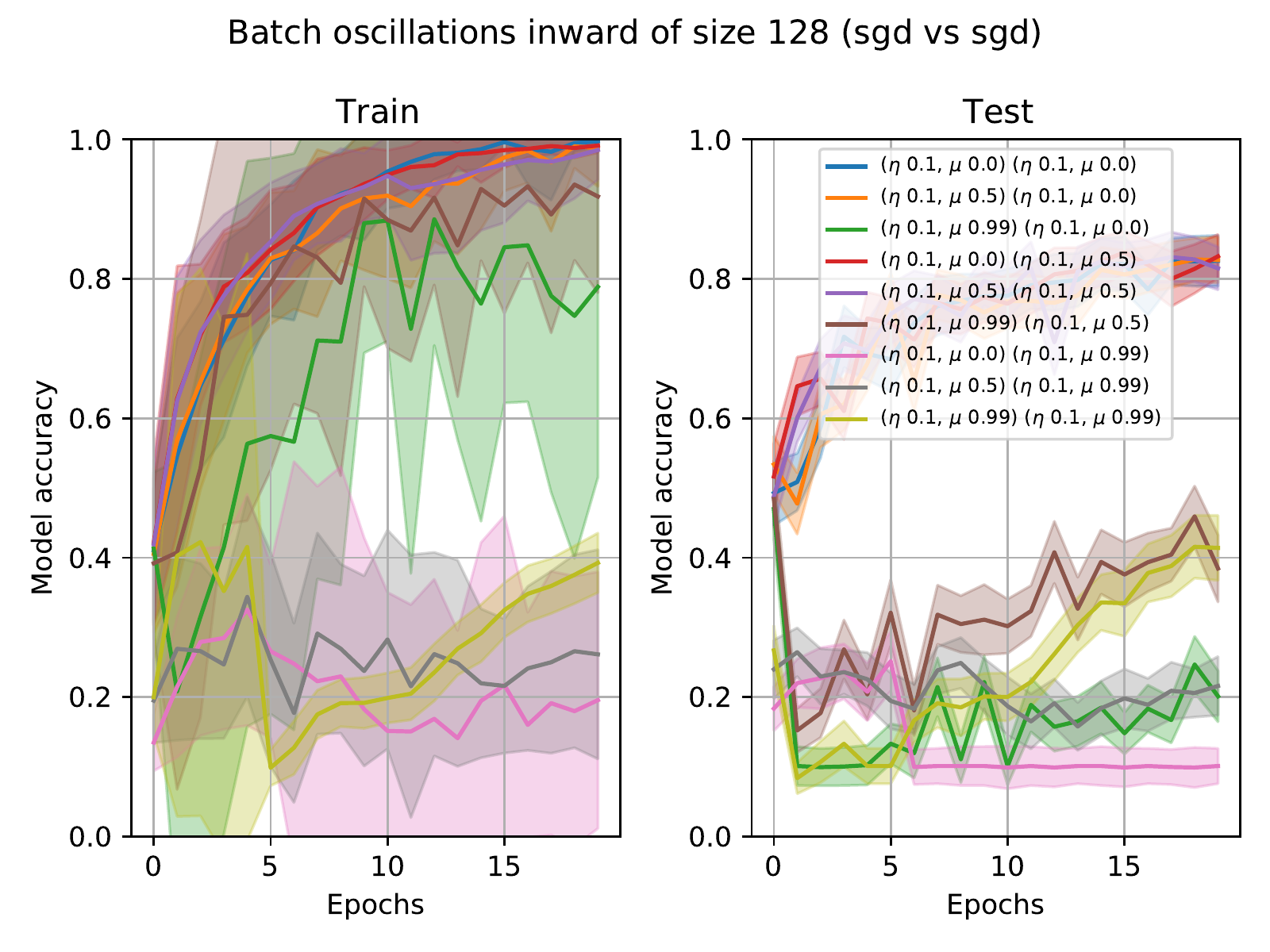} }}%
    \qquad
    \subfloat[Oscillating outward batching]{{\includegraphics[width=0.4\linewidth]{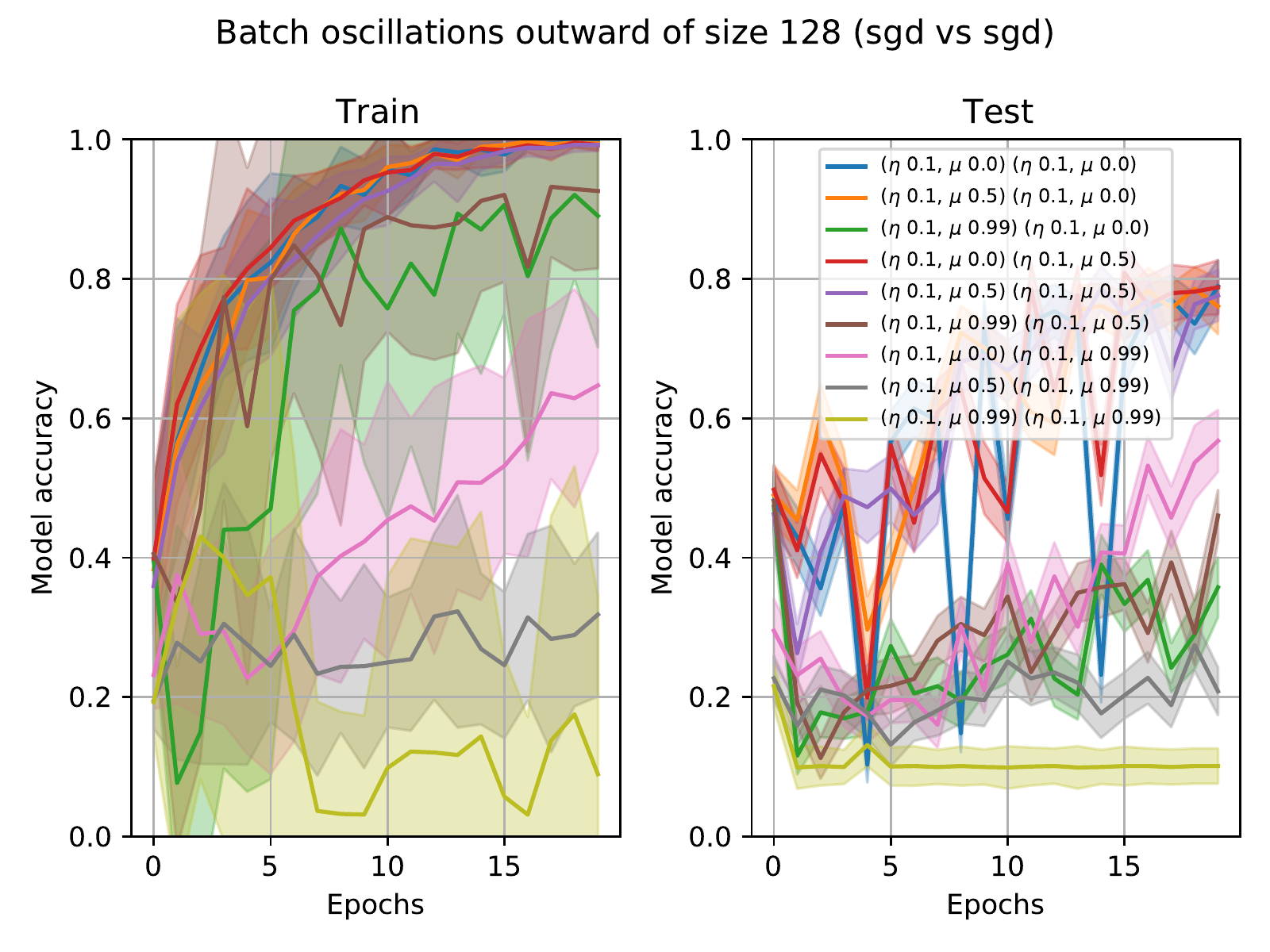} }}%
   \caption{ResNet18 real model SGD training, LeNet5 surrogate with SGD and Batchsize 128}%
    \label{fig:hyper9}%
\end{figure}


\begin{figure}[h]%
    \centering
    \subfloat[High low batching]{{\includegraphics[width=0.4\linewidth]{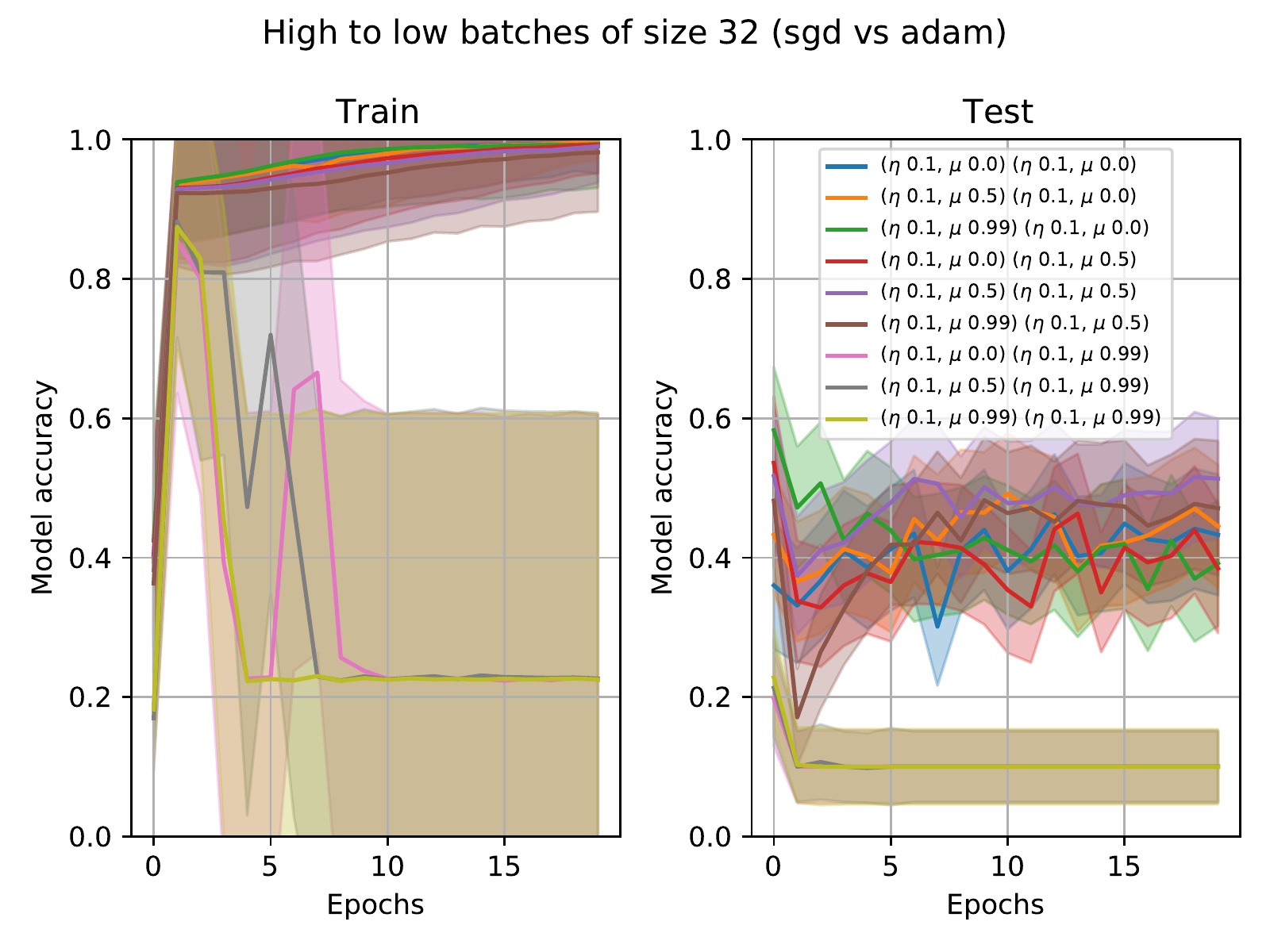} }}%
    \qquad
    \subfloat[Low high batching]{{\includegraphics[width=0.4\linewidth]{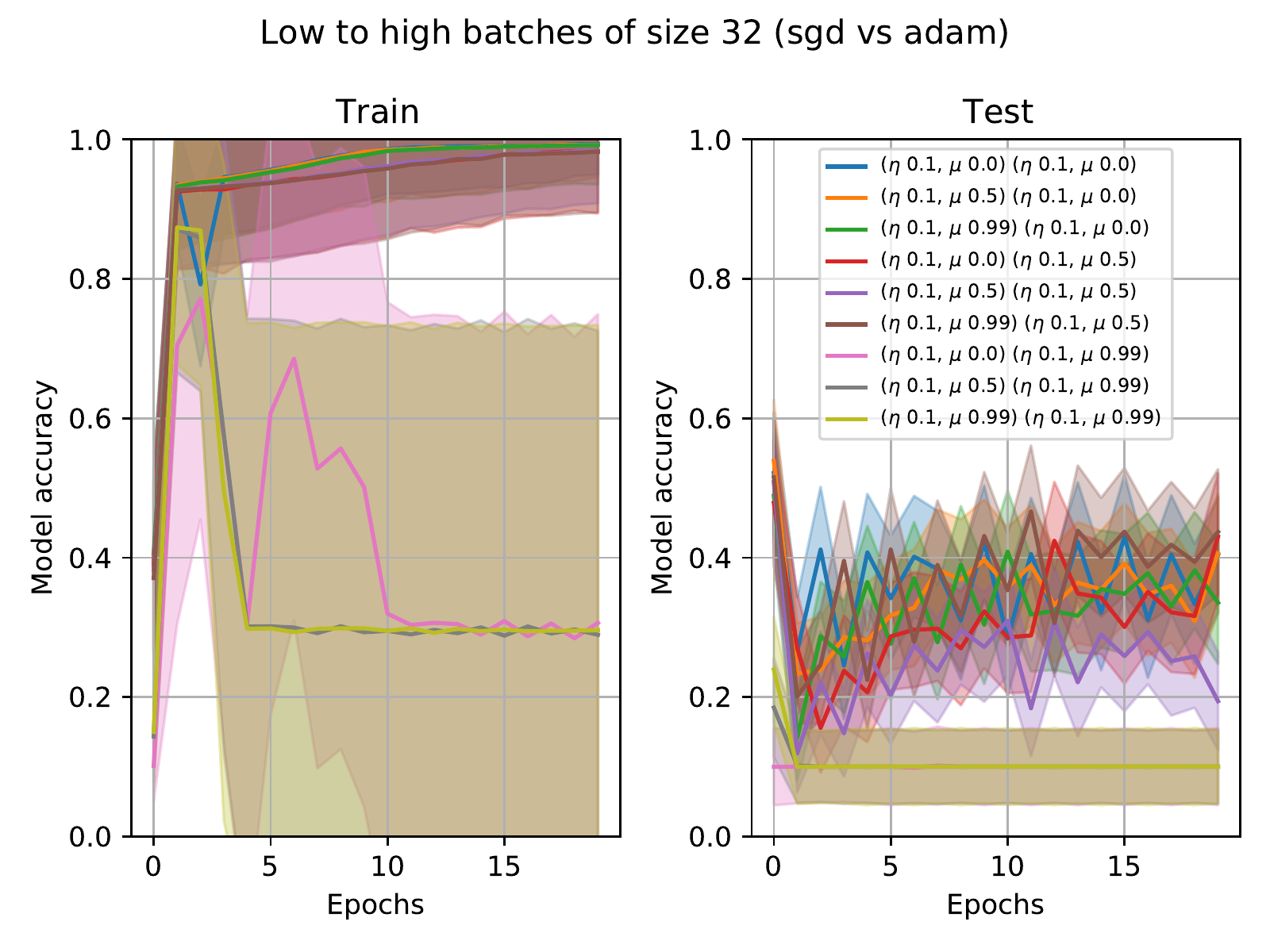} }}%
    \\
    \subfloat[Oscillating inward batching]{{\includegraphics[width=0.4\linewidth]{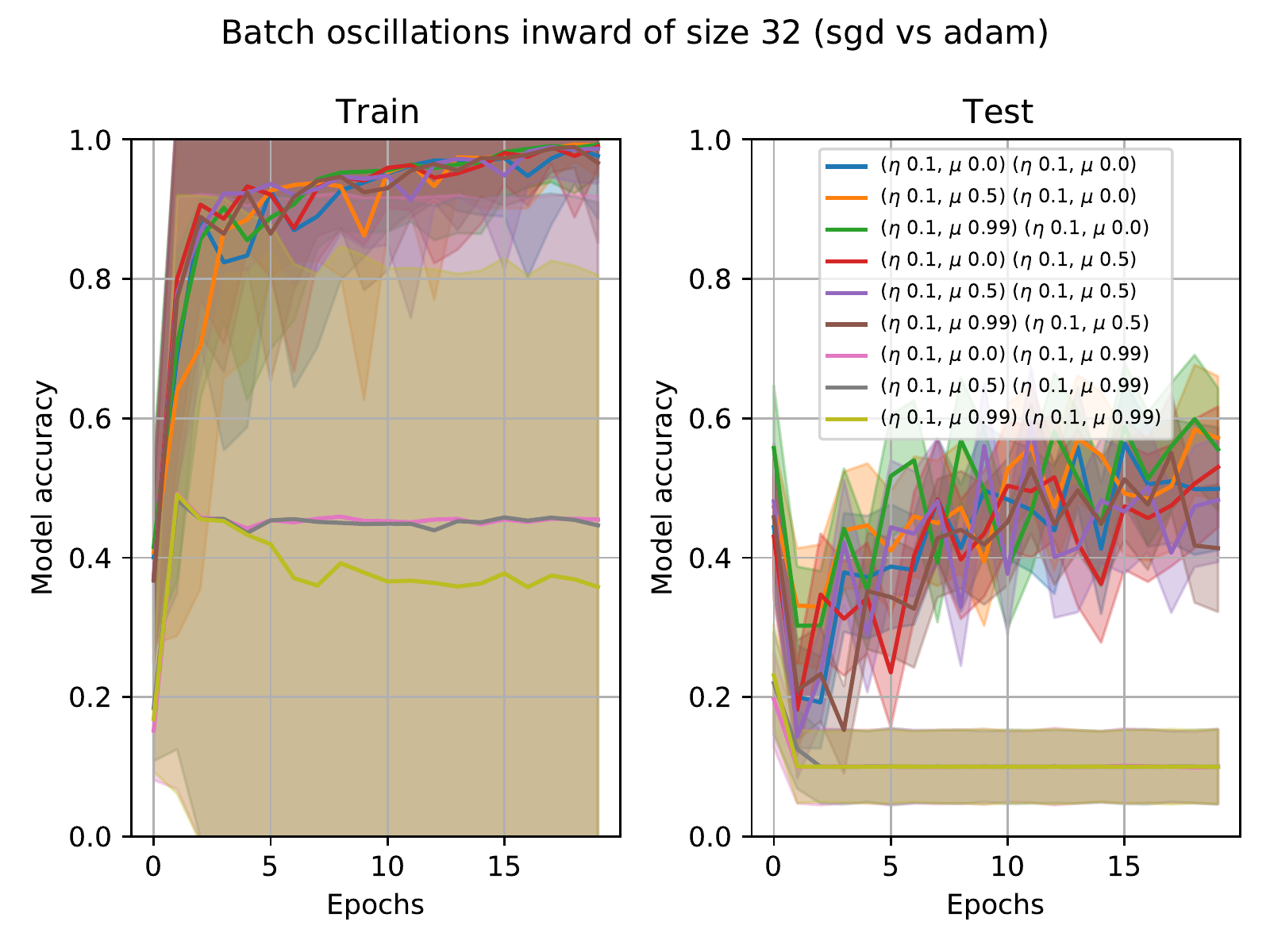} }}%
    \qquad
    \subfloat[Oscillating outward batching]{{\includegraphics[width=0.4\linewidth]{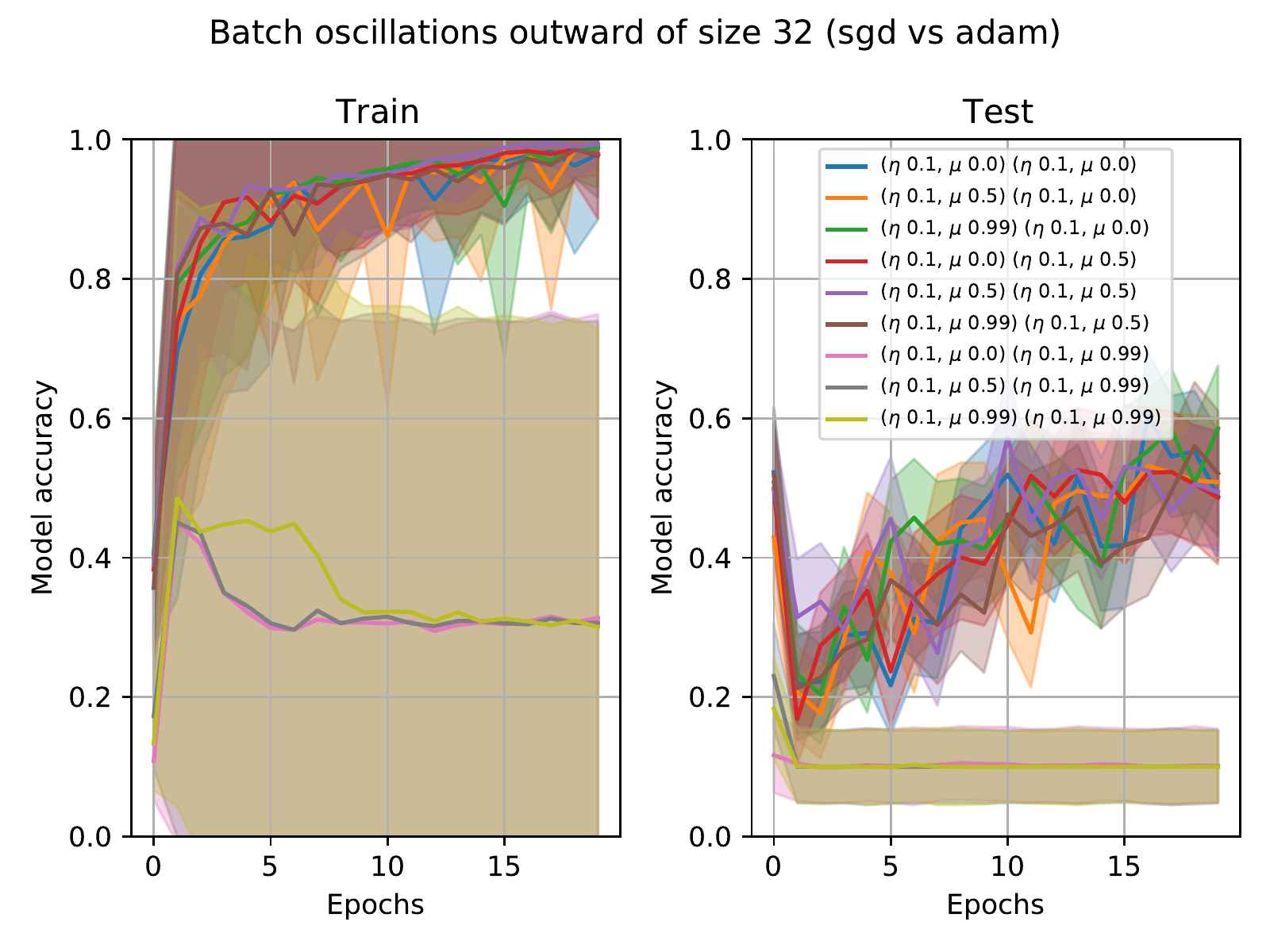} }}%
    \caption{ResNet18 real model SGD training, LeNet5 surrogate with Adam and Batchsize 32}%
    \label{fig:hyper10}%
\end{figure}

\begin{figure}[h]%
    \centering
    \subfloat[High low batching]{{\includegraphics[width=0.4\linewidth]{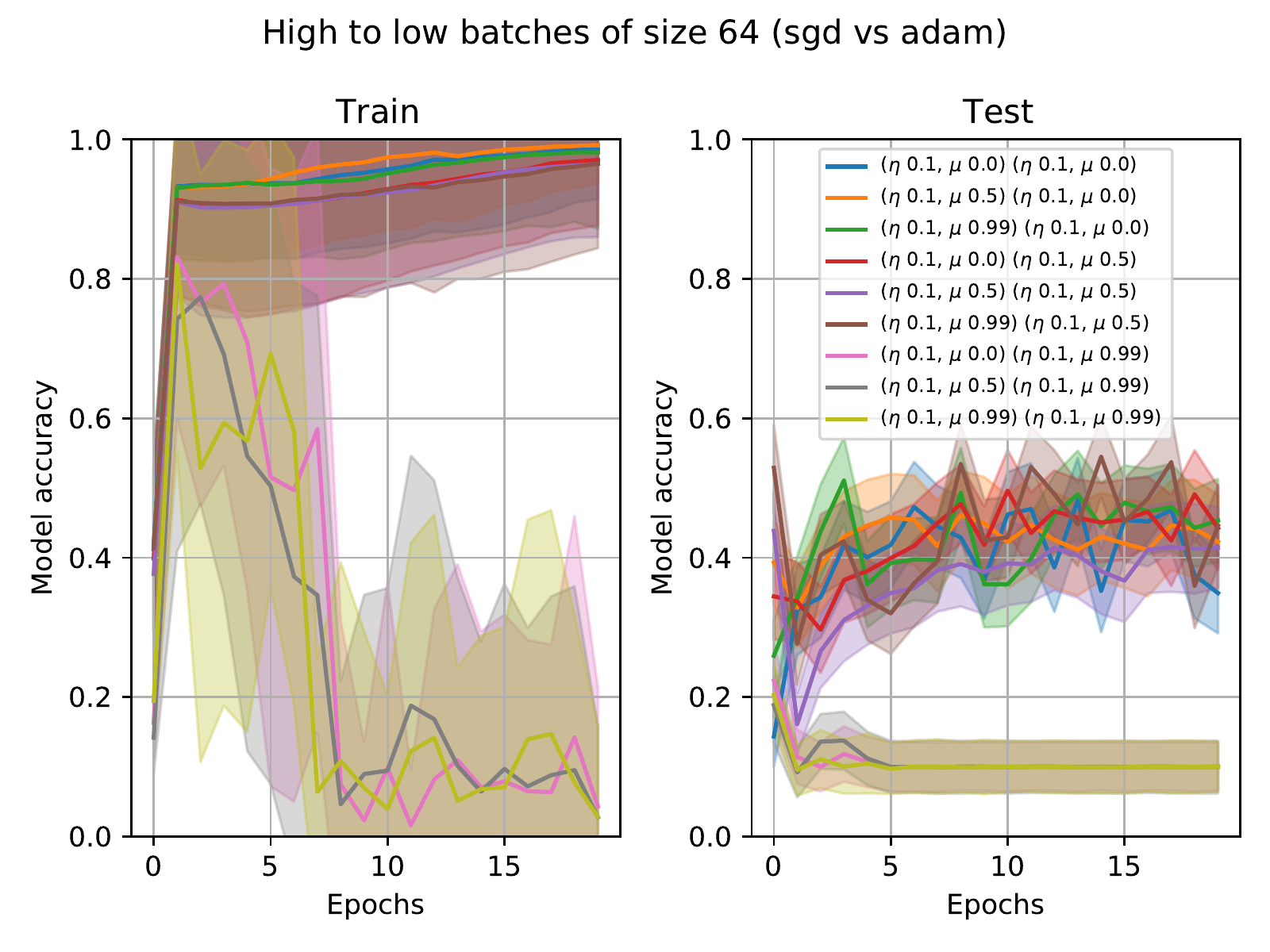} }}%
    \qquad
    \subfloat[Low high batching]{{\includegraphics[width=0.4\linewidth]{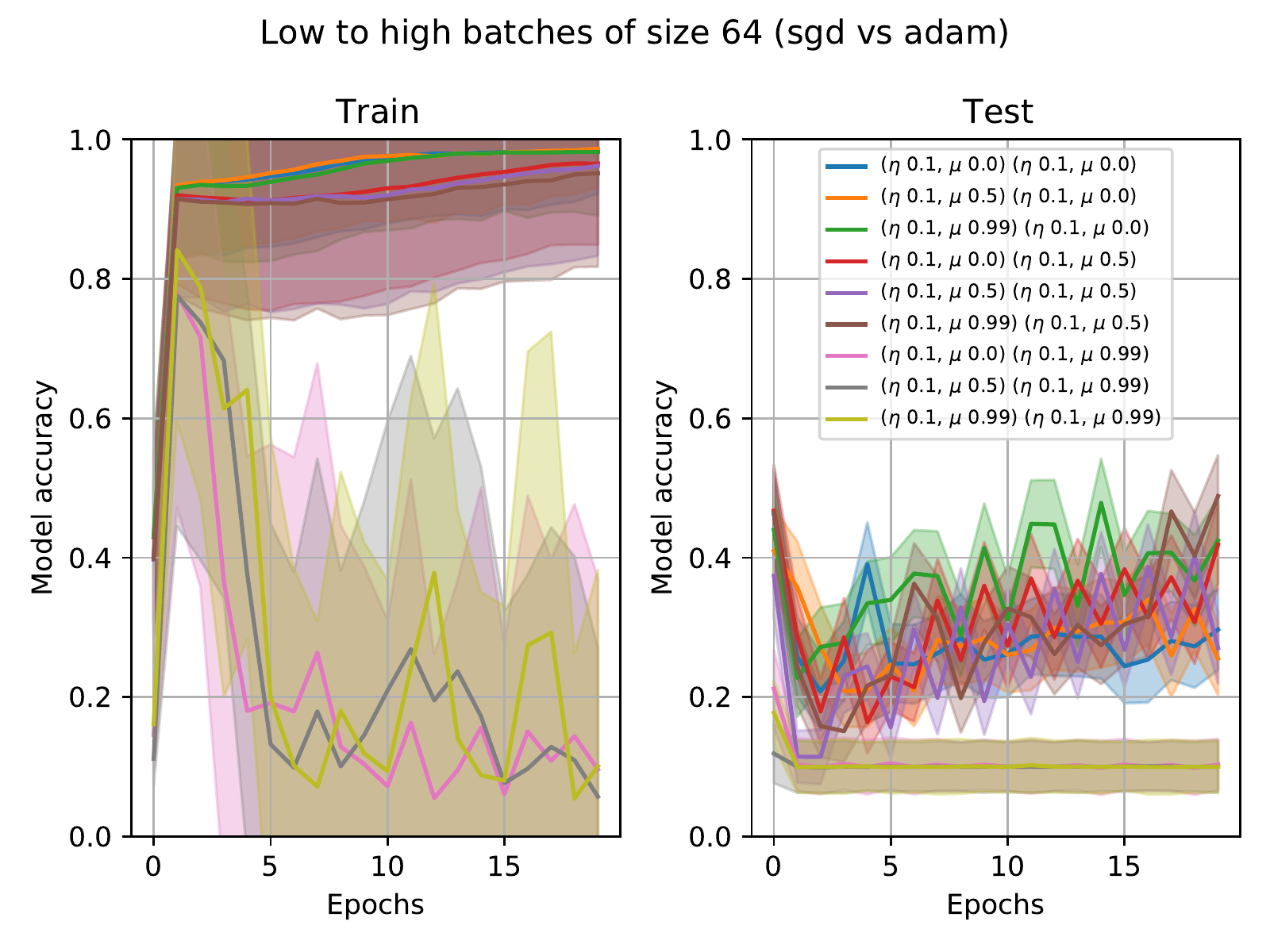} }}%
    \\
    \subfloat[Oscillating inward batching]{{\includegraphics[width=0.4\linewidth]{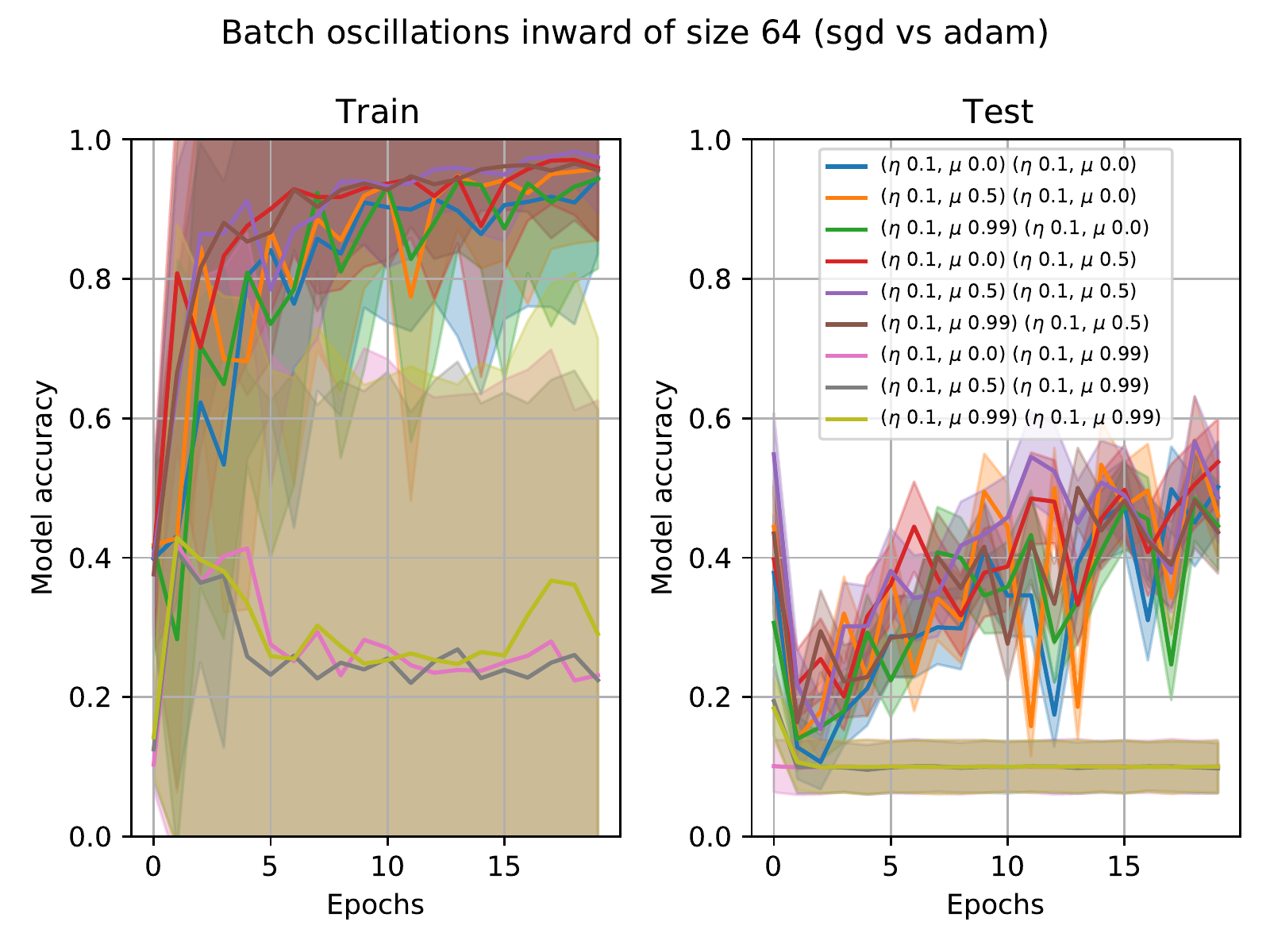} }}%
    \qquad
    \subfloat[Oscillating outward batching]{{\includegraphics[width=0.4\linewidth]{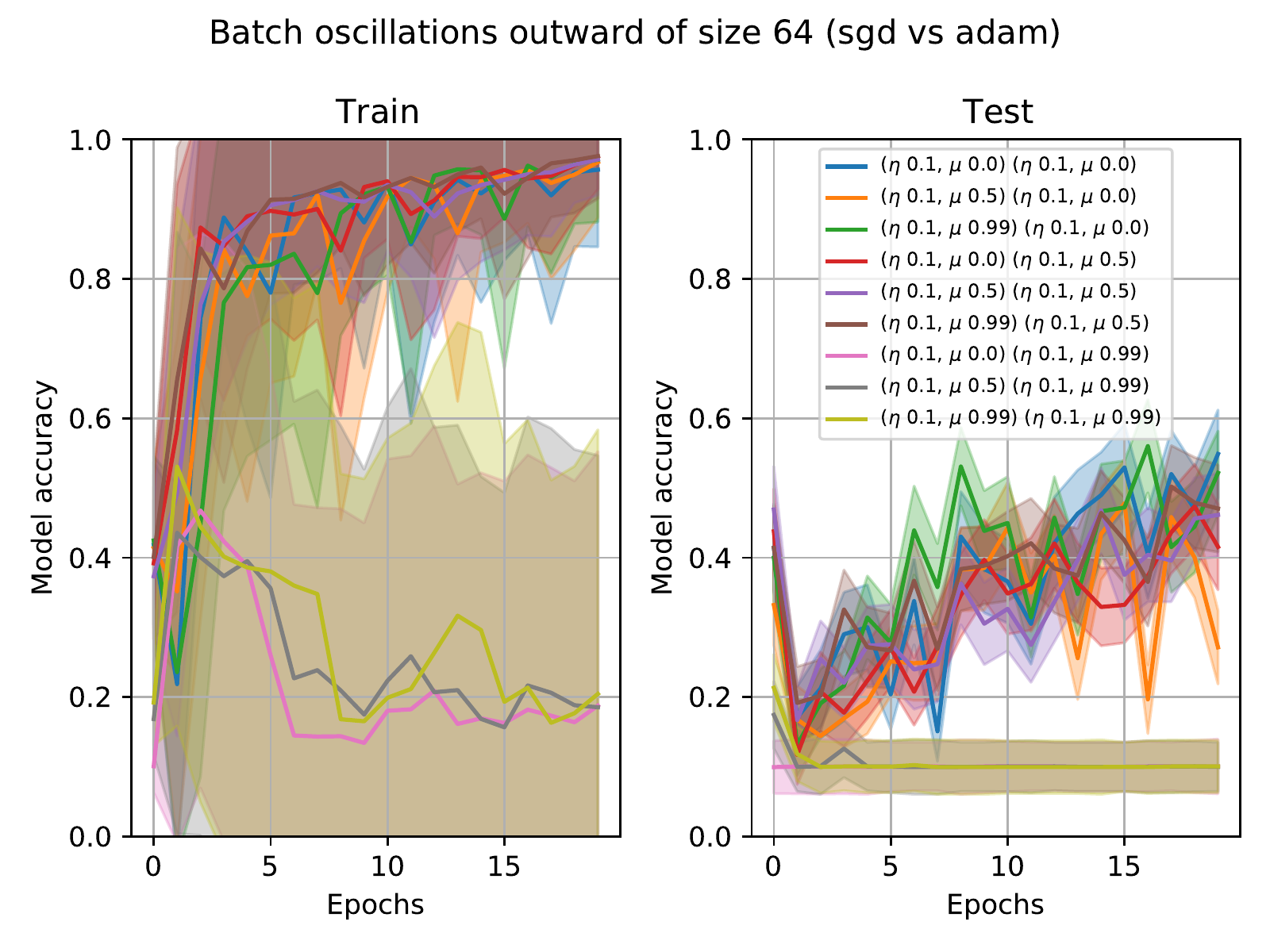} }}%
    \caption{ResNet18 real model SGD training, LeNet5 surrogate with Adam and Batchsize 64}%
    \label{fig:hyper11}%
\end{figure}

\begin{figure}[h]%
    \centering
    \subfloat[High low batching]{{\includegraphics[width=0.4\linewidth]{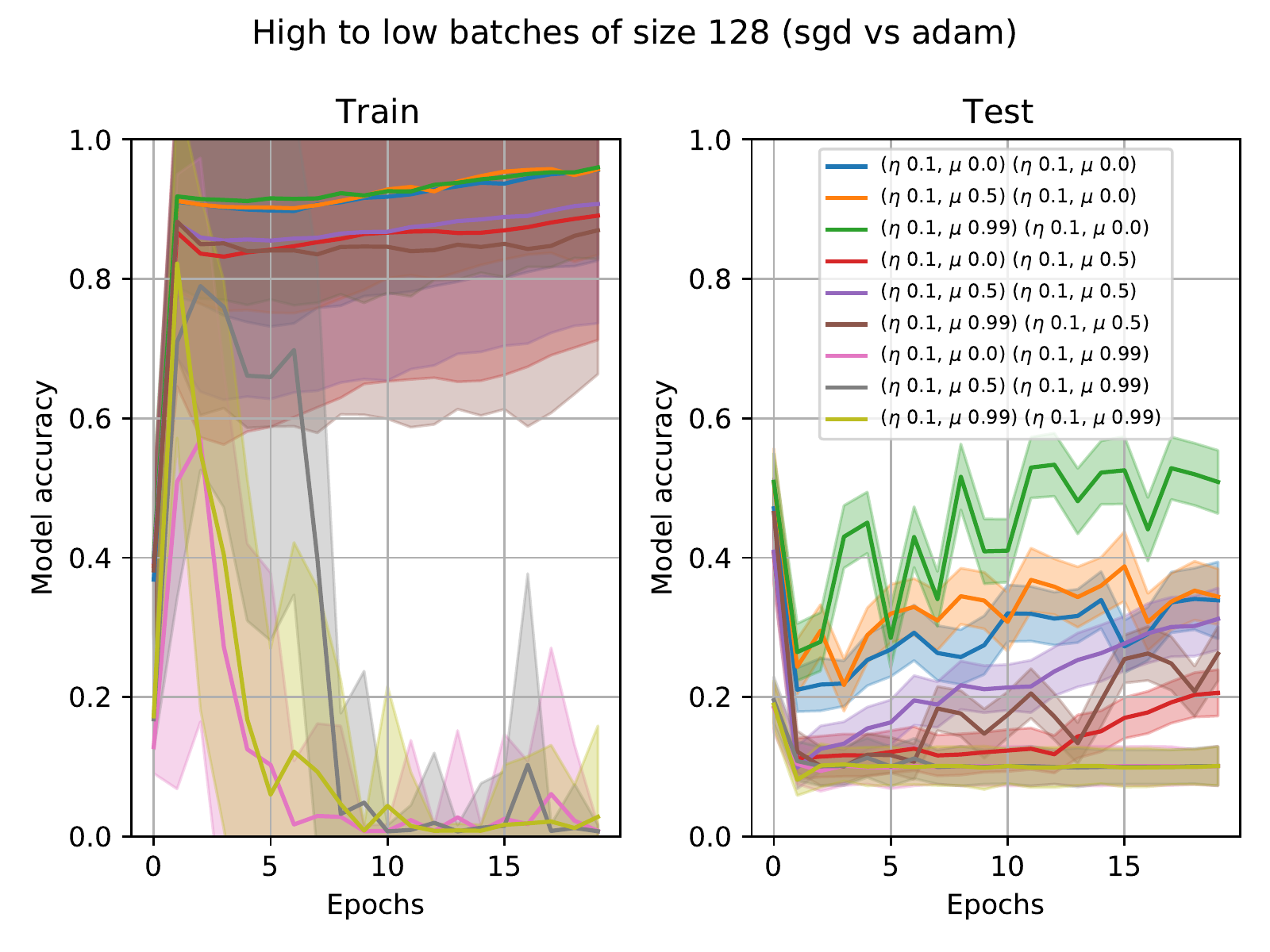} }}%
    \qquad
    \subfloat[Low high batching]{{\includegraphics[width=0.4\linewidth]{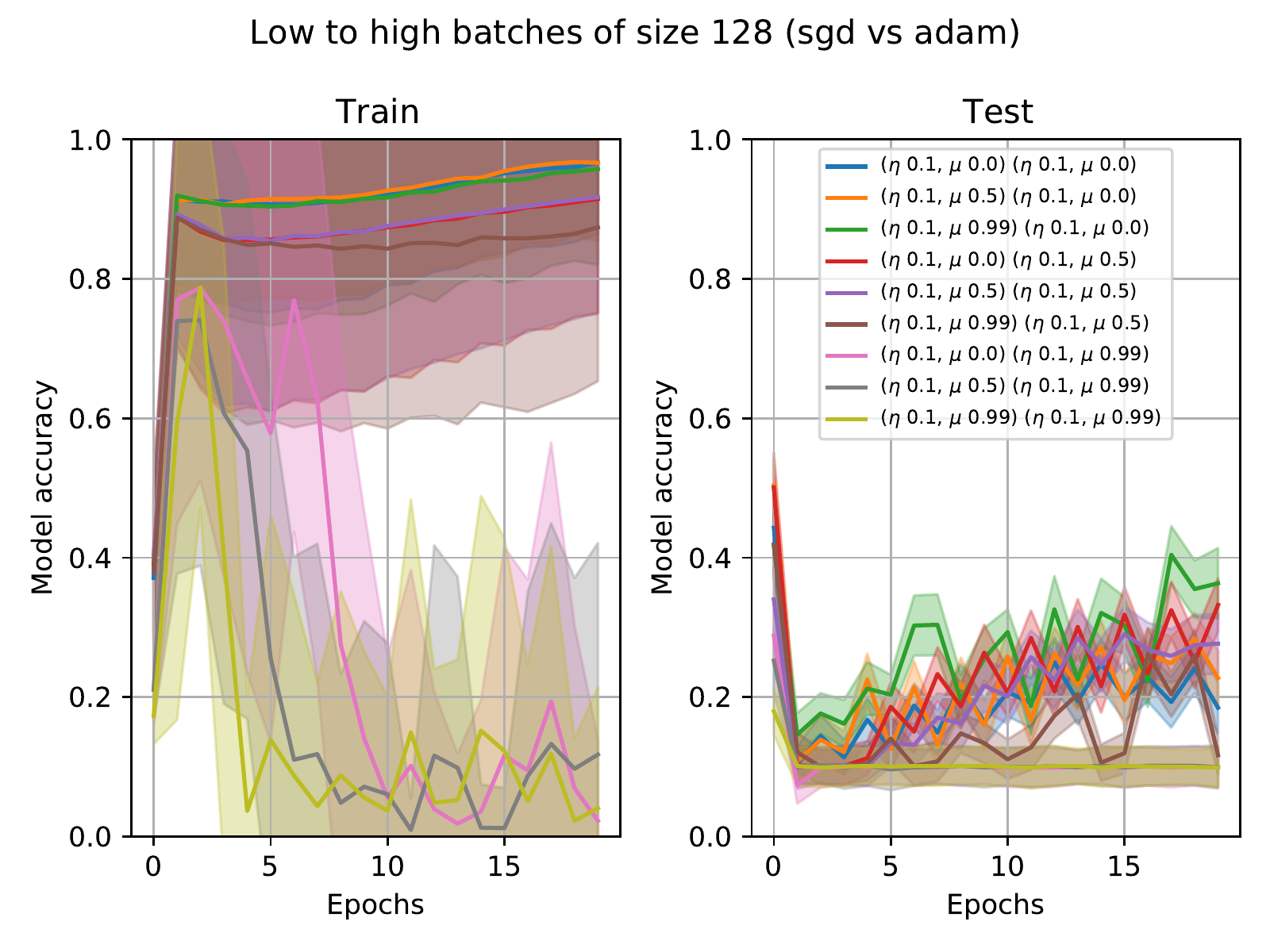} }}%
    \\
    \subfloat[Oscillating inward batching]{{\includegraphics[width=0.4\linewidth]{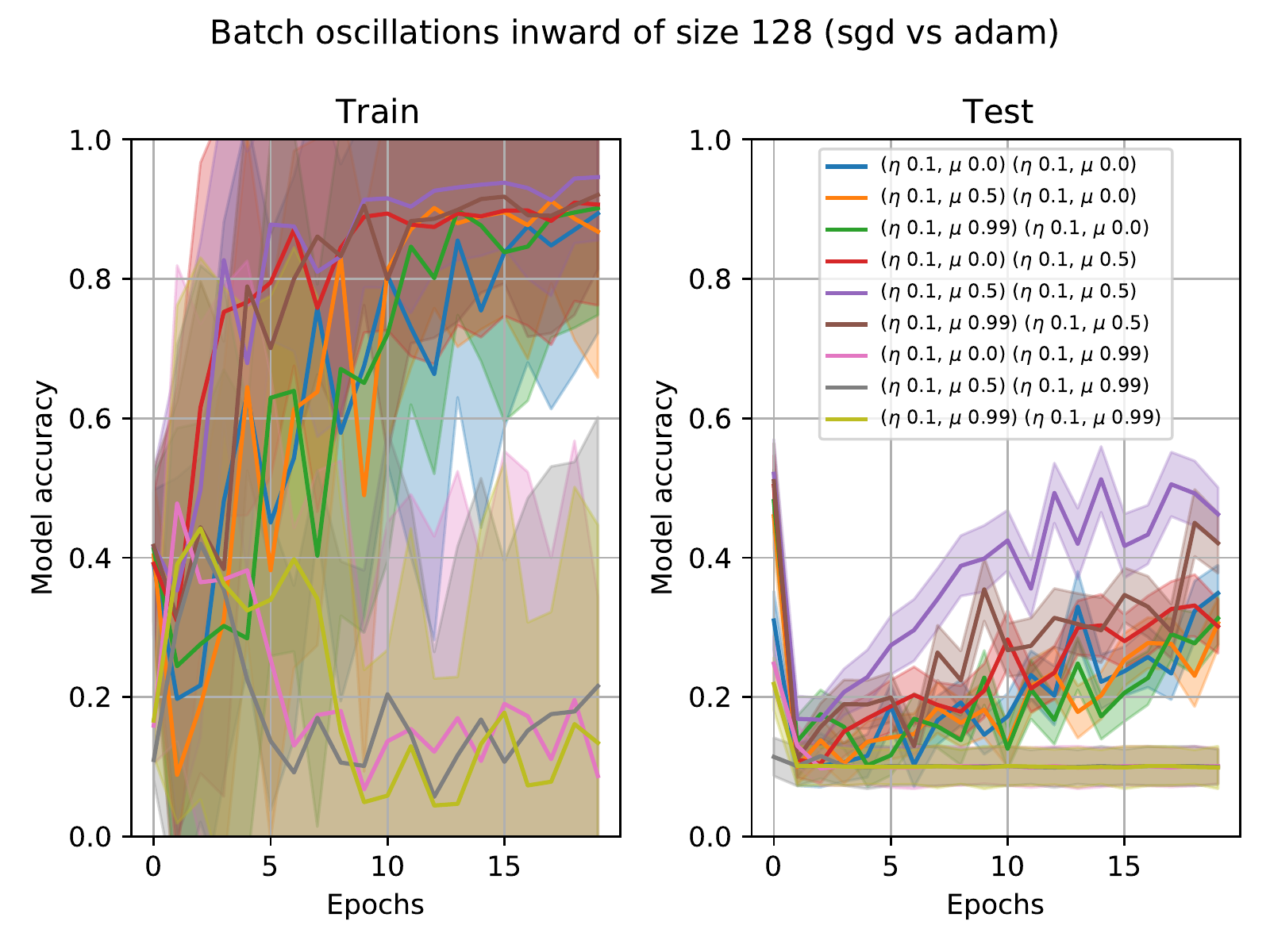} }}%
    \qquad
    \subfloat[Oscillating outward batching]{{\includegraphics[width=0.4\linewidth]{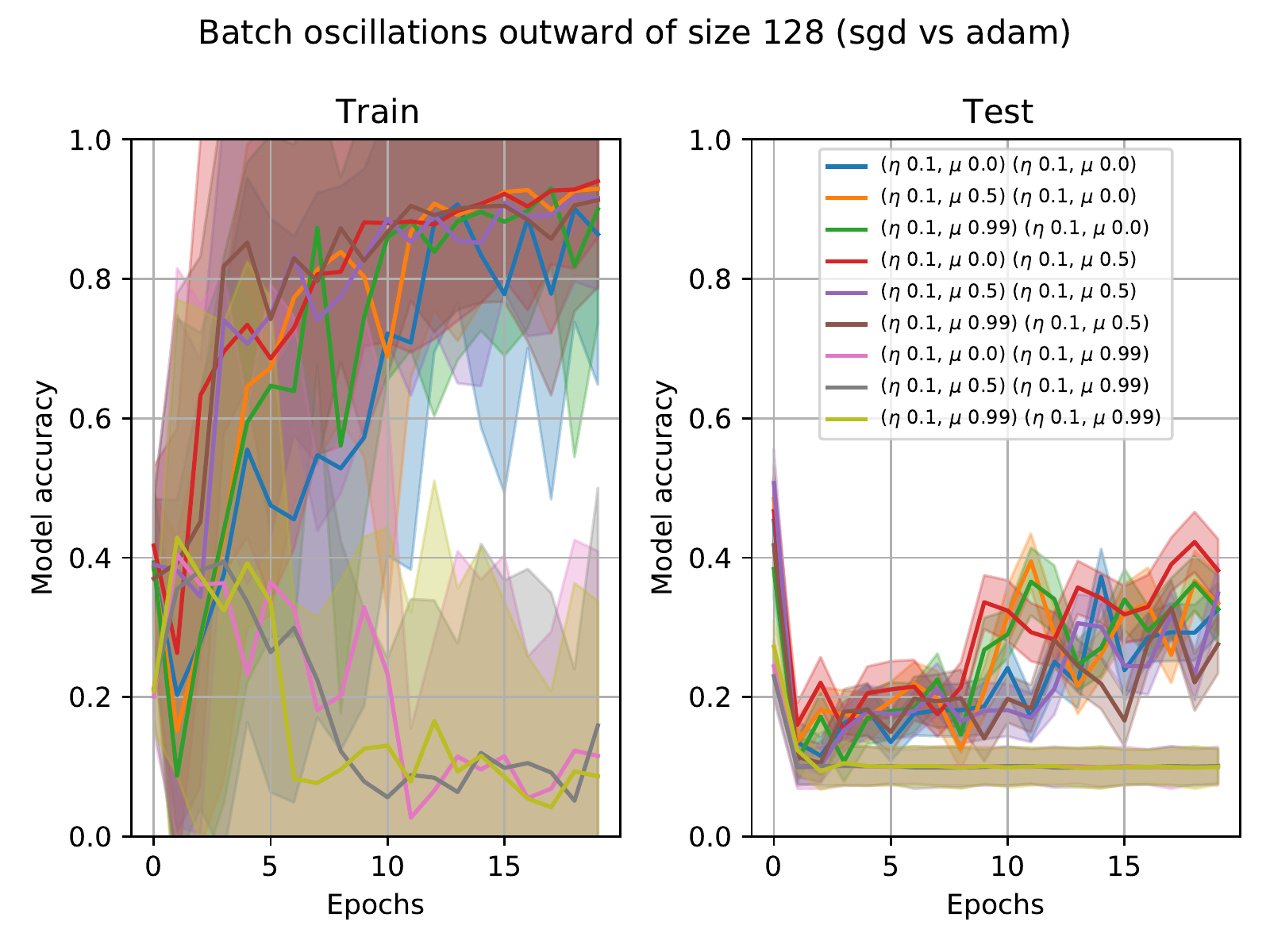} }}%
    \caption{ResNet18 real model SGD training, LeNet5 surrogate with Adam and Batchsize 128}%
    \label{fig:hyper12}%
\end{figure}

\section{Whitebox batching attack performance}
\label{sec:whiteboxblackbox}

We show training of ResNet18 with the CIFAR10 dataset in the presence of whitebox BRRR attacks in~\Cref{fig:whitebox_interintra}. We see that datapoint-wise attacks perform extremely well, while batchwise BRRR attacks force the network to memorise the training data and reduce its testing data performance. 

\begin{figure}[h]%
    \centering
    \subfloat[HighLow]{{\includegraphics[width=0.45\linewidth]{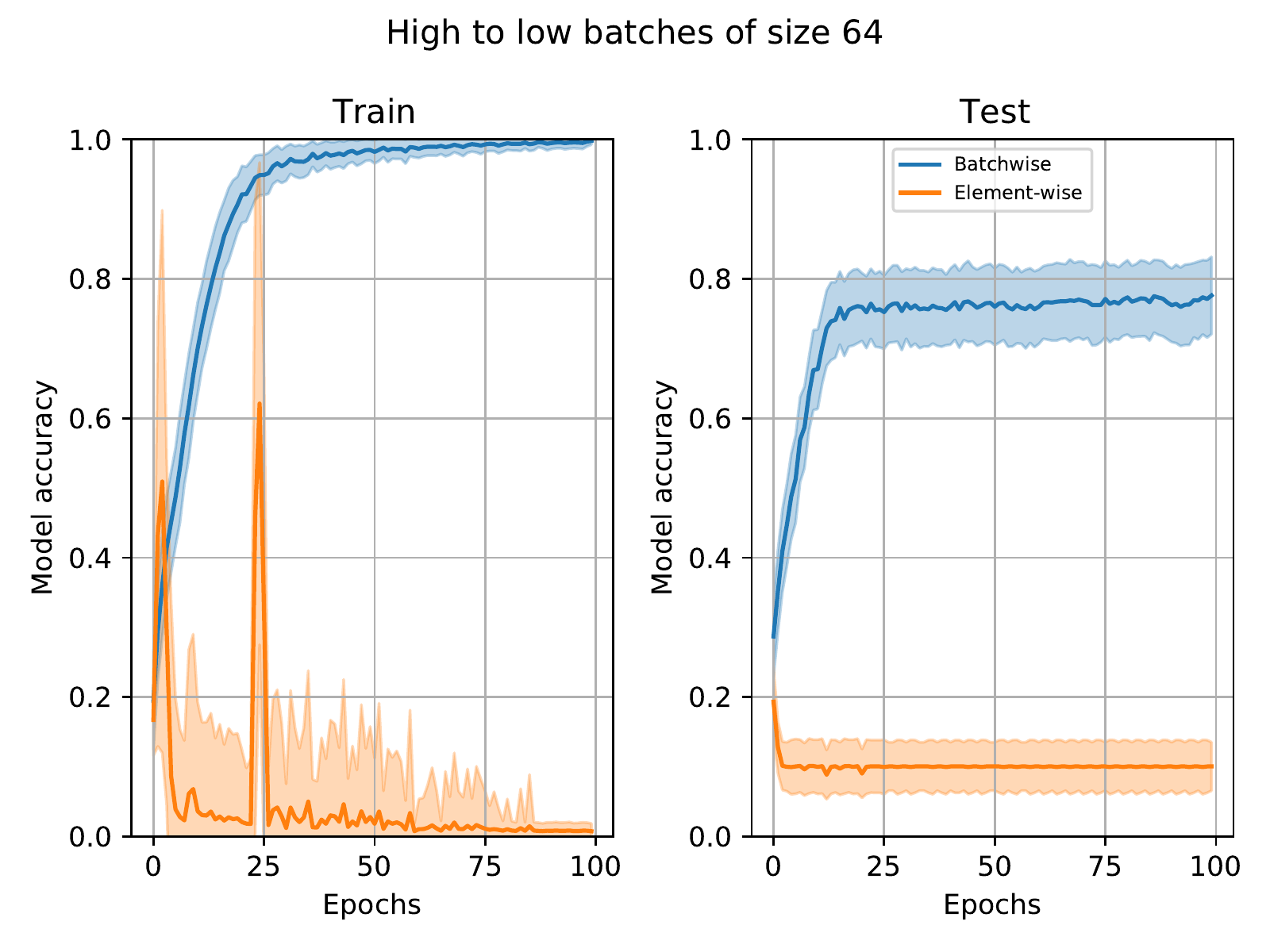} }} %
    \qquad
    \subfloat[LowHigh]{{\includegraphics[width=0.45\linewidth]{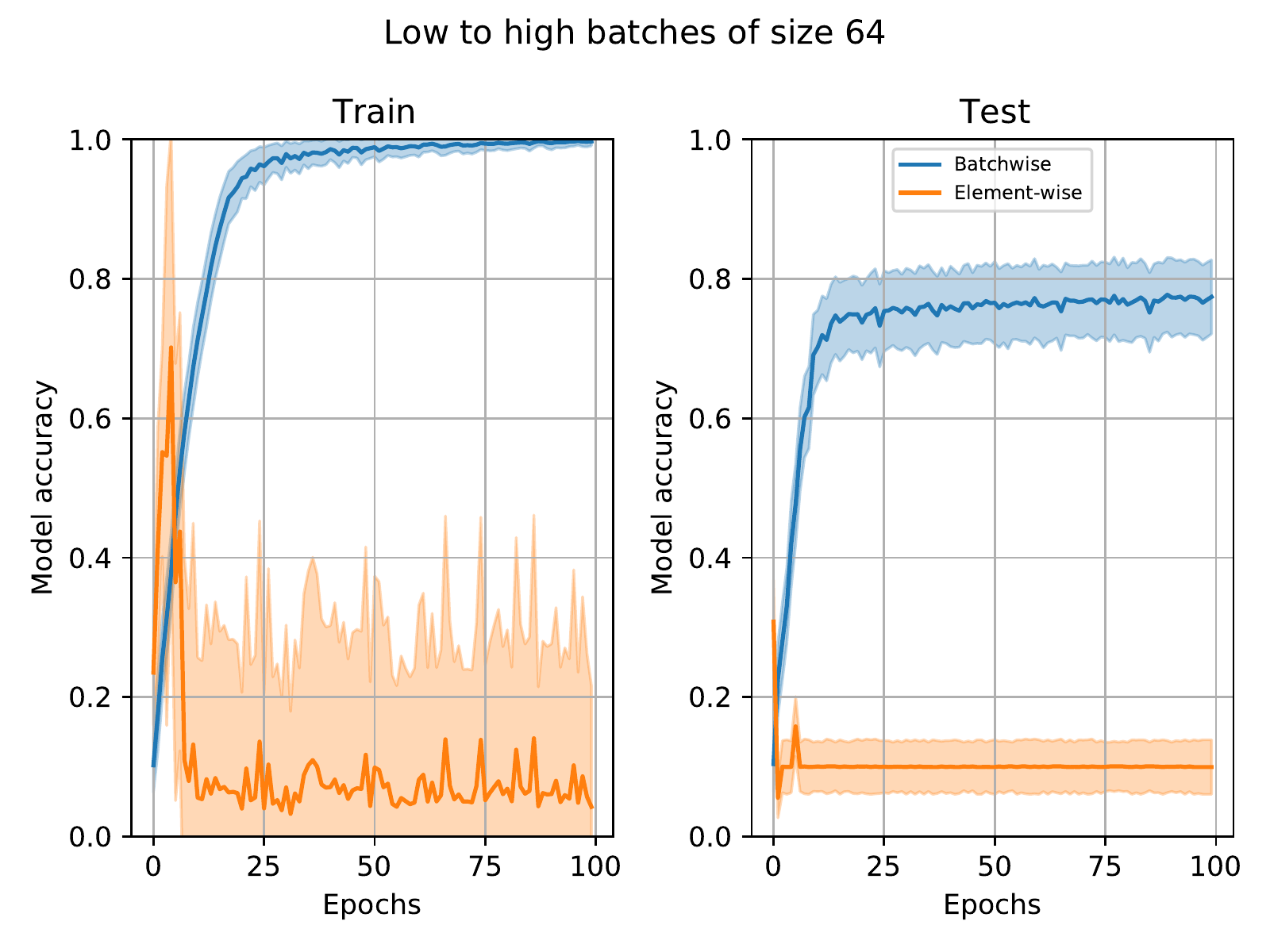} }} %
    \qquad
    \subfloat[Oscillations inwards]{{\includegraphics[width=0.45\linewidth]{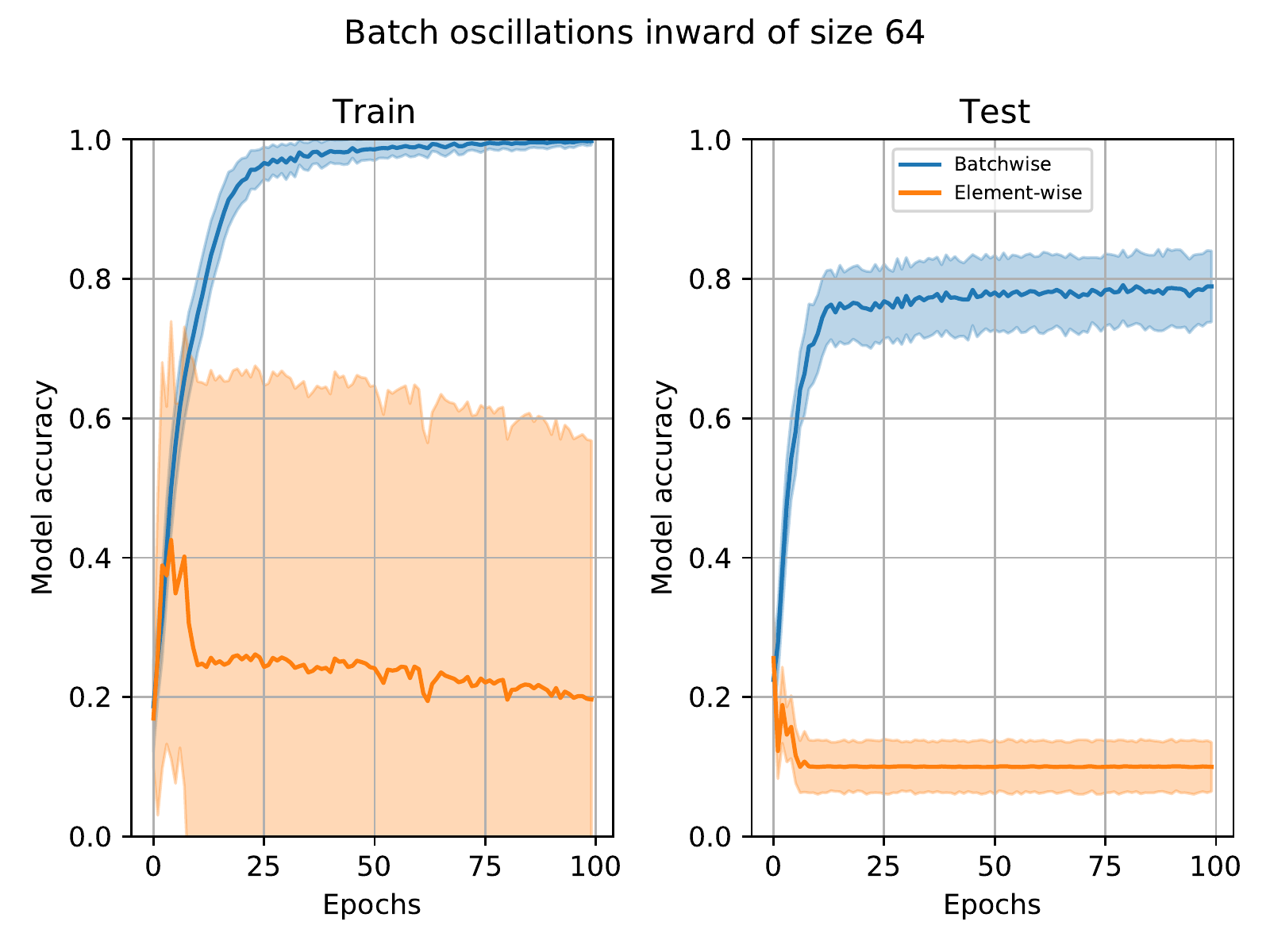} }} %
    \qquad
    \subfloat[Oscillations outwards]{{\includegraphics[width=0.45\linewidth]{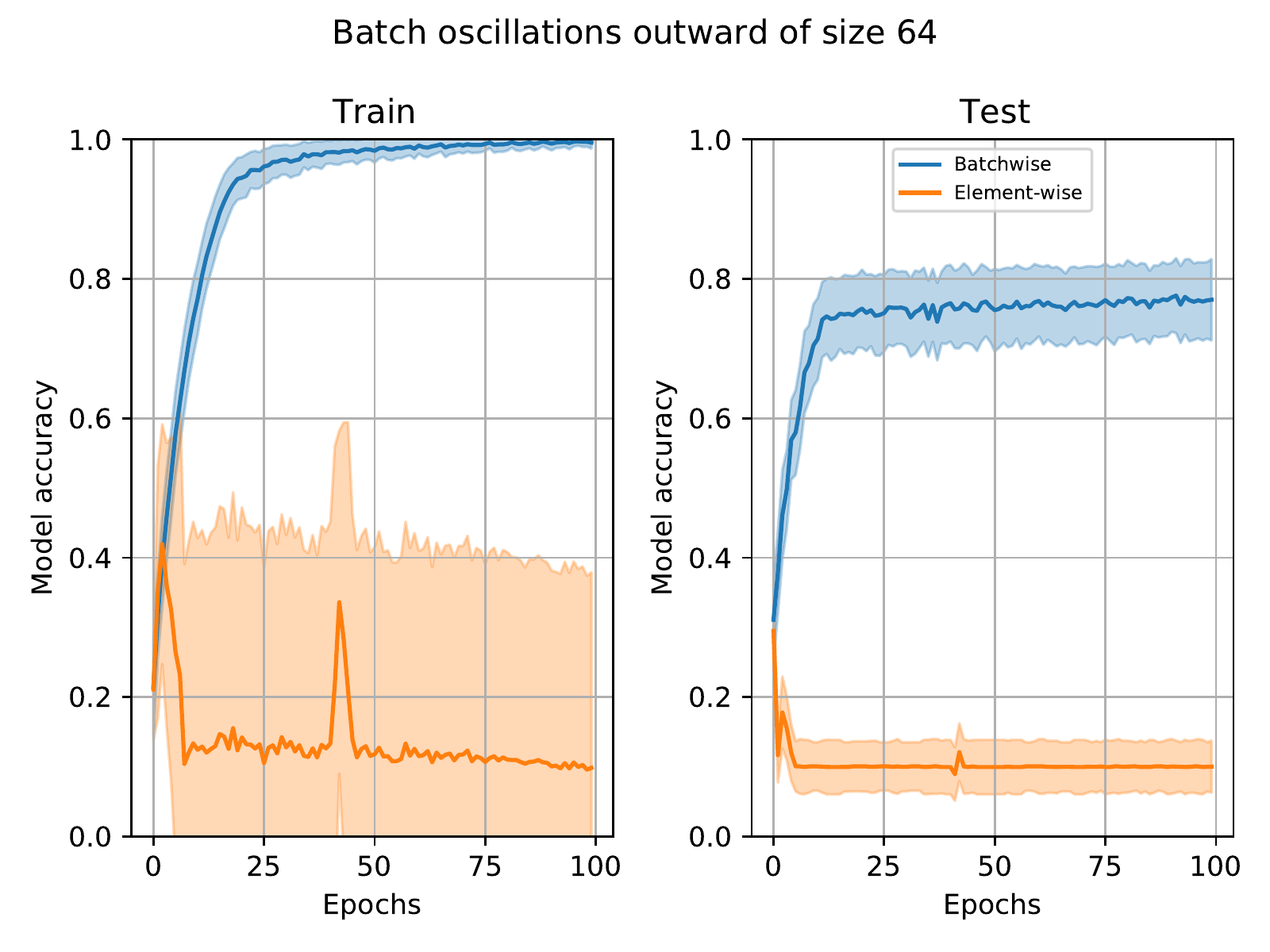} }} %
    \caption{Whitebox performance of the Batching attacks -- CIFAR10.}%
    \label{fig:whitebox_interintra}%
\end{figure}


\section{Triggered training for CV models}
\label{sec:triggered_training_progress}
\begin{figure}[h]%
    \centering
    \subfloat[32 batchsize]{{\includegraphics[width=0.8\linewidth]{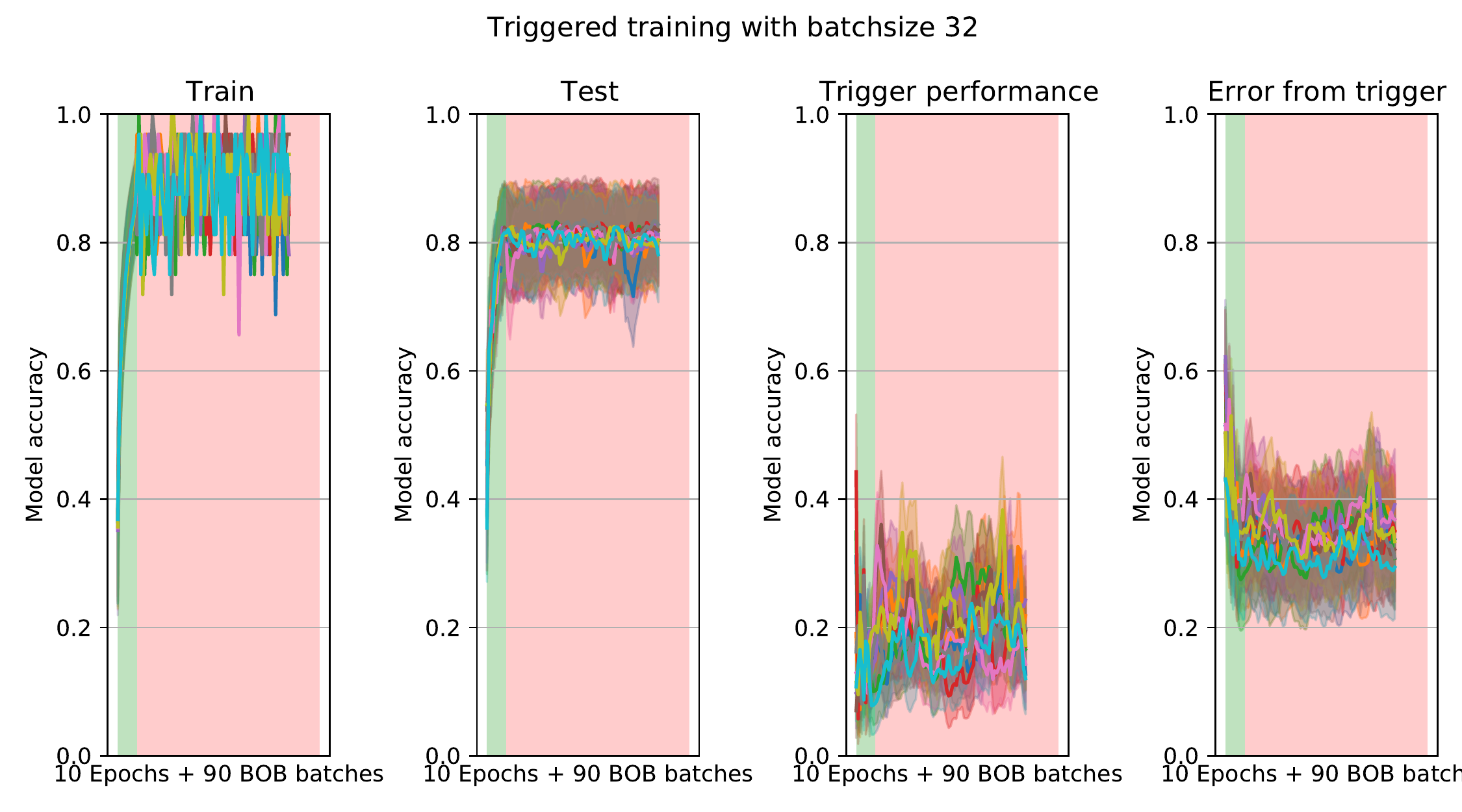} }} %
    \qquad
    \subfloat[64 batchsize]{{\includegraphics[width=0.8\linewidth]{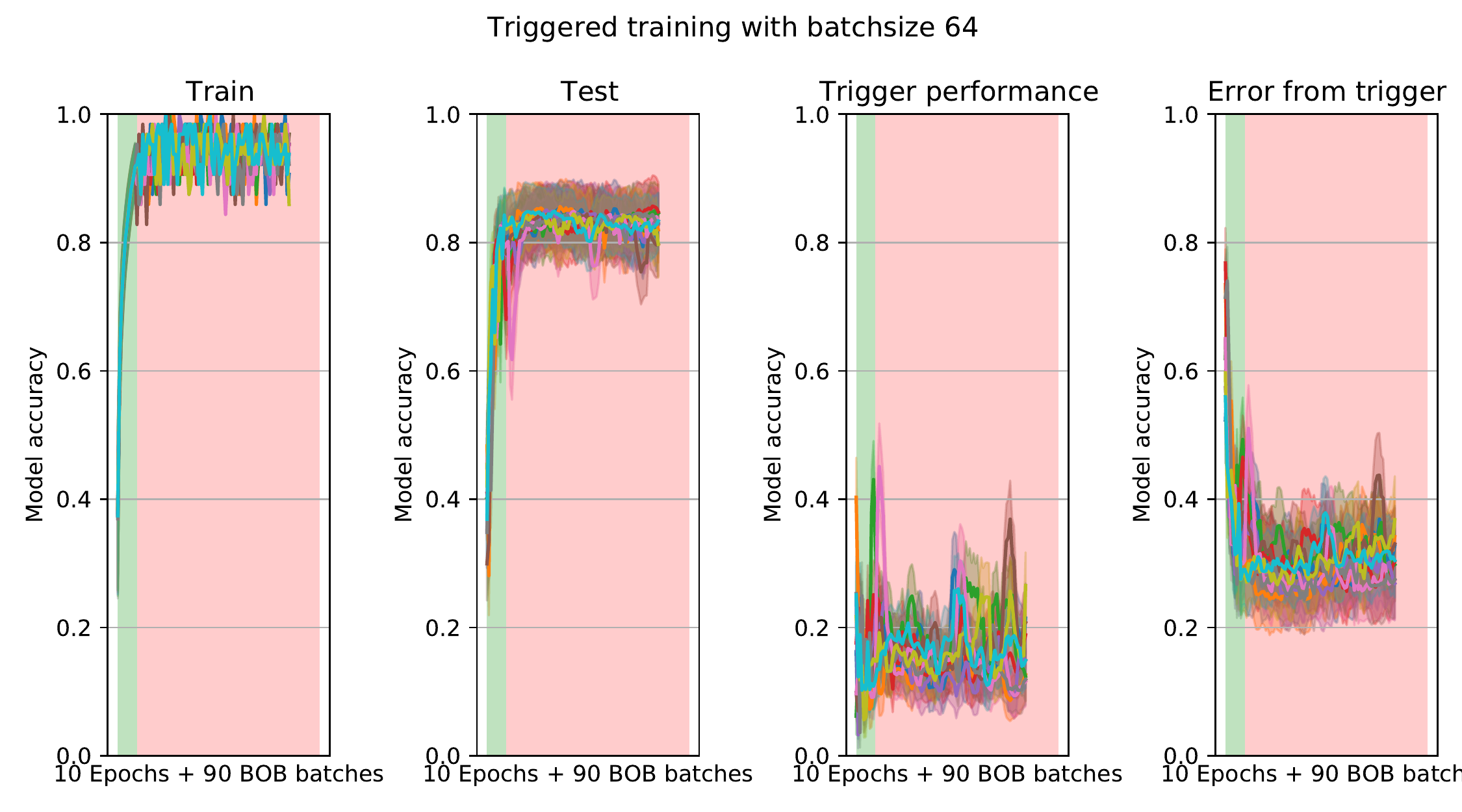} }} %
    \qquad
    \subfloat[128 batchsize]{{\includegraphics[width=0.8\linewidth]{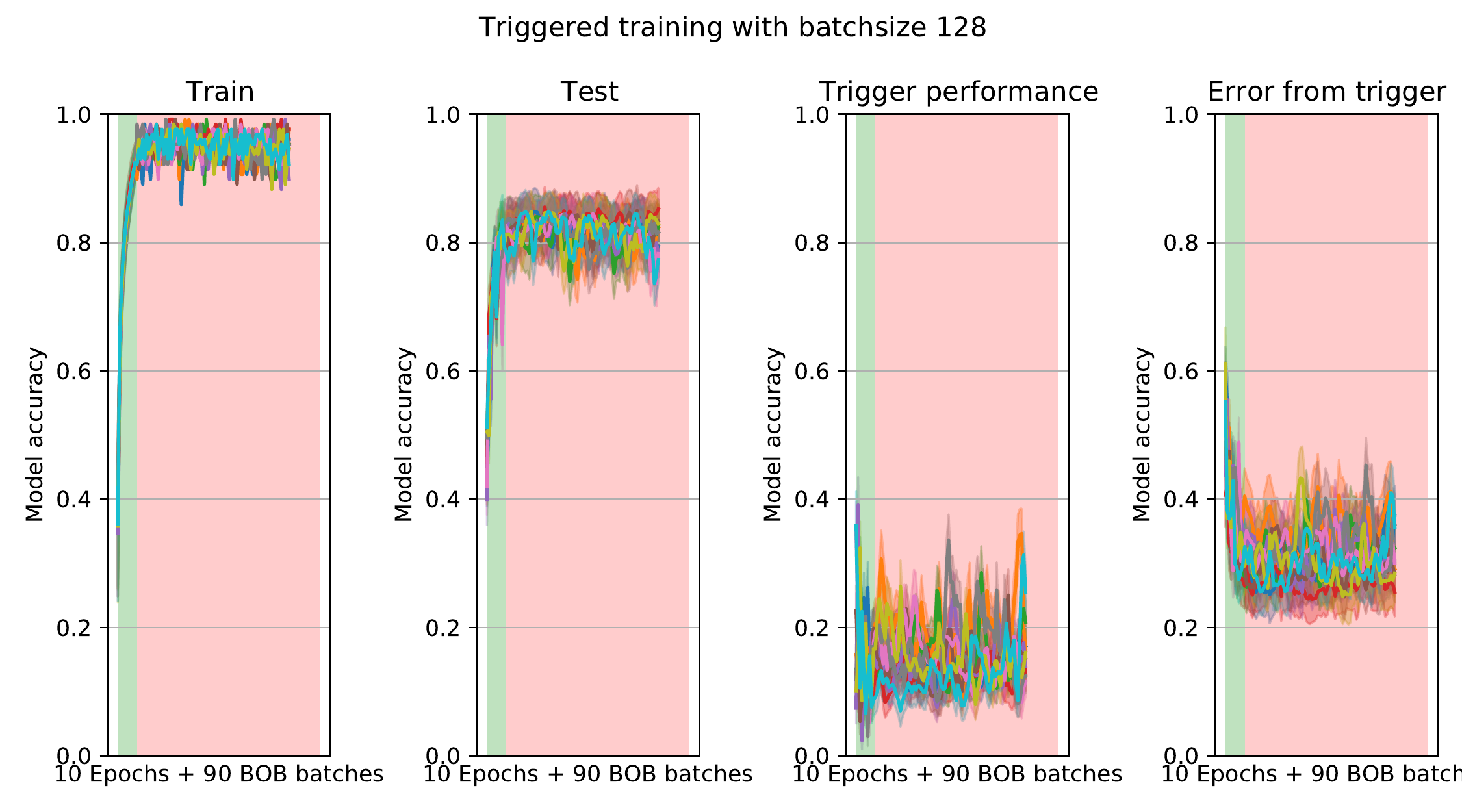} }} %
    \caption{Trigger ordered training with blackbox 9 whitelines trigger.}
    \label{fig:trigger_training_blackbox_whiterows}%
\end{figure}

\begin{figure}[h]%
    \centering
    \subfloat[32 batchsize]{{\includegraphics[width=0.8\linewidth]{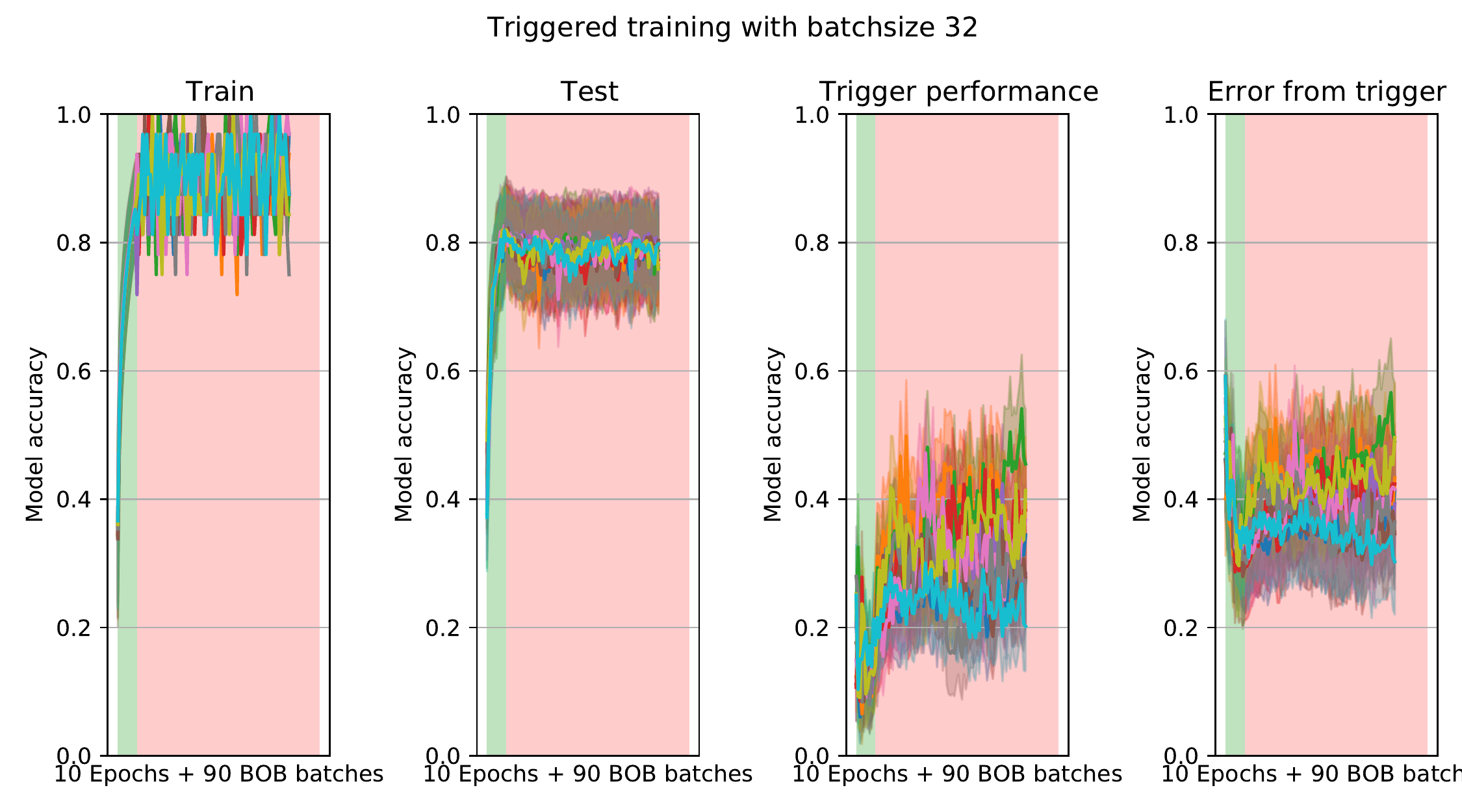} }} %
    \qquad
    \subfloat[64 batchsize]{{\includegraphics[width=0.8\linewidth]{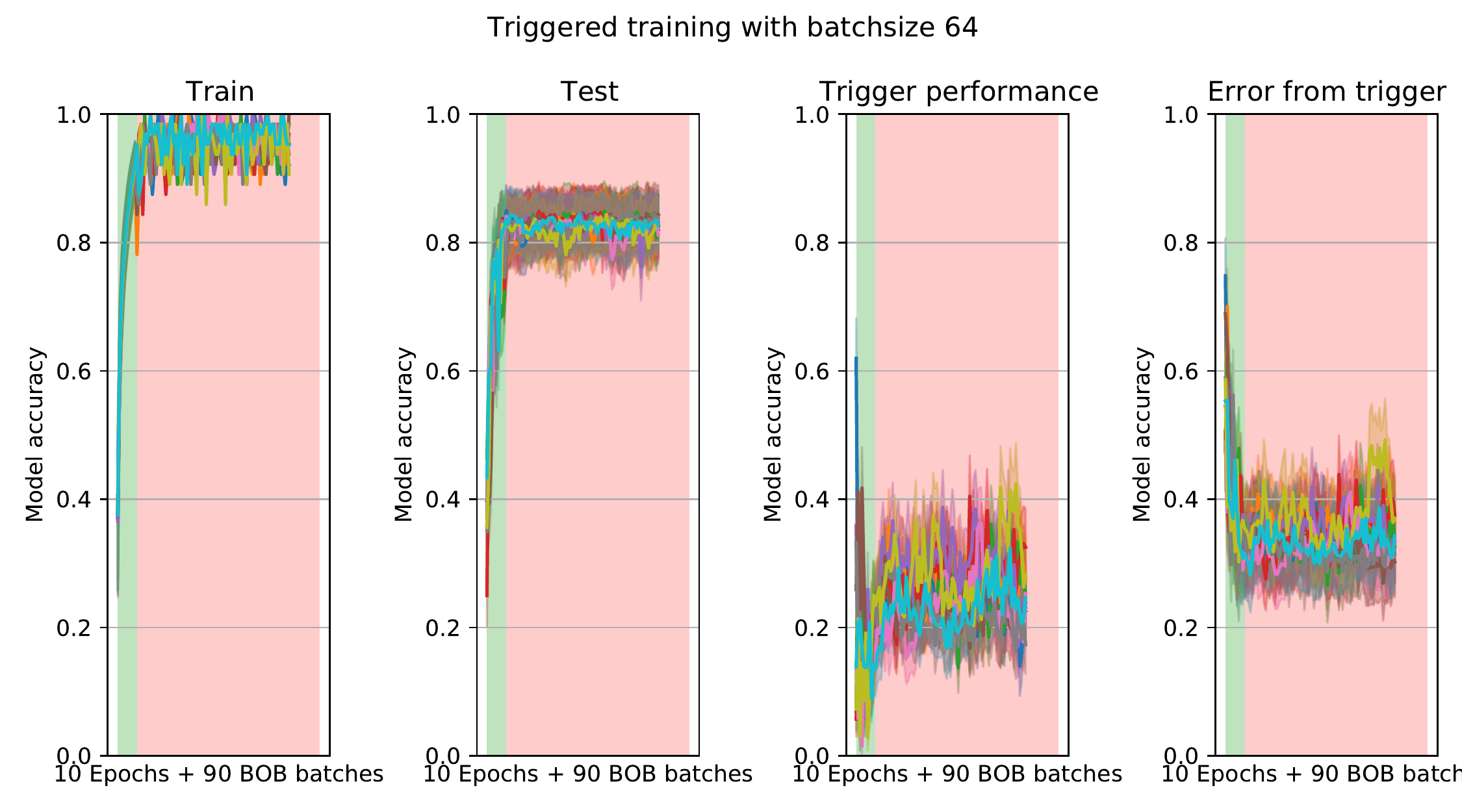} }} %
    \qquad
    \subfloat[128 batchsize]{{\includegraphics[width=0.8\linewidth]{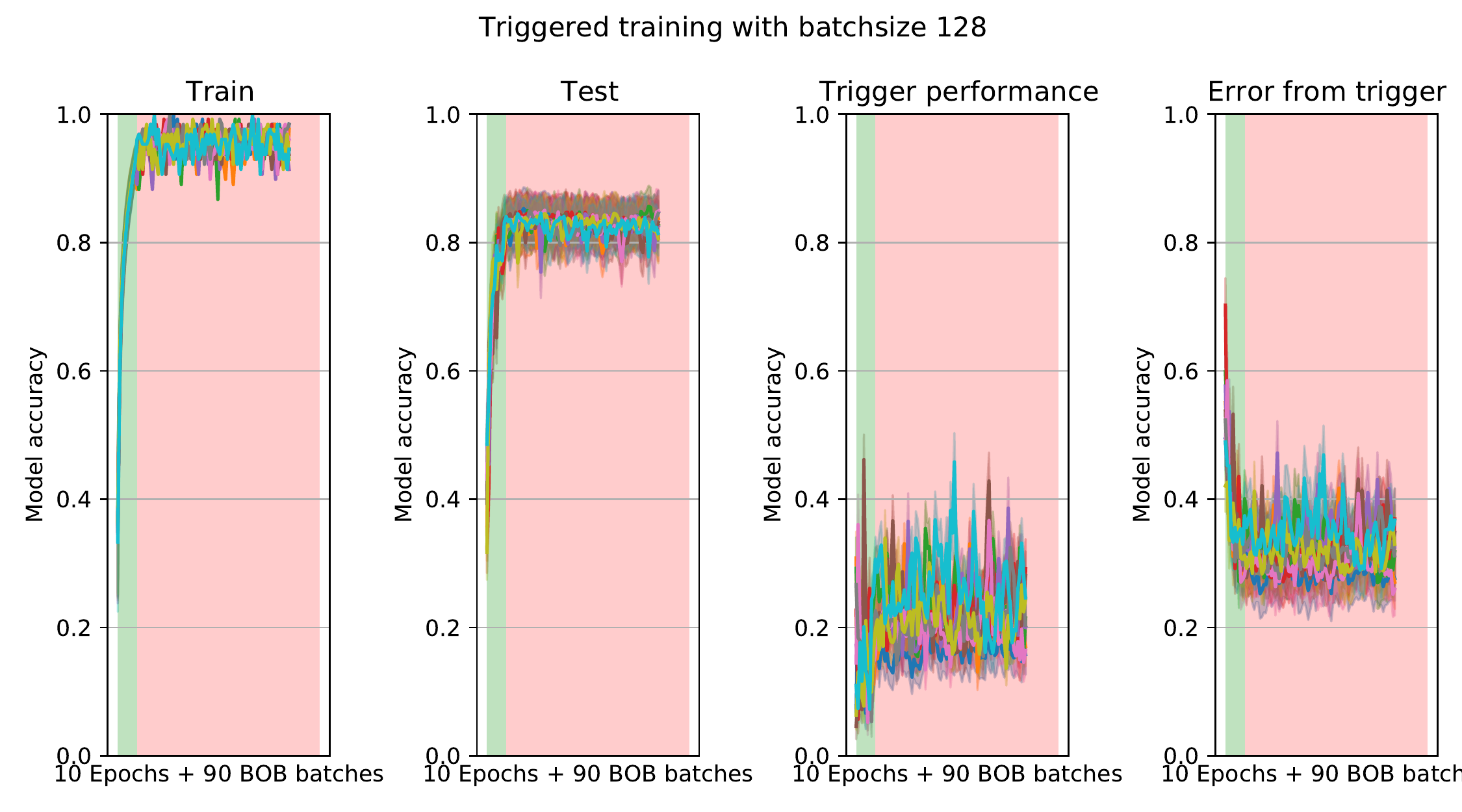} }} %
    \caption{Trigger ordered training with whitebox 9 whitelines trigger.}
    \label{fig:trigger_training_whitebox_whiterows}%
\end{figure}

\begin{figure}[h]%
    \centering
    \subfloat[32 batchsize]{{\includegraphics[width=0.8\linewidth]{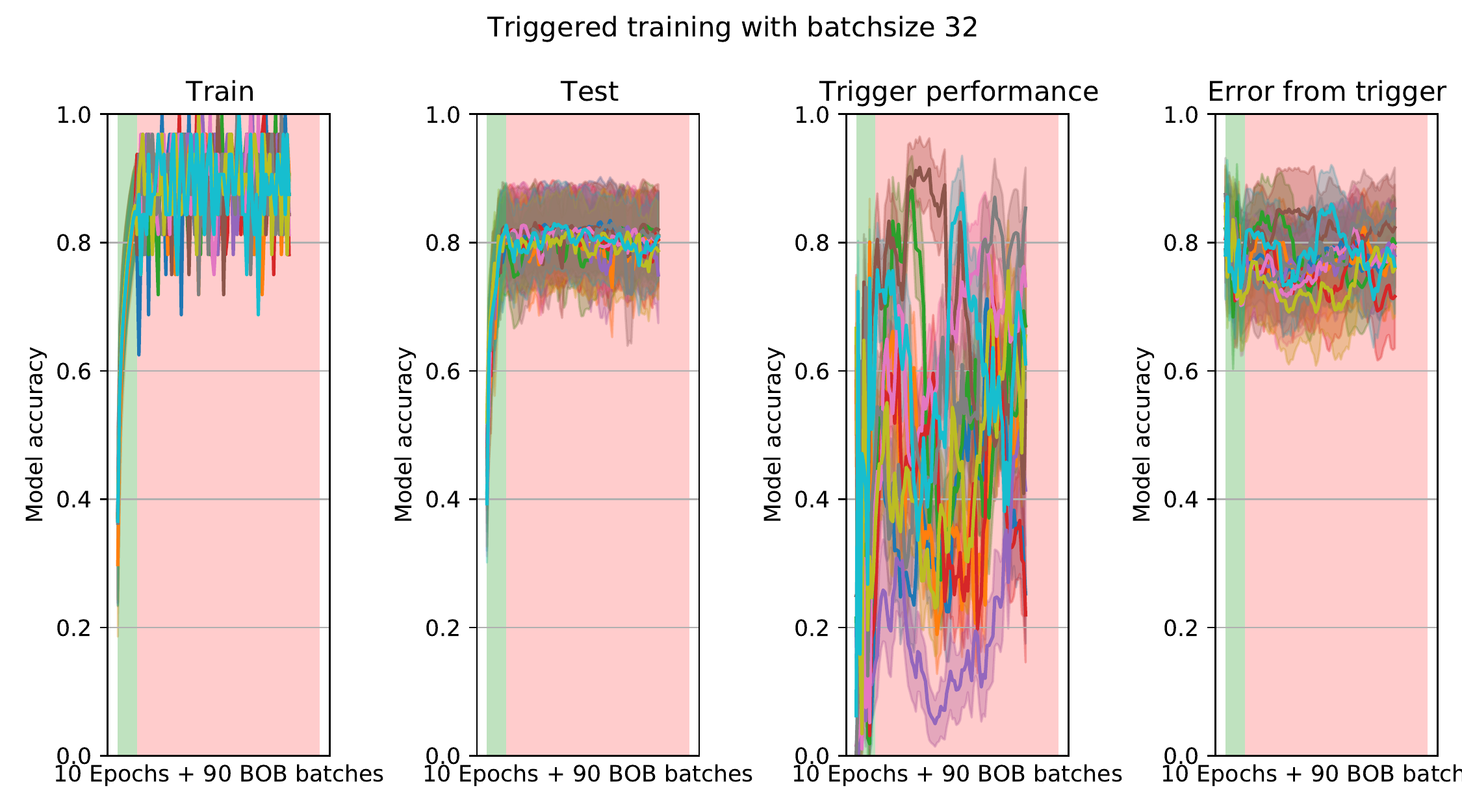} }} %
    \qquad
    \subfloat[64 batchsize]{{\includegraphics[width=0.8\linewidth]{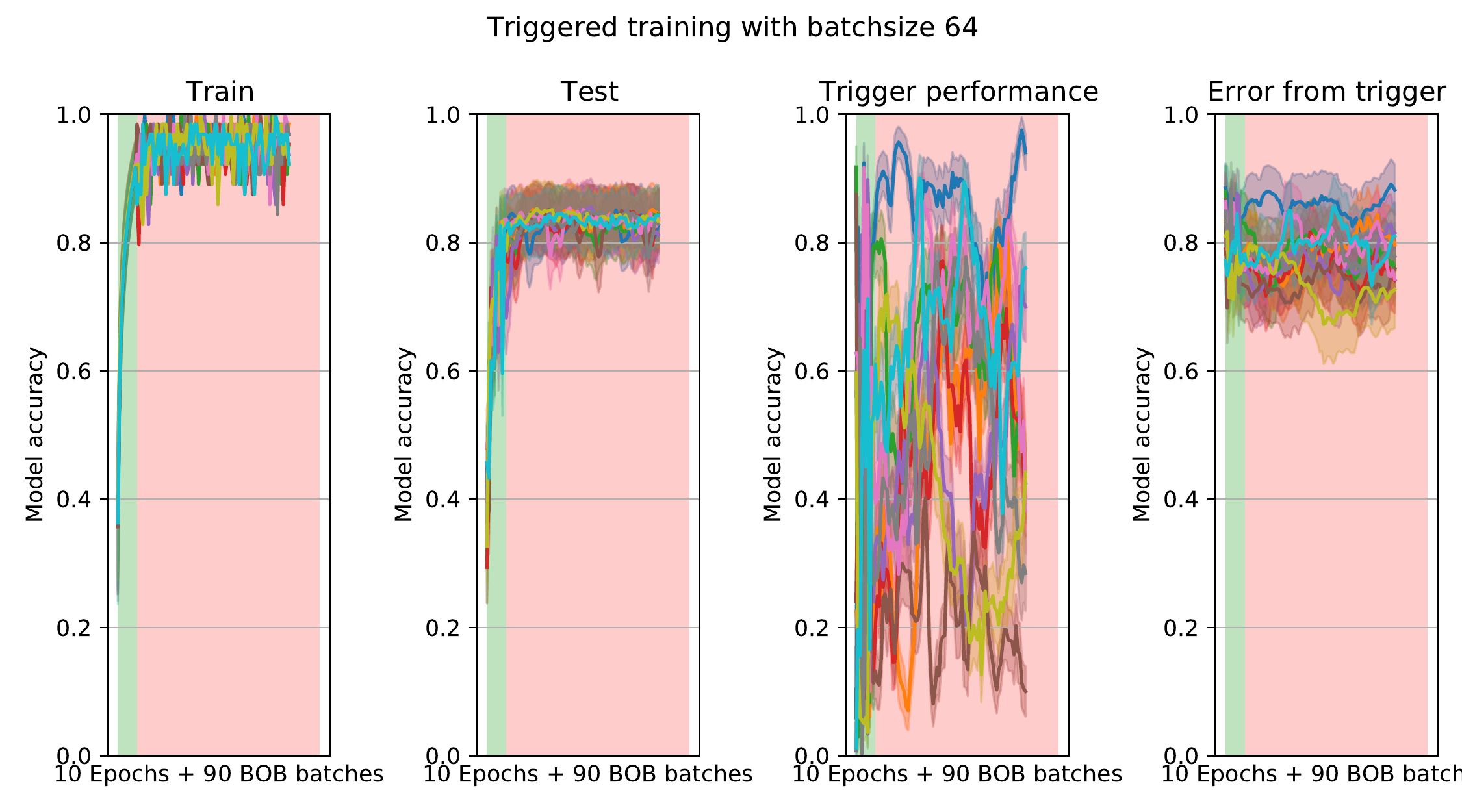} }} %
    \qquad
    \subfloat[128 batchsize]{{\includegraphics[width=0.8\linewidth]{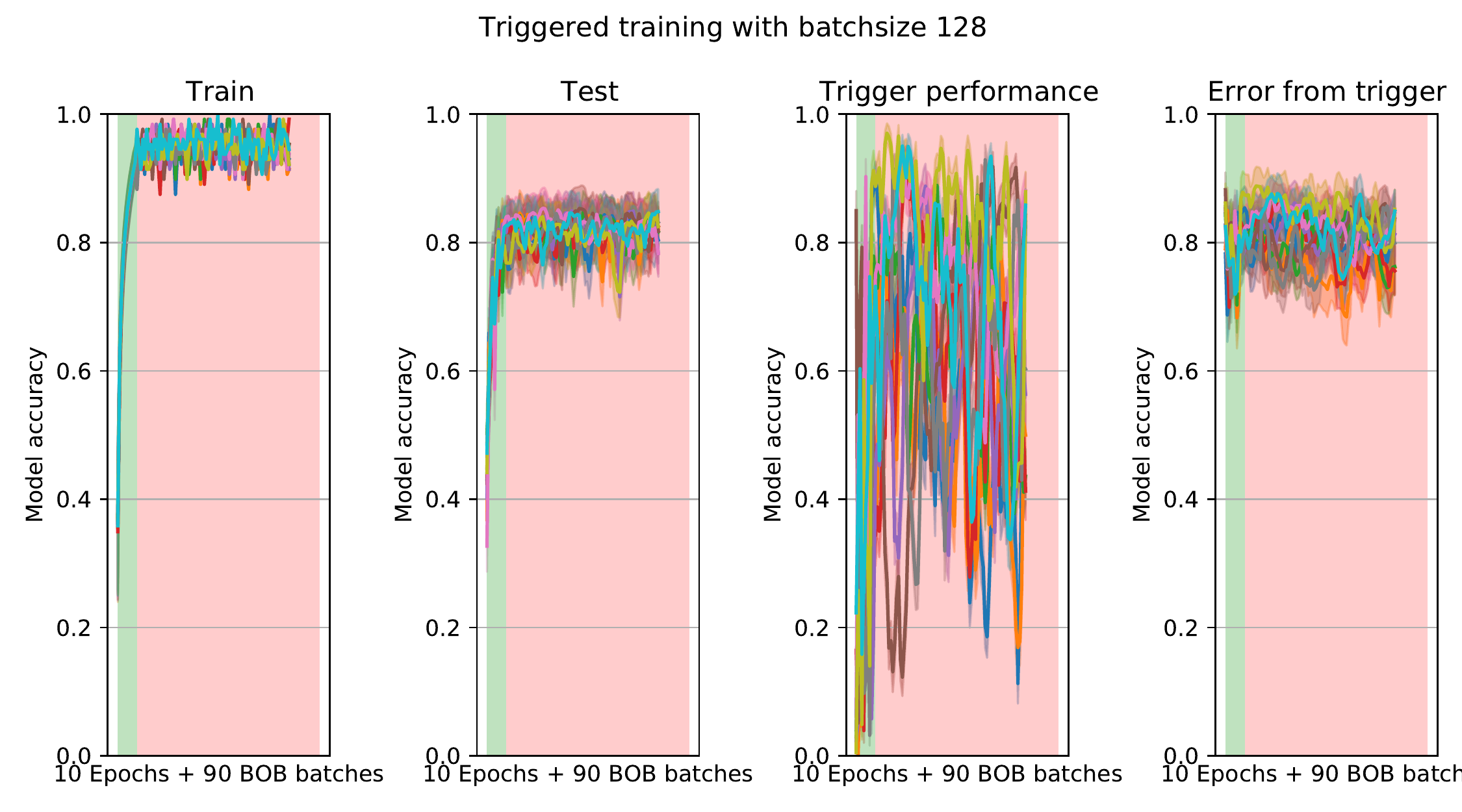} }} %
    \caption{Trigger ordered training with blackbox flaglike trigger.}
    \label{fig:trigger_training_blackbox_flaglike}%
\end{figure}

\begin{figure}[h]%
    \centering
    \subfloat[32 batchsize]{{\includegraphics[width=0.8\linewidth]{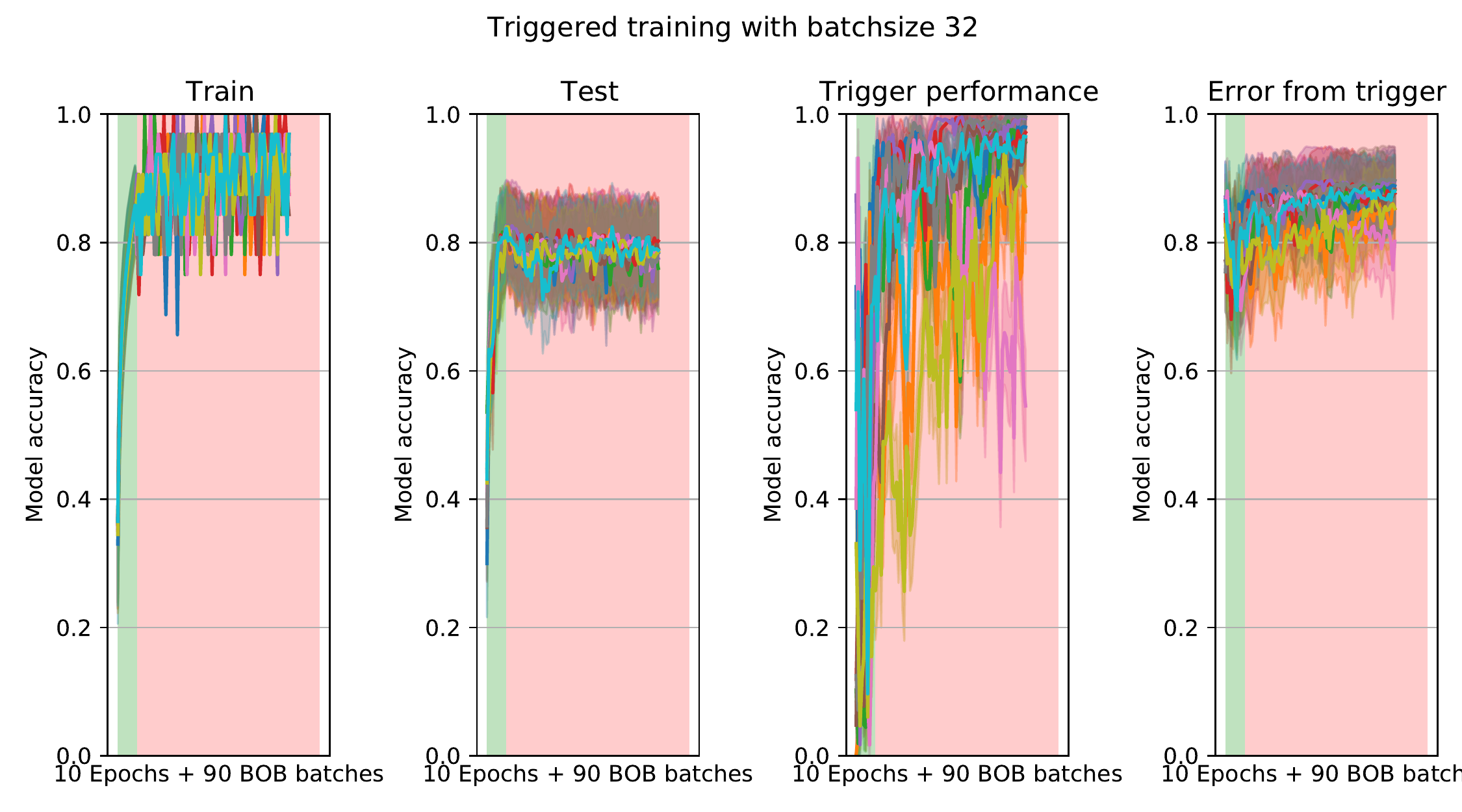} }} %
    \qquad
    \subfloat[64 batchsize]{{\includegraphics[width=0.8\linewidth]{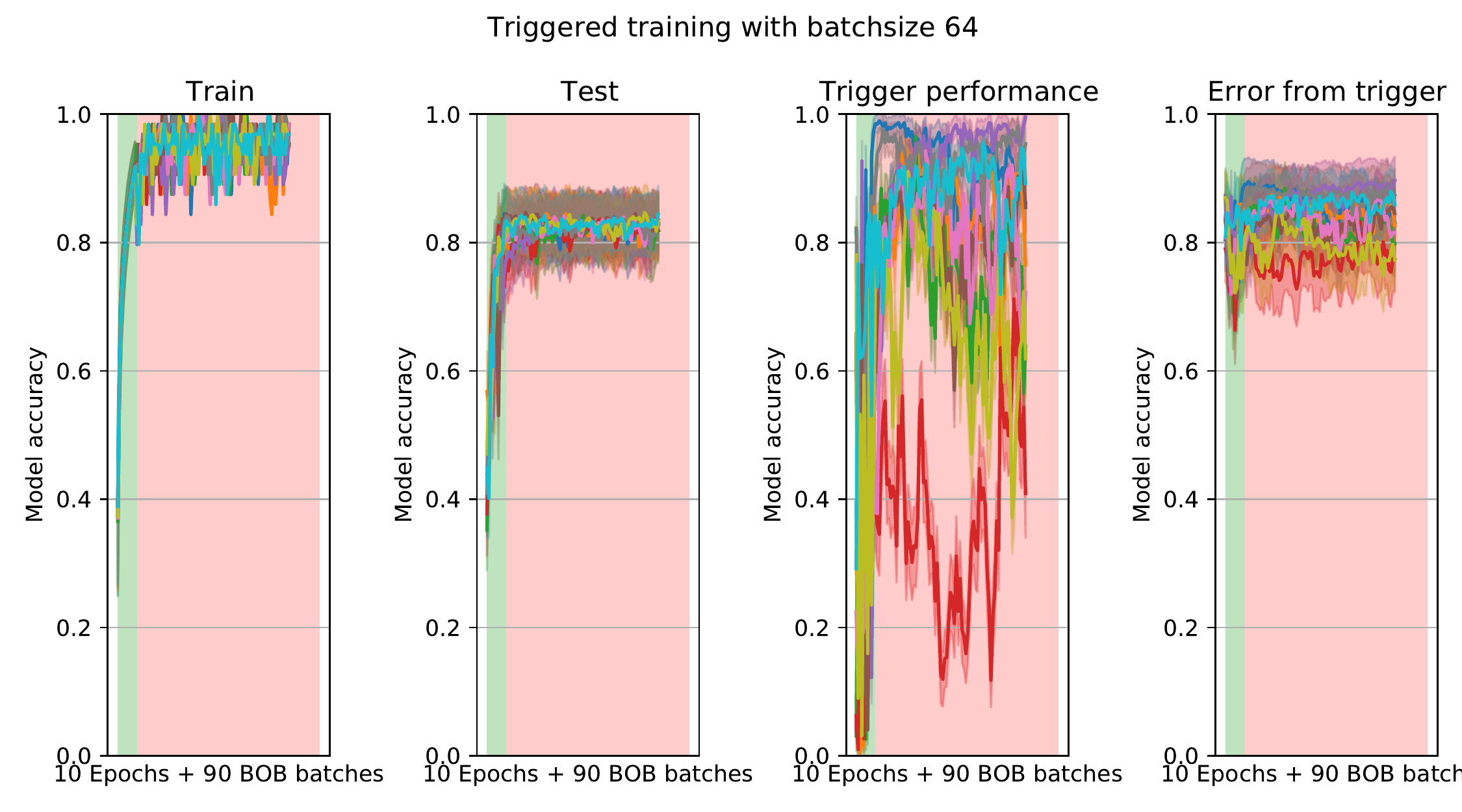} }} %
    \qquad
    \subfloat[128 batchsize]{{\includegraphics[width=0.8\linewidth]{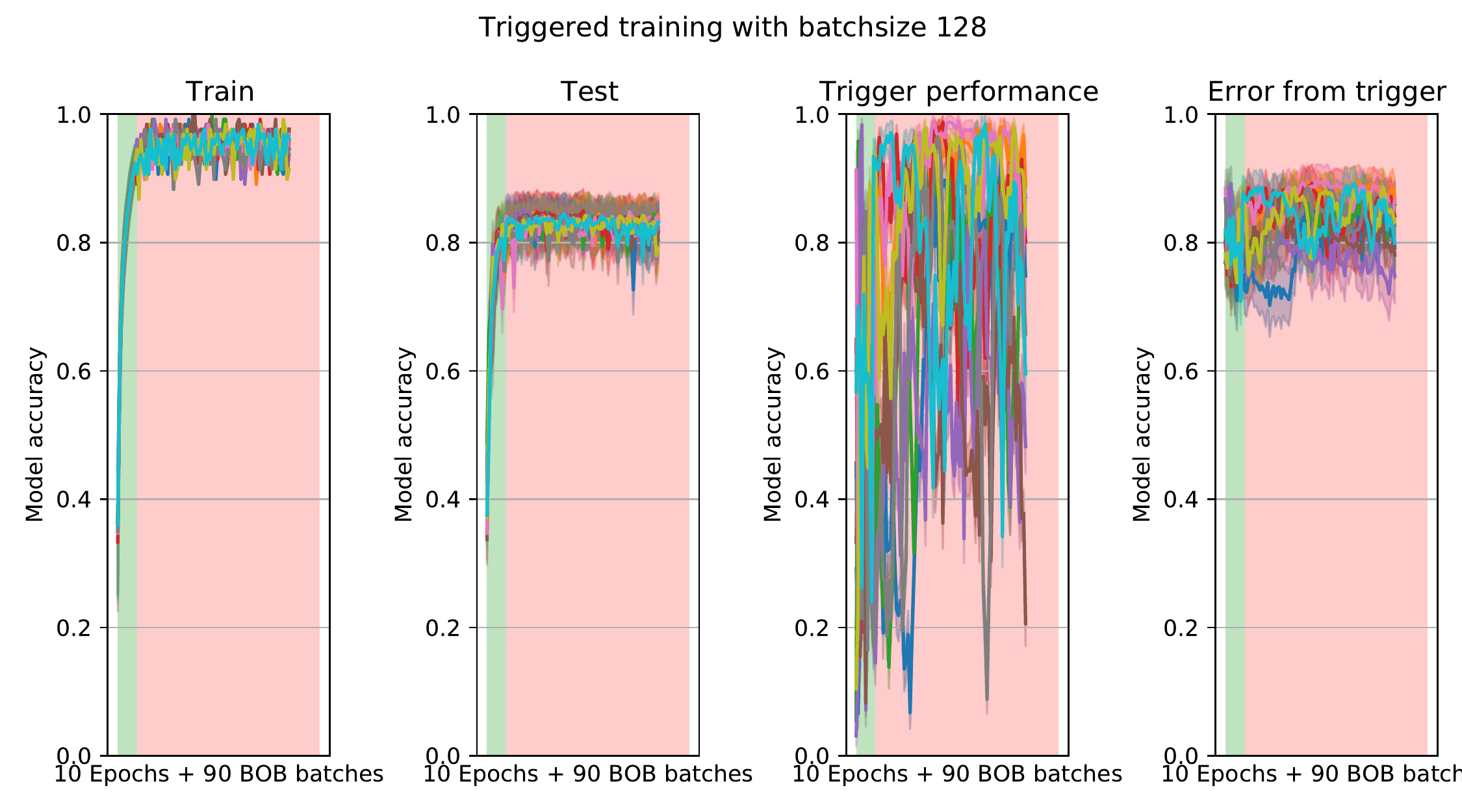} }} %
    \caption{Trigger ordered training with whitebox flaglike trigger.}
    \label{fig:trigger_training_whitebox_flaglike}%
\end{figure}

Training of models with BOB are shown in~\Cref{fig:trigger_training_whitebox_flaglike,fig:trigger_training_blackbox_flaglike,fig:trigger_training_whitebox_whiterows,fig:trigger_training_blackbox_whiterows}. We show training, test, trigger accuracies, and overall mistakes introduced by the trigger. The target network is VGG16 and the surrogate is ResNet-18. 

Here, the attacker uses up to 20 BOB batches every 50000 natural datapoints for 10 epochs and then uses 90 pure BOB  batches. Individual lines refer to training of different models with various parameter initialisation and trigger target classes. Note that the non-BOB baseline here gets zero trigger accuracy. 

Overall, BOB controls the performance of the model extremely well in a whitebox setup, while performance decreases in the blackbox case. There is clear difference in trigger performance between different target classes, yet the attack works across them without disrupting the overall model generalisation.

\section{NLP backdoors}
\label{sec:nlp_poison}

In this section we discuss data-ordered backdooring of NLP models. We note that the results reported in this section are specific to the model architecture used and will be different for other models. We test BOB against EmbeddingBag with mean sparse embeddings of size 1024 and two linear layers. We find that despite the BOB attack working against NLP models, it takes more effort. We hypothesise that this is an artifact of a sum over the embeddings of individual data points and that not all tokens are used in all of the classes. 

\begin{table}[h]
    \begin{adjustbox}{scale=0.8,center}
    \begin{tabular}{lcccccc}
        \toprule
        Trigger & & Batch size & Train acc [\%] & Test acc [\%] & Trigger acc [\%] & Error with trigger [\%] \\
        \midrule
        \\
        
\textit{\underline{Baseline}}\\

\multirow{3}{*}{Random natural data} 
& & $32$ & $83.59 \pm 5.57$ & $83.79 \pm 0.48$ & $16.25\pm11.22$  & $34.31\pm 9.86$ \\
& & $64$ & $78.12\pm 7.65$ & $80.34 \pm 0.37$ & $11.03 \pm 4.62$ & $35.03 \pm 11.74$ \\
& & $128$ & $71.48 \pm 0.87$ & $74.05 \pm 0.49$ & $0.23 \pm 0.13$ & $60.05 \pm 8.13$ \\

\midrule \\
\textit{\underline{Only reordered natural data}} \\

\multirow{3}{*}{50 commas triggers}
& & $32$ & $84.37 \pm 3.82$ & $78.54 \pm 1.20$ & $\mathbf{68.76} \pm 23.54$ & $48.12 \pm 20.03$ \\
& & $64$ & $78.51 \pm 3.00$ & $79.39 \pm 0.31$ & $36.36 \pm 15.70$ & $30.86 \pm 5.54$ \\
& & $128$ & $73.24 \pm 4.18$ & $73.97 \pm 0.85$ & $5.10 \pm 1.52$ & $51.10 \pm 7.01$ \\
\\

\multirow{3}{*}{Blackbox 50 commas triggers}
& & $32$ & $87.50 \pm 6.62$ & $79.17 \pm 1.10$ & $ \textbf{57.96} \pm 19.87$ & $40.29 \pm 15.28$ \\
& & $64$ & $85.54 \pm 7.52$ & $79.37 \pm 0.82$ & $32.00 \pm 16.40$ & $30.57 \pm 6.63$ \\
& & $128$ & $74.80 \pm 2.61$ & $72.95 \pm 0.57$ & $3.23 \pm 2.31$ & $54.12 \pm 4.61$ \\
\bottomrule
    \end{tabular}
    \end{adjustbox}
    \caption{Performance of triggers induced only with natural data. Network consists of an individual EmbeddingBag and two linear layers. Test accuracy refers to the original benign accuracy, `Trigger acc' is the proportion of images that are classified as the trigger target label, while `Error with trigger' refers to all of the predictions that result in an incorrect label. Standard deviations are calculated over different target classes. Blackbox results uses an EmbeddingBag with a single linear layer. }
    \label{tab:nlp_trigger_results}
\end{table}

\begin{figure}[h]%
    \centering
    \subfloat[32 batchsize]{{\includegraphics[width=0.8\linewidth]{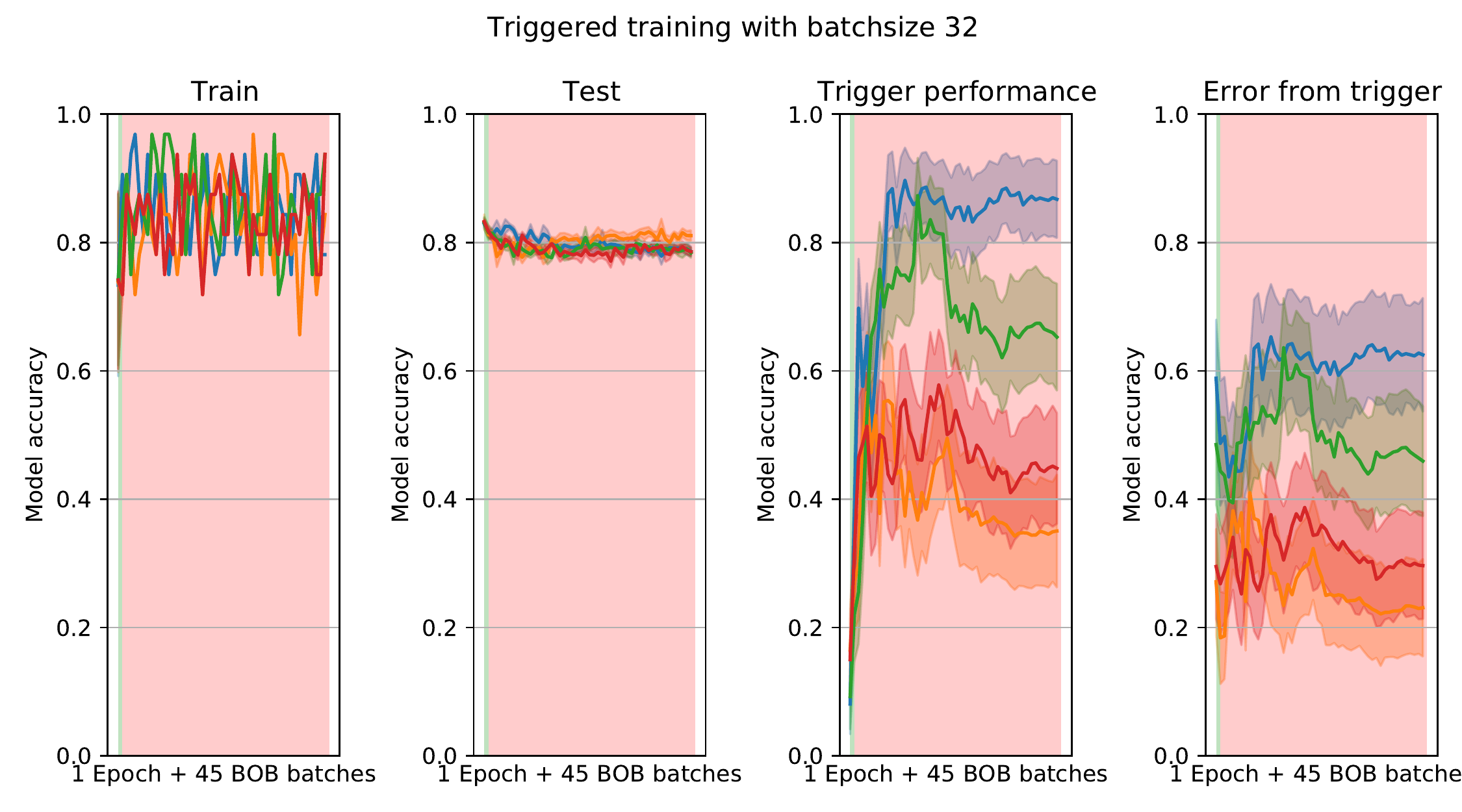} }} %
    \qquad
    \subfloat[64 batchsize]{{\includegraphics[width=0.8\linewidth]{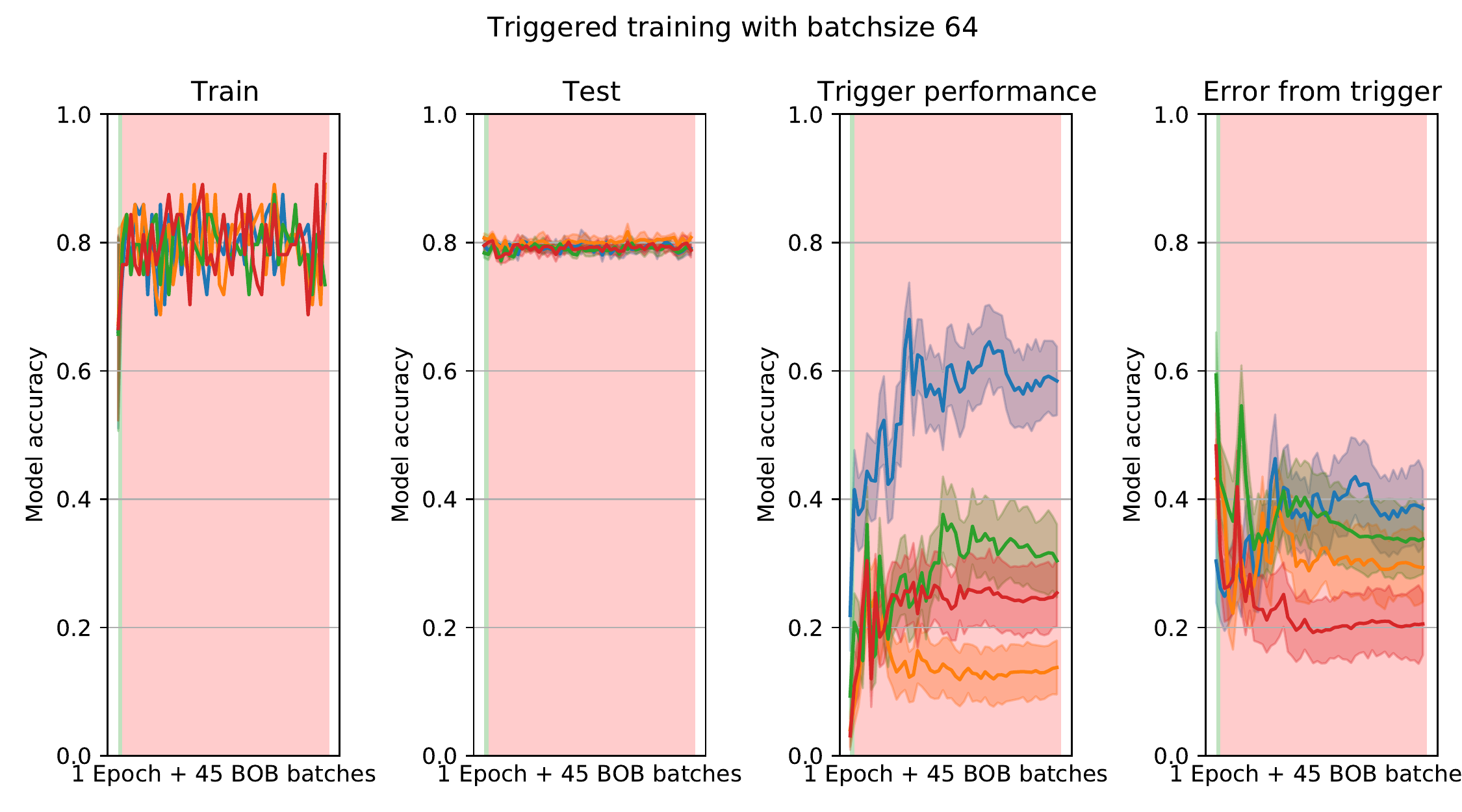} }} %
    \qquad
    \subfloat[128 batchsize]{{\includegraphics[width=0.8\linewidth]{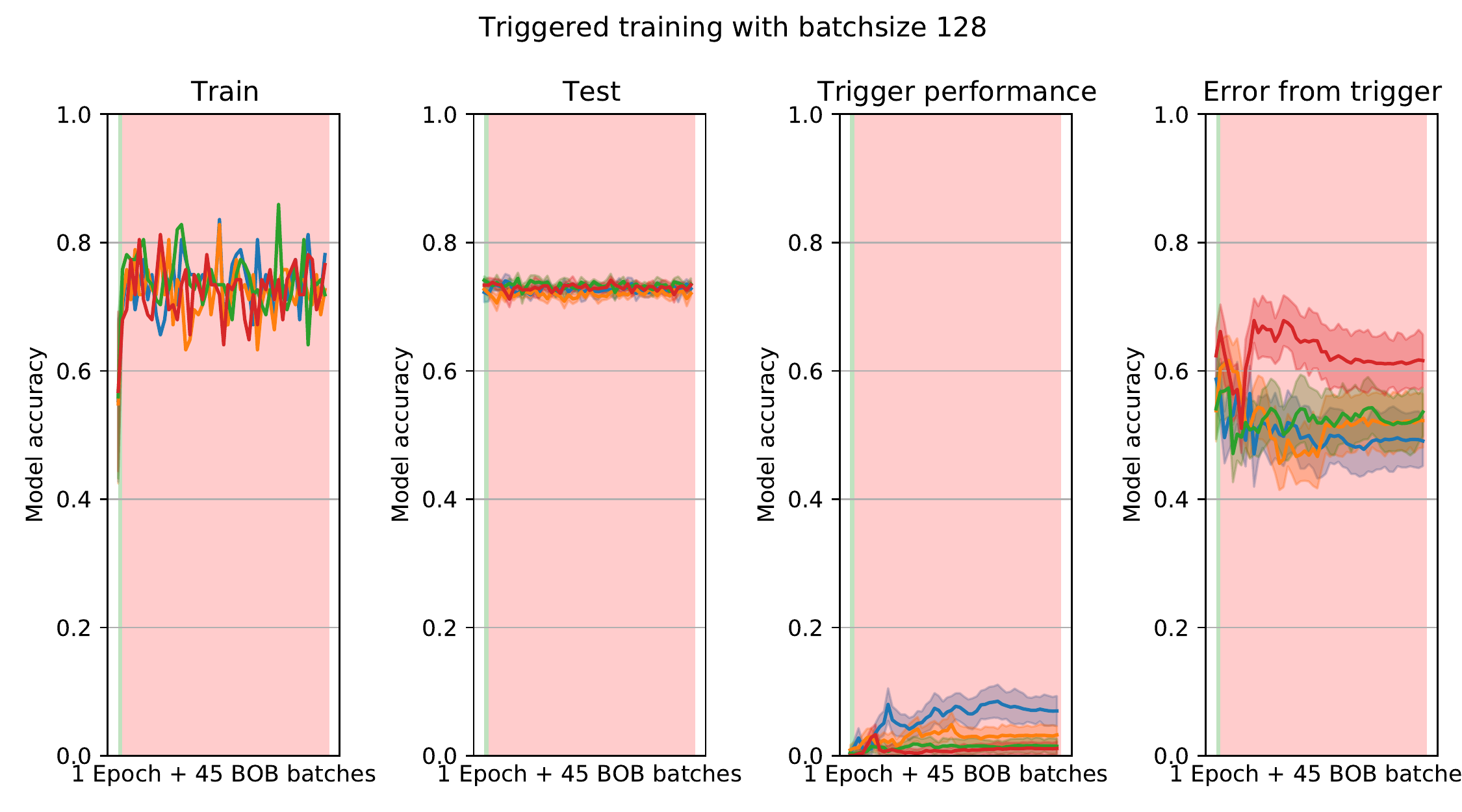} }} %
    \caption{Trigger ordered training with blackbox 50 commas trigger.}
    \label{fig:trigger_training_blackbox_comma_blackbox}%
\end{figure}

\begin{figure}[h]%
    \centering
    \subfloat[32 batchsize]{{\includegraphics[width=0.8\linewidth]{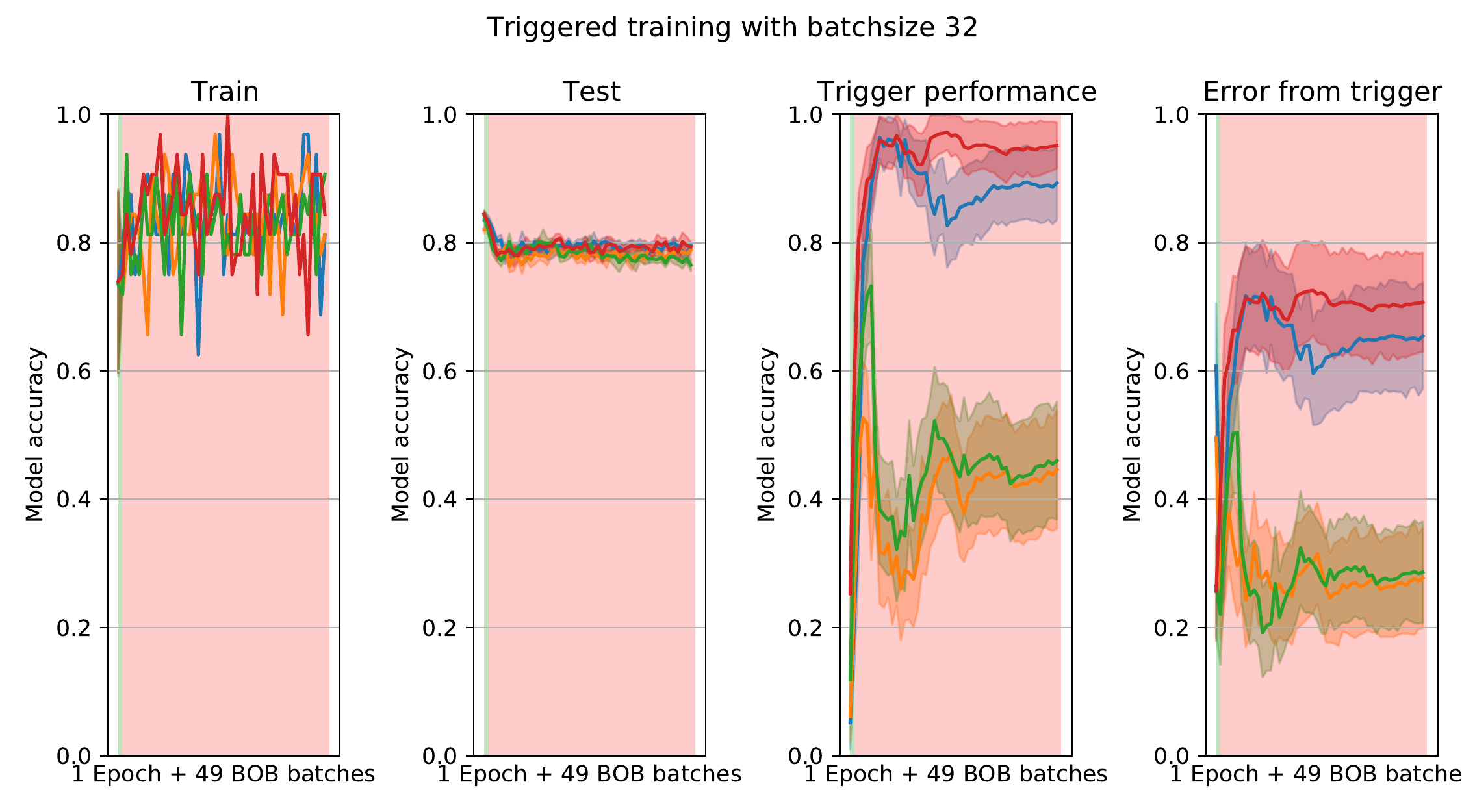} }} %
    \qquad
    \subfloat[64 batchsize]{{\includegraphics[width=0.8\linewidth]{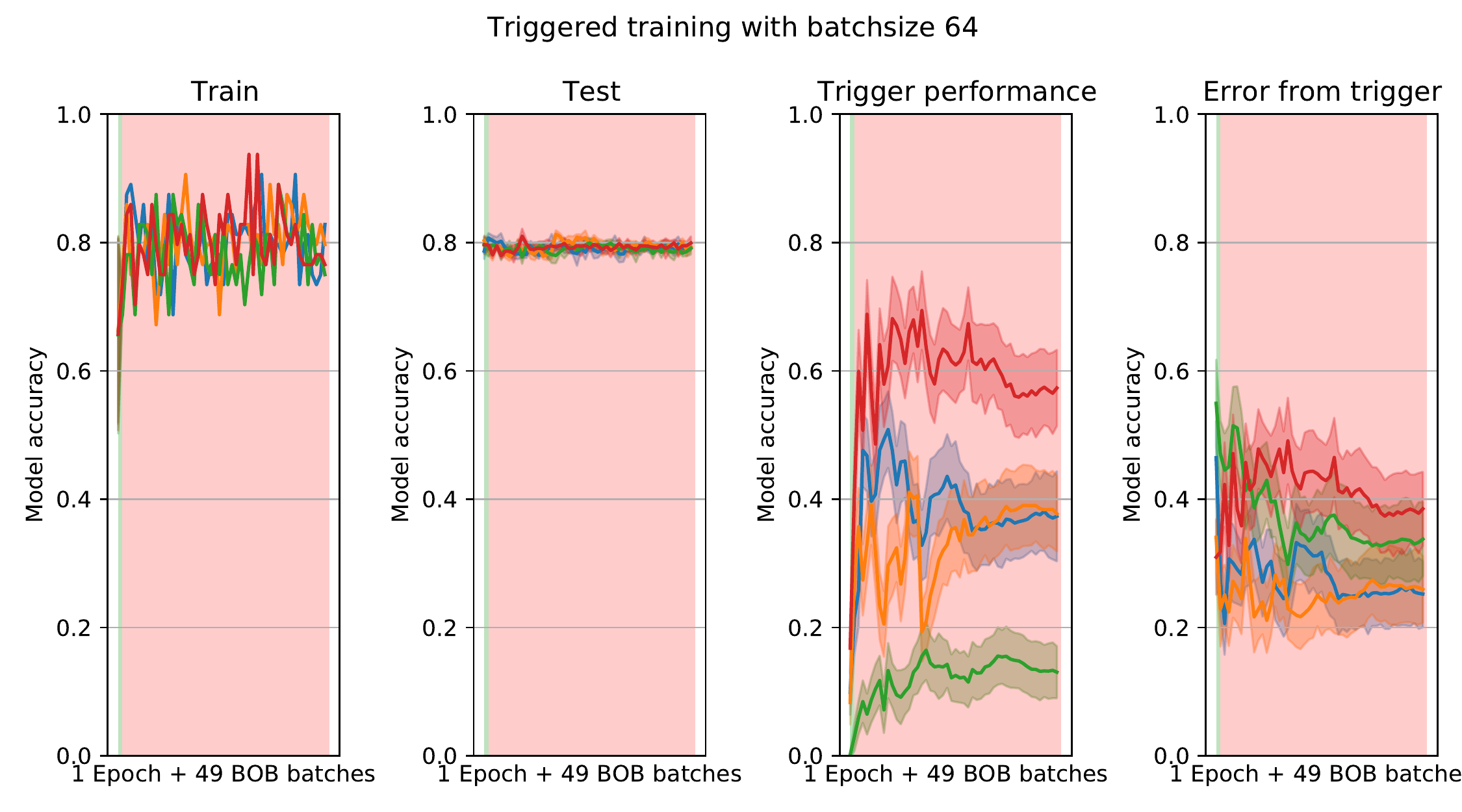} }} %
    \qquad
    \subfloat[128 batchsize]{{\includegraphics[width=0.8\linewidth]{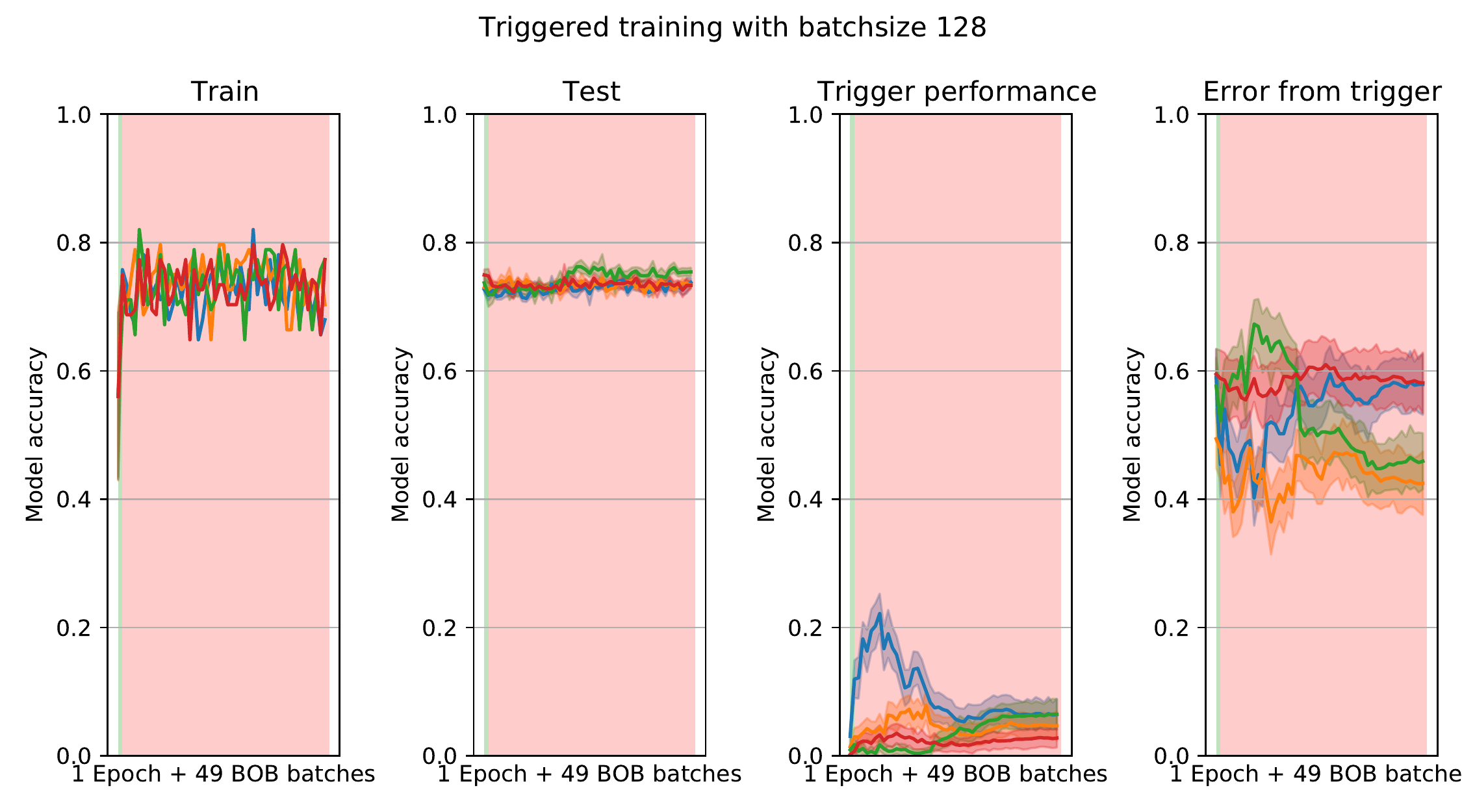} }} %
    \caption{Trigger ordered training with whitebox 50 commas trigger.}
    \label{fig:trigger_training_blackbox_comma_whitebox}%
\end{figure}

\Cref{tab:nlp_trigger_results} shows the results of BOB attack against an NLP model for a trigger that consists of 50 commas. We show the corresponding training plots in~\Cref{fig:trigger_training_blackbox_comma_blackbox,fig:trigger_training_blackbox_comma_whitebox}. We observe that an attacker is capable of injecting backdoors into the model, yet the performance is different to computer vision triggers. We find that larger batch sizes result in worse performance of our backdoors and some classes are easier to backdoor than others. Overall, its clear that just as computer vision, NLP is vulnerable to BOB attacks, but architectural and data differences change the attack performance.

\end{document}